\documentclass{article}

\usepackage[preprint]{colm2026_conference}

\usepackage{amsmath,amsfonts,bm}

\def\eqref#1{equation~\ref{#1}}

\def\1{\bm{1}}

\DeclareMathAlphabet{\mathsfit}{\encodingdefault}{\sfdefault}{m}{sl}
\SetMathAlphabet{\mathsfit}{bold}{\encodingdefault}{\sfdefault}{bx}{n}

\usepackage{amsmath}
\usepackage{amssymb}
\usepackage{amsthm}
\usepackage{bm}
\usepackage{bbm}
\usepackage[most]{tcolorbox}
\usepackage{xcolor}
\usepackage{tikz}
\usepackage{graphicx}
\usepackage{caption}
\usepackage{subcaption}
\usepackage{wrapfig}
\usepackage{enumitem}
\usepackage{microtype}
\usepackage{hyperref}
\usepackage{url}
\usepackage{booktabs}
\usepackage{multirow}
\usepackage{algorithm}
\usepackage{algpseudocode}
\usepackage{tabularx}
\usepackage{placeins}

\usepackage{lineno}

\definecolor{darkblue}{rgb}{0, 0, 0.5}
\hypersetup{colorlinks=true, citecolor=darkblue, linkcolor=darkblue, urlcolor=darkblue}

\newtcolorbox{problembox}[1]{
    enhanced,
    breakable,
    colback=gray!3,
    colframe=blue!60,
    title=\textbf{#1},
    fonttitle=\normalsize,
    arc=2mm,
    boxrule=0.6pt,
    left=2mm,
    right=2mm,
    top=2mm,
    bottom=2mm
}

\newtcolorbox{attemptbox}[1]{
    enhanced,
    breakable,
    colback=gray!3,
    colframe=red!60,
    title=\textbf{#1},
    fonttitle=\normalsize,
    arc=2mm,
    boxrule=0.6pt,
    left=2mm,
    right=2mm,
    top=2mm,
    bottom=2mm
}

\theoremstyle{plain}
\newtheorem{theorem}{Theorem}
\newtheorem{assumption}{Assumption}
\newtheorem{lemma}{Lemma}

\title{Understanding Scattered Forest Search: A Version-Space Perspective on Multi-Turn Program Correction}

\author{Yuto Tanaka \& Issei Sato \\
  Department of Computer Science \\
  The University of Tokyo \\
  Tokyo, Japan \\
  \{yuto-t, sato\}@g.ecc.u-tokyo.ac.jp
}

\begin{document}

\maketitle

\begin{abstract}
In multi-turn program correction, the state-of-the-art method \textsc{Scattered Forest Search} (SFS) has been proposed, employing Monte Carlo Tree Search (MCTS) with carefully crafted initial seeds and text-based optimization. However, since SFS integrates multiple components, the effects of each component on performance and the overall behavior of SFS have not been sufficiently analyzed. In this work, we theoretically analyze the refinement process of SFS from the perspective of \textit{version spaces} in learning theory and clarify its behavior. First, as a basis for the theoretical analysis, we introduce a sequential self-refinement method (Line), which starts from an initial program and repeatedly refines the resulting program. Furthermore, while Line progresses the refinement process in the depth direction, we introduce \textsc{Iterative Refinement of Repair Instructions} (IRRI) to capture the refinement process in the width direction, which fixes initial programs and iteratively refines repair instructions. We then analyze Line and IRRI and conduct a theoretical analysis of SFS by positioning it as an intermediate method between the two. Our theoretical analysis reveals that SFS exhibits behavior closer to IRRI than to Line, and this theoretical characteristic is also confirmed experimentally. Considering computational resources and methodological simplicity, these results suggest that a simpler method, IRRI, may achieve a similar refinement process without relying on the complex correction mechanism of SFS.
\end{abstract}

\section{Introduction}
A representative task in the study of inference scaling is multi-turn program correction, in which an initial program is refined through multiple rounds of feedback. Several representative approaches have been proposed, including best-of-N (BoN) sampling \citep{cobbe2021training, lightman2023let}, self-correction methods \citep{madaan2023self, shinn2023reflexion}, and tree-search-based methods that integrate multiple reasoning trajectories \citep{tian2024toward, zhou2023language}. These methods have been shown to improve performance \citep{chen2022codet, wang2024planning}. One reason for this effectiveness is that providing enough trials increases the probability that an LLM will sample a correct answer \citep{brown2024large, snell2024scaling}. However, \citet{light2025sfs} have shown that existing methods often generate highly similar outputs, resulting in insufficient exploration of the solution space.

\citet{light2025sfs} have proposed \textsc{Scattered Forest Search} (SFS) that builds on Monte Carlo Tree Search (MCTS) and incorporates techniques such as careful seed initialization and text-based optimization. In SFS, initial programs serve as the root nodes, and the program space is efficiently explored through MCTS by iteratively applying repair instructions generated by an LLM. Moreover, feedback is provided for the repair instructions used during the refinement process, and the resulting insights are accumulated as shared information. They have empirically reported that SFS achieves state-of-the-art performance in multi-turn program correction \citep{light2025sfs}. However, it remains unclear why SFS leads to performance improvements and how its refinement process proceeds. Understanding these factors is crucial for developing simpler and more efficient approaches that could retain the benefits of inference scaling without relying on costly or complex search mechanisms.

To address this gap, we develop a theoretical framework for analyzing refinement processes from the perspective of \textit{version spaces} in learning theory \citep{mitchell1977version}. As a basis for this framework, we first introduce a sequential self-refinement method, Line, which starts from an initial program and repeatedly refines the resulting program. This method captures the characteristics of refinement in the depth direction, where each refinement is applied to the program produced by the preceding step. However, Line does not capture refinement in the width direction, where initial programs can be refined using progressively improved repair instructions. To capture this aspect, we introduce \textsc{Iterative Refinement of Repair Instructions} (IRRI), which fixes initial programs and iteratively refines repair instructions. We then theoretically analyze the refinement processes of Line and IRRI from the perspective of version spaces in learning theory. Based on these analyses, we characterize SFS as an intermediate method between Line and IRRI and analyze its refinement process within the same theoretical framework. Our theoretical analysis reveals two key properties: (1) the hypothesis space is non-increasing for IRRI and probabilistically non-increasing for Line and SFS, and (2) SFS exhibits refinement behavior closer to IRRI than to Line.  We further empirically investigate these properties by examining the tree structures of SFS and inference scaling, as well as the diversity of repair instructions across refinement steps. Our theoretical and empirical results suggest that IRRI can capture a refinement process similar to that of SFS despite its simpler design.

\section{Related work}
Recent studies have increasingly focused on improving model performance by allocating more computation at inference time \citep{brown2024large, snell2024scaling}. One of the simplest and most common methods is best-of-N (BoN) sampling \citep{cobbe2021training, lightman2023let}. In this approach, multiple samples are generated from the same prompt to ensure diversity among responses and improve output quality. Prior work has also proposed self-correction methods in which an initial response is generated and then iteratively verified and revised \citep{madaan2023self, shinn2023reflexion}. These methods improve the model's responses by incorporating internal or external feedback as additional information. \citet{yang2025probabilistic} has proposed a probabilistic theory of self-correction to explain the performance improvements observed in these methods.

Furthermore, incorporating tree search methods such as DFS, BFS, and MCTS has been shown to be effective for inference scaling. Prior work has applied tree search to multi-step tasks, including selecting the next reasoning step in reasoning processes \citep{yao2023tree, besta2024graph, tian2024toward, zhou2023language} and decoding the next token during generation \citep{feng2023alphazero, zhang2023planning}. Tree search has also been applied to optimization problems, where it functions as a black-box method for exploring a solution space \citep{light2025sfs}.

Overall, we build upon previous research \citep{light2025sfs} and aim to identify the factors that truly contribute to the behavior and performance of increasingly complex inference scaling methods. Motivated by the lack of a theoretical understanding of such methods, we analyze their iterative refinement processes from the perspective of version spaces in learning theory \citep{mitchell1977version} and develop a theoretical framework for characterizing their behavior.

\section{Problem setting and refinement methods}
\subsection{Problem setting}
\label{subsec:problem-setting}

Following \citet{light2025sfs}, we formulate program generation as a black-box optimization problem \citep{golovin2017google}. Each task is represented as $x = \langle p,  D_{\text{gt}}\rangle \in X$, where $p$ denotes a prompt, and $D_{\text{gt}}$ represents a ground-truth test-case set. The objective is to provide an LLM with a prompt $p$ and generate a program $c \sim \mathrm{LLM}(p)$ that passes all test cases in $D_{\text{gt}}$. Although the model cannot access $D_{\text{gt}}$, it can generate a validation test-case set $D_{\text{val}}$. Using $D_{\text{val}}$ and the program execution environment, the model obtains feedback on its generated program $c$. This feedback is then used to determine whether the program should be refined. In the evaluation phase, generated programs are evaluated against $D_{\text{gt}}$ and are deemed correct only when all test cases are passed. For each task, an LLM is allowed to generate programs up to $k$ times, and a task is considered solved if at least one of them is correct. Following these rules, the evaluation metric Pass@k Rate \citep{chen2021evaluating} is calculated as the proportion of tasks $x \in X$ that are solved with at most $k$ generated programs. An overview of this problem setting is illustrated in Appendix~\ref{app:overview-of-problem-setting}.

\subsection{Scattered Forest Search (SFS)}
\label{subsec:introduce-sfs}

Several inference-time methods have been explored in prior work on multi-turn program correction. However, \citet{light2025sfs} have shown that these methods often generate highly similar candidate programs, resulting in insufficient exploration of the solution space. To address this challenge, they have proposed \textsc{Scattered Forest Search} (SFS) \citep{light2025sfs}. SFS is a MCTS-based method that combines careful seed initialization with text-based optimization and consists of the following three components. 
\begin{itemize}
    \item \textsc{Scattering}: Dynamically varies input prompts with different repair instructions when sampling revised programs. 
    \item \textsc{Foresting}: Generates multiple initial programs that act as starting points for MCTS.
    \item \textsc{Scouting}: Provides feedback on repair instructions employed during the refinement process and distills general insights that can be shared across search branches.
\end{itemize}

Their study has confirmed that SFS outperforms existing methods in terms of inference scaling performance \citep{light2025sfs}. Details of SFS are provided in Appendix~\ref{app:details-of-sfs}.

\subsection{Sequential self-refinement method (Line)}
\label{subsec:introduce-line}

As a basis for our theoretical analysis, we introduce a sequential self-refinement method, Line, which builds upon the frameworks of \citet{madaan2023self} and \citet{shinn2023reflexion}. This method captures the characteristics of the refinement process in the depth direction. 

\textbf{Initial Generation}: Given an input prompt $p$, Line generates an initial program $c_1$.
\begin{equation}
    c_1 \sim \mathrm{LLM}(p).
\end{equation}

\textbf{Refinement Process}: First, based on the refinement target $c_t$, Line generates a repair instruction $r_t$. Next, Line refines $c_t$ using $r_t$.
\begin{equation}
    r_t \sim \mathrm{LLM}(p, c_t), \quad c_{t + 1} \sim \mathrm{LLM}(p, c_t, r_t).
\end{equation}

\subsection{Iterative Refinement of Repair Instructions (IRRI)}
\label{subsec:introduce-irri}

While Line captures the refinement process in the depth direction, it does not capture refinement in the width direction, where initial programs can be revised using increasingly refined repair instructions. To capture this aspect, we introduce \textsc{Iterative Refinement of Repair Instructions} (IRRI), which refines initial programs through iterative feedback on repair instructions. IRRI accumulates feedback from previous repair instructions as shared information and uses it to generate subsequent ones, thereby progressively improving the repair instructions used for correction. An overview of IRRI is illustrated in Appendix~\ref{app:overview-of-irri}. 

\textbf{Initial Generation}: Given an input prompt $p$, IRRI generates one or multiple initial programs $c_1, \ldots, c_n$.
\begin{equation}
    c_1 \sim \mathrm{LLM}(p), \ldots, ~c_n \sim \mathrm{LLM}(p).
\end{equation}

\textbf{Refinement Process}: First, based on the selected initial program $c_t$ and the shared information $I_t$, IRRI generates one or more repair instructions $r_1, \ldots, r_m$. Next, IRRI refines $c_t$ using one of them $r_t$.
\begin{equation}
    r_1, \ldots, r_m \sim \mathrm{LLM}(p, c_t, I_t), \quad c_{\text{ref}} \sim \mathrm{LLM}(p, c_t, r_t).
\end{equation}

\textbf{Feedback for Repair Instructions}: After the refinement process, IRRI evaluates the revised program $c_{\text{ref}}$ on $D_{\text{val}}$ and updates the shared information $I_t$ based on $c_t$, $r_t$, and $c_{\text{ref}}$.
\begin{equation}
    I_{t + 1} \sim \mathrm{LLM}(p, c_t, I_t, r_t, c_{\text{ref}}).
\end{equation}

\section{Theoretical analysis using version spaces}
In this section, we present a theoretical analysis of Line, IRRI, and SFS. We begin by analyzing two contrasting methods, Line and IRRI, and then analyze SFS, which occupies an intermediate position between the two.

In multi-turn program correction, the prompt for program refinement at each turn can be represented as a pair consisting of the program to be revised and the repair instruction applied to it. These two elements are the only components that vary across turns, and thus fully characterize the prompt used to revise the program. Throughout the refinement process, the goal of each method can be regarded as identifying the true pair among these. From this perspective, each method can be viewed as a process in which the feedback obtained at each refinement step is used as a clue to eliminate pairs that are likely incorrect, thereby narrowing down the candidates to the correct pair.

To this end, we adopt the notion of a version space from learning theory \citep{mitchell1977version}, treating the product space of programs and repair instructions as the hypothesis space, and analyze how this hypothesis space shrinks and evolves throughout the refinement process. This framework allows us to compare how each method utilizes feedback to update the hypothesis space and how differences in this updating process relate to differences in performance among the methods.

\subsection{Definition and assumption}
Let $C$ denote the space of programs, $R$ denote the space of repair instructions, $H = C \times R$ denote the hypothesis space, and $E$ denote the space of observations. Following Section~\ref{subsec:problem-setting}, each element $e \in E$ represents a failed test case on $D_{\text{val}}$. For $c \in C$ and $r \in R$, we write $\text{Pass}(c, r)$ to indicate whether applying $r$ to $c$ yields a program that satisfies the specification and $\text{Cons}(c, r, e)$ to indicate whether applying $r$ to $c$ is consistent with an observation $e \in E$. Using these, we define the true hypothesis space as $H^{\star} = \left\{ h = (c, r) \in H \mid \text{Pass}(c, r) = 1 \right\}$. For a fixed program $c \in C$, we write $H_c$ for the hypothesis space associated with $c$ and $H^{\star}_{c}$ for the corresponding true hypothesis space. We denote by $E_c$ the space of observations that can be obtained from $c$, and we define $H_c(e) = \left\{ h = (c, r) \in H_c \mid \text{Cons}(c, r, e) = 1 \right\}$ as the subset of $H_c$ that is consistent with $e \in E_c$. We then denote the history up to time $t$ by $O_t = ((h_1, e_1), \cdots, (h_t, e_t))$, where, at each time step $s \, (1 \le s \le t)$, the hypothesis $h_s = (c_s, r_s)$ belongs to $H_{c_s}$ and the observation $e_s$ belongs to $E_{c_s}$. Finally, we write $\mathcal{T}_t$ for the true hypothesis space at time $t$ and $\mathcal{V}_t$ for the hypothesis space that survives (i.e., remains consistent with the history $O_t$) at time $t$.

Throughout this analysis, we make the following assumption, which states that any true hypothesis is consistent with every observation.
\begin{assumption}
    \label{asm:soundness}
    (Soundness)
    \begin{align}
        h^{\star} = (c^{\star}, r^{\star}) \in H^{\star} \Longrightarrow \forall e\in E_{c^{\star}}. h^{\star} \in H_{c^{\star}}(e).
    \end{align}
\end{assumption}

\subsection{Sequential self-refinement method (Line)}
We first analyze Line, where the refinement target changes throughout the correction process. Let $c_1 \in C$ denote the initial program, and let $c_{t + 1} \in C$ denote the program obtained by the $t$-th correction. Since the refinement target keeps changing, the refinement target $c_t$ may be either a previously unknown program or an already known program. We therefore define the probability that the refinement target $c_t$ is a known program as $q_t^{\text{old}} = \Pr[c_t = \text{known} \mid O_{t - 1}]$. Since the refinement target at time $1$ is the initial program, $q_1^{\text{old}} = 1$ holds. In the early stages of the refinement process, only a small number of programs have been generated, so $q_t^{\text{old}}$ remains relatively low. However, as the process proceeds, the number of generated programs increases, and $q_t^{\text{old}}$ increases accordingly. Based on this observation, we model $q_t^{\text{old}}$ as
\begin{align}
    q_t^{\text{old}} = 1 - \exp (- \tau t) \quad (t \ge 2),
\end{align}
where $\tau$ is a small positive constant that depends on factors such as the model and the dataset.

In Line, the hypothesis space surviving at time $t$ can be expressed as
\begin{align}
    \mathcal{V}_0 = H_{c_1}, \quad
    \mathcal{V}_t = 
    \begin{cases}
        \mathcal{V}_{t-1} \cup H_{c_t}(e_t) & (c_t = \text{unknown})\\
        (\mathcal{V}_{t-1} \cap H_{c_t}(e_t)) \cup (\mathcal{V}_{t-1} \setminus H_{c_t}) & (c_t = \text{known})
    \end{cases}
    \, (t \ge 1).
\end{align}

Correspondingly, the true hypothesis space at time $t$ can be expressed as
\begin{align}
    \mathcal{T}_0 = H_{c_1}^{\star}, \quad
    \mathcal{T}_t =
    \begin{cases}
        \mathcal{T}_{t - 1} \cup H_{c_t}^{\star} & (c_t = \text{unknown})\\
        \mathcal{T}_{t - 1} & (c_t = \text{known})
    \end{cases}
    \, (t \ge 1).
\end{align}

In addition to Assumption~\ref{asm:soundness}, we make the following assumption for analyzing Line.
\begin{assumption}
    \label{asm:realizability-of-line}
    (Probabilistic realizability)
    For the refinement target $c_t \in C$, define the random variable $Z_{c_t} = \mathbbm{1} [H_{c_t}^{\star} \neq \varnothing] \in \{0, 1\}$. We assume the following properties of the LLM's sampling distribution.
    \begin{enumerate}
        \item (Initial distribution) For the initial program $c_1 \in C$, there exists $q^{\text{Line}} \in (0, 1)$ such that $\Pr[H_{c_1}^{\star} \neq \varnothing] = q^{\text{Line}}$.
        \item (Uniform lower bound) There exists $\pi^{\text{Line}} \in (0, 1)$ such that $\Pr[Z_{c_t} = 1 \mid c_t = \text{unknown}] \ge \pi^{\text{Line}}$ for every time $t$.
        \item (Independence) For different programs $c \neq c'$, $Z_{c}$ and $Z_{c'}$ are mutually independent.
    \end{enumerate}
\end{assumption}

This assumption captures the realizability of the LLM, requiring that the initial program admit a true hypothesis with probability $q^{\text{Line}}$ and that an unknown program admit a true hypothesis with probability at least a uniform positive lower bound $\pi^{\text{Line}}$. Once the true hypothesis space $\mathcal{T}_t$ first becomes non-empty, the true hypothesis space is never pruned, so it remains non-empty at all subsequent times. We therefore define the first time at which $\mathcal{T}_t$ becomes non-empty as $t_c  = \min\{t \ge 0 \mid \mathcal{T}_t \neq \varnothing\}$. Then, we obtain the following theorem, which characterizes the refinement process of Line.
\begin{theorem}
    \label{thm:theorem-of-line}
    Under Assumptions~\ref{asm:soundness} and~\ref{asm:realizability-of-line}, for any history $O_t = ((h_1, e_1),(h_2, e_2), \cdots, (h_t, e_t))$ and time $t_c = \min\{t \ge 0 \mid \mathcal{T}_t \neq \varnothing\}$, the following properties hold.
    \begin{enumerate}
        \item (($1 - \delta_{\text{Line}}$)-monotonic non-increasing) \\
        At each time $t$,
        \begin{itemize}
            \item With probability $1 - q_t^{\text{old}}$, the refinement target $c_t$ is unknown, in which case $\mathcal{V}_t \supseteq \mathcal{V}_{t - 1}$.
            \item With probability $q_t^{\text{old}}$, the refinement target $c_t$ is known, in which case $\mathcal{V}_t \subseteq \mathcal{V}_{t - 1}$.
        \end{itemize}
        Moreover, with probability at most
        \begin{align}
            1 - \delta_{\text{Line}} = \exp \left( - A_{\tau} \left(1 - e^{- \tau (t - 1)}\right) \right) \quad \left(A_{\tau} = \frac{e^{- 2 \tau}}{1 - e^{- \tau}}\right),
        \end{align}
        all refinement targets $c_1, \ldots, c_t$ are known, in which case $\mathcal{V}_t \subseteq \mathcal{V}_{t - 1} \subseteq \cdots \subseteq\mathcal{V}_0$.
        \item (Preservation of the true hypothesis) $\mathcal{T}_t \subseteq \mathcal{V}_t \quad (t \ge t_c)$.
        \item (Non-emptiness)
        \begin{align}
            \Pr[\mathcal{V}_t \neq \varnothing] \ge 1 - (1 - q^{\text{Line}})\exp \left( - \pi^{\text{Line}} A_{\tau} \left(1 - e^{- \tau (t - 1)}\right) \right) \quad \left(A_{\tau} = \frac{e^{- 2 \tau}}{1 - e^{- \tau}}\right).
        \end{align}
    \end{enumerate}
\end{theorem}

The proof of Theorem~\ref{thm:theorem-of-line} is provided in Appendix~\ref{subapp:proof-of-thm:theorem-of-line}. This theorem shows that there is always a upper bound on the probability that all refinement targets are known and, consequently, the hypothesis space is monotonically non-increasing. This indicates that there is a limit to the extent to which Line can properly narrow down the hypothesis space. There is also a further limitation in that the lower bound on the probability of non-emptiness does not reach $1$ even as $t \rightarrow \infty$. Further discussion is provided in Appendices~\ref{subapp:further-discussion-of-line} and~\ref{subapp:parameter-and-upper-bound-estimation-in-line}.

\subsection{Iterative Refinement of Repair Instructions (IRRI)}
\label{subsec:theoretical-analysis-of-irri}

We next turn to the analysis of IRRI. Unlike Line, IRRI first generates a set of initial programs $C_{\text{init}} \subseteq C$ and iteratively refines repair instructions. Given $C_{\text{init}}$, the hypothesis space surviving at time $t$ can be written as
\begin{align}
    \mathcal{V}_0 = \bigcup_{c \in C_{\text{init}}} H_c, \quad
    \mathcal{V}_t = (\mathcal{V}_{t-1} \cap H_{c_t}(e_t)) \cup (\mathcal{V}_{t-1} \setminus H_{c_t}) \, (t \ge 1).
\end{align}

Correspondingly, the true hypothesis space does not depend on $t$ and is given by
\begin{align}
    \mathcal{T}_t = \bigcup_{c \in C_{\text{init}}} H_c^{\star} \quad (t \ge 0).
\end{align}

In addition to Assumption~\ref{asm:soundness}, the analysis of IRRI relies on the following assumption.
\begin{assumption}
    \label{asm:realizability-of-irri}
    (Probabilistic realizability)
    For each initial program $c \in C_{\text{init}}$, define the random variable $Y_c = \mathbbm{1} [H_c^{\star} \neq \varnothing] \in \{0, 1\}$. We assume the following properties of the LLM's sampling distribution.
    \begin{enumerate}
        \item (Uniform lower bound) There exists $q^{\text{IRRI}} \in (0, 1)$ such that $\Pr[Y_c = 1] \ge q^{\text{IRRI}}$ for every $c \in C_{\text{init}}$.
        \item (Independence) $\{Y_c\}_{c \in C_{\text{init}}}$ are conditionally independent given the prompt.
    \end{enumerate}
\end{assumption}

This assumption captures the realizability of the LLM, requiring a uniform positive lower bound on the probability that each initial program admits a true hypothesis. Under these assumptions, we obtain the following theorem, which characterizes the refinement process of IRRI.
\begin{theorem}
    \label{thm:theorem-of-irri}
    Under Assumptions~\ref{asm:soundness} and~\ref{asm:realizability-of-irri}, for any history $O_t = ((h_1, e_1),(h_2, e_2), \cdots, (h_t, e_t))$, the following properties hold.
    \begin{enumerate}
        \item (Monotonic non-increasing) $\mathcal{V}_t \subseteq \mathcal{V}_{t-1}$.
        \item (Preservation of the true hypothesis) $\mathcal{T}_0 \subseteq \mathcal{V}_t$.
        \item (Non-emptiness) $\Pr[\mathcal{V}_t \neq \varnothing] \ge 1 - (1 - q^{\text{IRRI}})^N$.
    \end{enumerate}
\end{theorem}

The proof of Theorem~\ref{thm:theorem-of-irri} is provided in Appendix~\ref{subapp:proof-of-thm:theorem-of-irri}. This theorem guarantees that IRRI can appropriately narrow down the hypothesis space during the refinement process. Here, letting $G(N) = 1 - (1 - q^{\text{IRRI}})^N$, since $q^{\text{IRRI}} \in (0, 1)$ implies $0 < 1 - q^{\text{IRRI}} < 1$, $G(N)$ is monotonically increasing in $N$ and satisfies
\begin{align}
    \lim_{N \rightarrow \infty} G(N) = 1 - \lim_{N \rightarrow \infty} (1 - q^{\text{IRRI}})^N = 1.
\end{align}
This shows that increasing the number of initial programs generated has the benefit of increasing the probabilities that $\mathcal{T}_0 \neq \varnothing$ and $\mathcal{V}_t \neq \varnothing$.

\subsection{Scattered Forest Search (SFS)}
\label{subsec:theoretical-analysis-of-sfs}

We then analyze SFS by positioning it as an intermediate method between Line and IRRI. SFS first generates a set of initial programs $C_{\text{init}} \subseteq C$, but MCTS may select the refinement target outside $C_{\text{init}}$. Let $c_t \in C$ denote the refinement target at time $t$, and let $\tilde{c}_t \in C$ denote the program obtained at time $t$. We also denote the tree structure after the $t$-th correction by $S_t = (U_t, F_t) \quad (U_t \subseteq C, \, F_t \subseteq C \times C)$. At time $t = 0$, $S_0 = (C_{\text{init}}, \varnothing)$. At time $t \ge 1$, since refining $c_t$ yields $\tilde{c}_t$, we have $S_t = (U_{t - 1} \cup \{\tilde{c}_t\}, F_{t - 1} \cup \{(c_t, \tilde{c}_t)\})$. On the tree structure $S_t = (U_t, F_t)$, we recursively define the depth function $\text{dep}: \bigcup_t U_t \rightarrow \mathbb{Z}_{\ge 0}$ as follows. Since the set of initial programs $C_{\text{init}}$ forms the roots of the tree structure, $\text{dep}(c) = 0 \quad (c \in C_{\text{init}})$. When $(c_t, \tilde{c}_t)$ is added at time $t$, $\text{dep}(\tilde{c}_t) = \text{dep}(c_t) + 1$. Due to \textsc{Foresting}, the tree structure $S_t$ may have multiple roots. However, since $(c_t, \tilde{c}_t)$ represents the relation "refining $c_t$ yields $\tilde{c}_t$", each node has at most one parent. Consequently, the function $\text{dep}$ is uniquely determined for any $c \in C$.

We denote the maximum depth of the tree structure $S_t$ by $K_t = \max_{c \in U_t} \text{dep}(c)$. We also denote the probability that the refinement target $c_t$ is a node at depth $k$ by $q_t^{\text{dep}}(k) = \Pr[\text{dep}(c_t) = k \mid O_{t - 1}]$. Furthermore, we define the known probability conditioned on the depth as $q_t^{\text{old}}(k) = \Pr[c_t = \text{known} \mid \text{dep}(c_t) = k, O_{t - 1}]$. Since nodes of depth $0$ belongs to $C_{\text{init}}$, we have $\forall t. q_t^{\text{old}}(0) = 1$. In the early stages of the refinement process, MCTS performs breadth-wise exploration rather than depth-wise exploration, so $q_t^{\text{dep}}(0)$ remains high. However, as the process proceeds, MCTS shifts its exploration from breadth-wise to depth-wise, and $q_t^{\text{dep}}(0)$ decreases accordingly. Based on this observation, we model $q_t^{\text{dep}}(0)$ as
\begin{align}
    q_t^{\text{dep}}(0) = \exp \left(- \gamma (t - 1)\right) \quad (t \ge 1),
\end{align}
where $\gamma$ is a small positive constant that depends on factors such as the model, the dataset, and the number of initial programs.

In SFS, the hypothesis space surviving at time $t$ can be expressed as
\begin{align}
    \mathcal{V}_0 = \bigcup_{c \in C_{\text{init}}} H_c, \quad
    \mathcal{V}_t = 
    \begin{cases}
        \mathcal{V}_{t-1} \cup H_{c_t}(e_t) & (c_t = \text{unknown})\\
        (\mathcal{V}_{t-1} \cap H_{c_t}(e_t)) \cup (\mathcal{V}_{t-1} \setminus H_{c_t}) & (c_t = \text{known})
    \end{cases}
    \, (t \ge 1).
\end{align}

Correspondingly, the true hypothesis space at time $t$ can be expressed as
\begin{align}
    \mathcal{T}_0 = \bigcup_{c \in C_{\text{init}}} H_c^{\star}, \quad
    \mathcal{T}_t =
    \begin{cases}
        \mathcal{T}_{t - 1} \cup H_{c_t}^{\star} & (c_t = \text{unknown}) \\
        \mathcal{T}_{t - 1} & (c_t = \text{known})
    \end{cases}
    \, (t \ge 1).
\end{align}

In addition to Assumption~\ref{asm:soundness}, we make the following assumption for analyzing SFS.
\begin{assumption}
    \label{asm:realizability-of-sfs}
    (Probabilistic realizability)
    For each initial program $c \in C_{\text{init}}$, define the random variable $Y_c = \mathbbm{1} [H_c^{\star} \neq \varnothing] \in \{0, 1\}$. For the refinement target $c_t \in C$, define the random variable $Z_{c_t} = \mathbbm{1} [H_{c_t}^{\star} \neq \varnothing] \in \{0, 1\}$. We assume the following properties of the LLM's sampling distribution and MCTS.
    \begin{enumerate}
        \item (Uniform lower bound) There exists $q^{\text{SFS}} \in (0, 1)$ such that $\Pr[Y_c = 1] \ge q^{\text{SFS}}$ for every $c \in C_{\text{init}}$. There also exists $\pi^{\text{SFS}} \in (0, 1)$ such that $\Pr[Z_{c_t} = 1 \mid c_t = \text{unknown}] \ge \pi^{\text{SFS}}$ for every time $t$.
        \item (Independence) $\{Y_c\}_{c \in C_{\text{init}}}$ are conditionally independent given the prompt. Also, for different programs $c \neq c'$, $Z_{c}$ and $Z_{c'}$ are mutually independent.
        \item (Upper bound on the known probability at depth $k \ge 1$) There exists $\theta^{\text{SFS}} \in (0, 1)$ such that $q_t^{\text{old}}(k) \le 1 - \theta^{\text{SFS}}$ for any time $t$ and any depth $k \, (\ge 1)$.
    \end{enumerate}
\end{assumption}

This assumption captures the realizability of the LLM and MCTS, requiring a uniform positive bound $q^{\text{SFS}}$ (resp. $\pi^{\text{SFS}}$) on the probability that an initial (resp. unknown) program admits a true hypothesis and a uniform upper bound $1 - \theta^{\text{SFS}}$ on the probability that a program at depth $k \, (\ge 1)$ is already known. Once the true hypothesis space $\mathcal{T}_t$ first becomes non-empty, the true hypothesis space is never pruned, so it remains non-empty at all subsequent times. We therefore define the first time at which $\mathcal{T}_t$ becomes non-empty as $t_c  = \min\{t \ge 0 \mid \mathcal{T}_t \neq \varnothing\}$. Then, we obtain the following theorem, which characterizes the refinement process of SFS.
\begin{theorem}
    \label{thm:theorem-of-sfs}
    Under Assumptions~\ref{asm:soundness} and~\ref{asm:realizability-of-sfs}, for any history $O_t = ((h_1, e_1),(h_2, e_2), \cdots, (h_t, e_t))$ and time $t_c = \min\{t \ge 0 \mid \mathcal{T}_t \neq \varnothing\}$, the following properties hold.
    \begin{enumerate}
        \item ((1 - $\delta_{\text{SFS}}$)-monotonic non-increasing) \\
        At each time $t$,
        \begin{itemize}
            \item With probability $\sum_{k = 1}^{K_{t - 1}} q_t^{\text{dep}}(k) \left(1 - q_t^{\text{old}}(k)\right)$, the refinement target $c_t$ is unknown, in which case
            $\mathcal{V}_t \supseteq \mathcal{V}_{t - 1}$.
            \item With probability $\sum_{k = 0}^{K_{t - 1}} q_t^{\text{dep}}(k) q_t^{\text{old}}(k)$, the refinement target $c_t$ is known, in which case $\mathcal{V}_t \subseteq \mathcal{V}_{t - 1}$. 
        \end{itemize}
        Moreover, with probability at least
        \begin{align}
            1 - \delta_{\text{SFS}} = \exp \left(- \gamma \frac{t(t - 1)}{2}\right),
        \end{align}
        all refinement targets $c_1, \ldots, c_t$ are known,, in which case $\mathcal{V}_t \subseteq \mathcal{V}_{t - 1} \subseteq \cdots \subseteq \mathcal{V}_0$.
        \item (Preservation of the true hypothesis) $\mathcal{T}_t \subseteq \mathcal{V}_t \quad (t \ge t_c)$.
        \item (Non-emptiness)
        \begin{align}
            \Pr[\mathcal{V}_t \neq \varnothing] \ge 1 - (1 - q^{\text{SFS}})^N \exp \left( - \pi^{\text{SFS}} \theta^{\text{SFS}} \left(t - \frac{1 - e^{- \gamma t}}{1 - e^{- \gamma}} \right) \right).
        \end{align}
    \end{enumerate}
\end{theorem}

The proof of Theorem~\ref{thm:theorem-of-sfs} is provided in Appendix~\ref{subapp:proof-of-thm:theorem-of-sfs}. This theorem shows that there is always a lower bound on the probability that the hypothesis space is monotonically non-increasing. Since this lower bound becomes larger as $t$ decreases, the result suggests that SFS can properly narrow down the hypothesis space when the number of refinement steps is relatively small. In such cases, MCTS prioritizes breadth-wise exploration, so the refinement targets are drawn from $C_{\text{init}}$ with high probability. This means that SFS exhibits refinement behavior similar to that of IRRI. Moreover, the lower bound on the probability of non-emptiness converges to $1$ as $N \rightarrow \infty$ and $t \rightarrow \infty$. This shows that increasing the number of initial programs generated has the benefit of increasing the probabilities that $\mathcal{T}_0 \neq \varnothing$ and $\mathcal{V}_t \neq \varnothing$. Further discussion is provided in Appendices~\ref{subapp:further-discussion-of-sfs} and~\ref{subapp:parameter-and-lower-bound-estimation-in-sfs}.

\section{Experiments}
We empirically investigate the theoretical properties of SFS and the implications of our version-space analysis.
\begin{itemize}
    \item Section~\ref{subsec:empirical-analysis-of-tree-structures-in-sfs} and~\ref{subsec:inference-scaling} examine the tree structure of SFS and inference scaling to validate the theoretical characterization presented in Section~\ref{subsec:theoretical-analysis-of-sfs}, which shows that SFS exhibits refinement behavior closer to IRRI than to Line.
    \item Section~\ref{subsec:diversity-of-repair-instructions} investigates the diversity of repair instructions. Our theoretical analysis shows that the hypothesis space is non-increasing for IRRI and probabilistically non-increasing for Line and SFS. We examine whether this property is reflected in the diversity of repair instructions across refinement steps.
\end{itemize}

\subsection{Empirical analysis of tree structures in SFS}
\label{subsec:empirical-analysis-of-tree-structures-in-sfs}

We analyzed the search trees produced by SFS with gpt-4o-mini as the base model across four datasets: HumanEval \citep{chen2021evaluating}, MBPP \citep{austin2021program}, APPS \citep{hendrycks2021measuring} and CodeContests \citep{li2022competition}. We configured SFS to generate $5$ initial programs and perform up to $11$ revision steps.

We measured the distribution of the maximum depth (Table~\ref{tab:max-depth-in-sfs}) and the distribution of the depth at which the first correct programs appears for problems where a correct program is found (Table~\ref{tab:first-correct-depth-in-sfs}). These results suggest that, in relatively short refinement processes, SFS behaves like IRRI by prioritizing breadth-wise exploration over depth-wise exploration, thereby enabling it to reach the correct program.

\begin{table}[htbp]
    \caption{\textbf{Maximum depth in SFS.}}
    \label{tab:max-depth-in-sfs}
    \small
    \centering
    \begin{tabular}{c|cccccc}
        \toprule
        \textbf{Dataset} & \textbf{Pass@1} & \textbf{Pass@16} & \textbf{Depth $=$ 0} & \textbf{Depth $=$ 1} & \textbf{Depth $\ge$ 2} \\
        \midrule
        HumanEval & 88.13\% & 95.00\% & 58.12\% & 33.12\% & 8.76\% \\
        MBPP & 77.58\% & 86.65\% & 55.92\% & 30.23\% & 13.85\% \\
        APPS & 25.00\% & 39.50\% & 8.00\% & 61.00\% & 31.00\% \\
        CodeContests & 7.88\% & 18.79\% & 1.82\% & 72.73\% & 25.45\% \\
        \bottomrule
    \end{tabular}
\end{table}

\begin{table}[htbp]
    \caption{\textbf{Depth at which SFS first reaches a correct program.}}
    \label{tab:first-correct-depth-in-sfs}
    \small
    \centering
    \begin{tabular}{c|ccccc}
        \toprule
        \textbf{Dataset} & \textbf{Pass@1} & \textbf{Pass@16} & \textbf{Depth $=$ 0} & \textbf{Depth $=$ 1} & \textbf{Depth $\ge$ 2} \\
        \midrule
        HumanEval & 88.13\% & 95.00\% & 93.38\% & 6.62\% & 0.00\% \\
        MBPP & 77.58\% & 86.65\% & 91.57\% & 7.27\% & 1.16\% \\
        APPS & 25.00\% & 39.50\% & 75.32\% & 24.68\% & 0.00\% \\
        CodeContests & 7.88\% & 18.79\% & 58.07\% & 41.93\% & 0.00\% \\ 
        \bottomrule
    \end{tabular}
\end{table}

\subsection{Inference scaling}
\label{subsec:inference-scaling}

\begin{table}[htbp]
    \caption{\textbf{Detailed settings of all evaluated methods.}}
    \label{tab:detailed-settings-of-methods}
    \small
    \centering
    \begin{tabularx}{\linewidth}{c|X}
        \toprule
        \textbf{Method} & \textbf{Description} \\
        \midrule
        BoN & Generates $16$ programs from the same prompt. \\
        \hline
        Line & Generates $1$ initial program and performs up to $15$ refinement steps (Section~\ref{subsec:introduce-line}). \\
        \hline
        Tree & Building upon the framework of \citet{zhou2023language}, generates $1$ initial program followed by up to $15$ refinement steps. \\
        \hline
        SFS & Generates $5$ initial programs and performs up to $11$ revision steps (Section~\ref{subsec:introduce-sfs} and Appendix~\ref{app:details-of-sfs}). Generates two repair instructions at each revision step. \\
        \hline
        IRRI & Generates $5$ initial programs with up to $11$ revision steps (Section~\ref{subsec:introduce-irri} and Appendix~\ref{app:overview-of-irri}). Generates two repair instructions at each revision step. \\
        \bottomrule
    \end{tabularx}
\end{table}

To examine whether IRRI can serve as a viable alternative to SFS, we compared the inference scaling performance of SFS and IRRI using gpt-4o-mini, gpt-4.1-mini, Llama-3.1-8B-Instruct \citep{grattafiori2024llama}, and Llama-3.2-3B-Instruct \citep{grattafiori2024llama} as base models, on several program generation benchmarks, including HumanEval \citep{chen2021evaluating}, MBPP \citep{austin2021program}, and APPS \citep{hendrycks2021measuring}. We used 160 problems from HumanEval and 397 problems from MBPP. For APPS, we randomly sampled 200 problems due to budget constraints. The detailed settings of all evaluated methods are provided in Table~\ref{tab:detailed-settings-of-methods}.

Figure~\ref{fig:scaling-experiments} shows the results on APPS obtained using gpt-4o-mini and the results on MBPP obtained using Llama-3.2-3B-Instruct. The results for the other datasets and base models are reported in Appendix~\ref{app:additional-results-for-inference-scaling}. Each curve represents the mean performance over five runs, and each shaded region denotes 95\% confidence intervals computed using the $t$-distribution. We found that IRRI achieves inference scaling performance comparable to or better than that of prior methods. Notably, although IRRI is simpler, it attains inference performance on par with SFS.

\begin{figure}[htbp]
    \centering
    \begin{minipage}{0.49\linewidth}
        \centering
        \includegraphics[width=\linewidth]{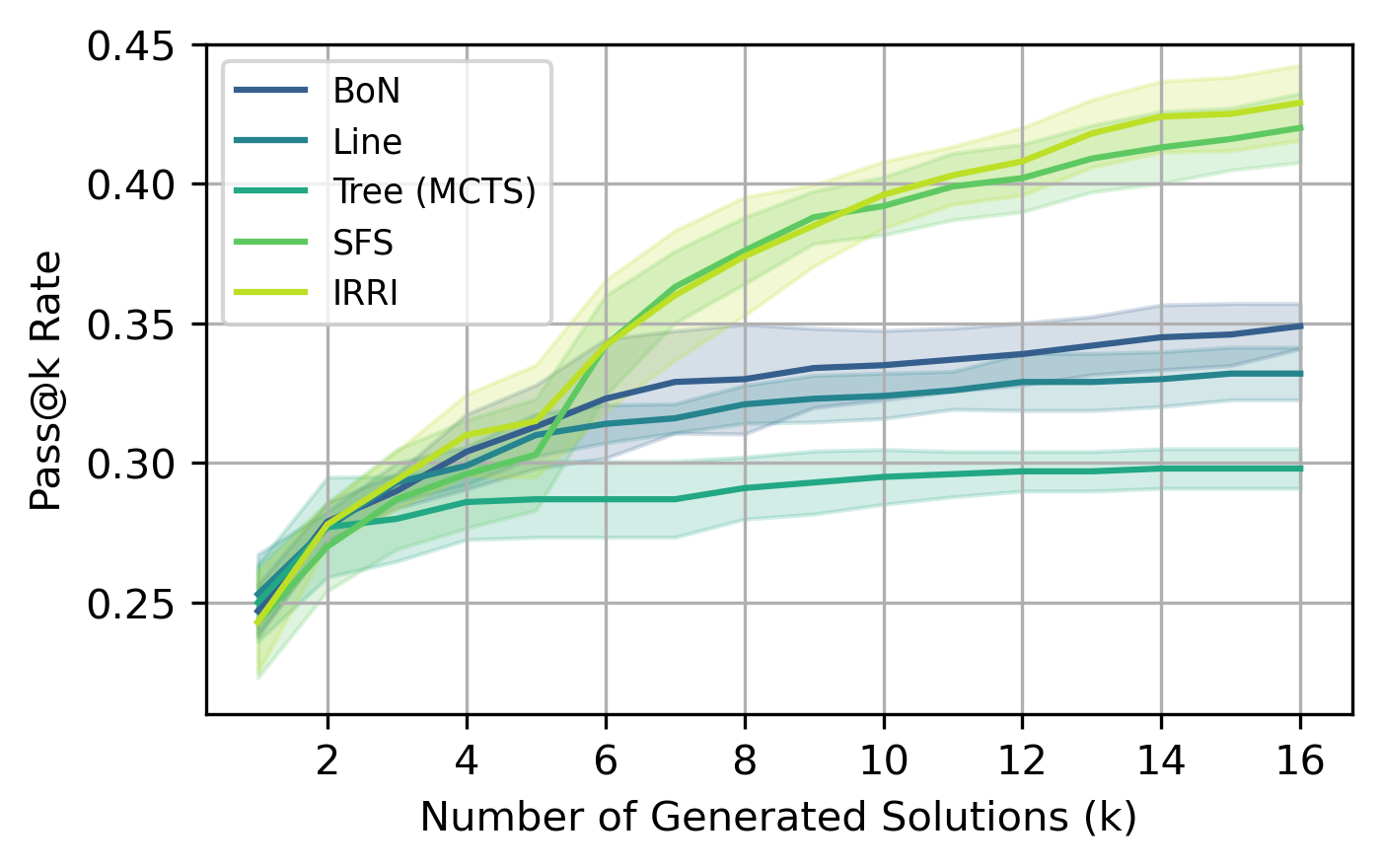}
    \end{minipage}
    \begin{minipage}{0.49\linewidth}
        \centering
        \includegraphics[width=\linewidth]{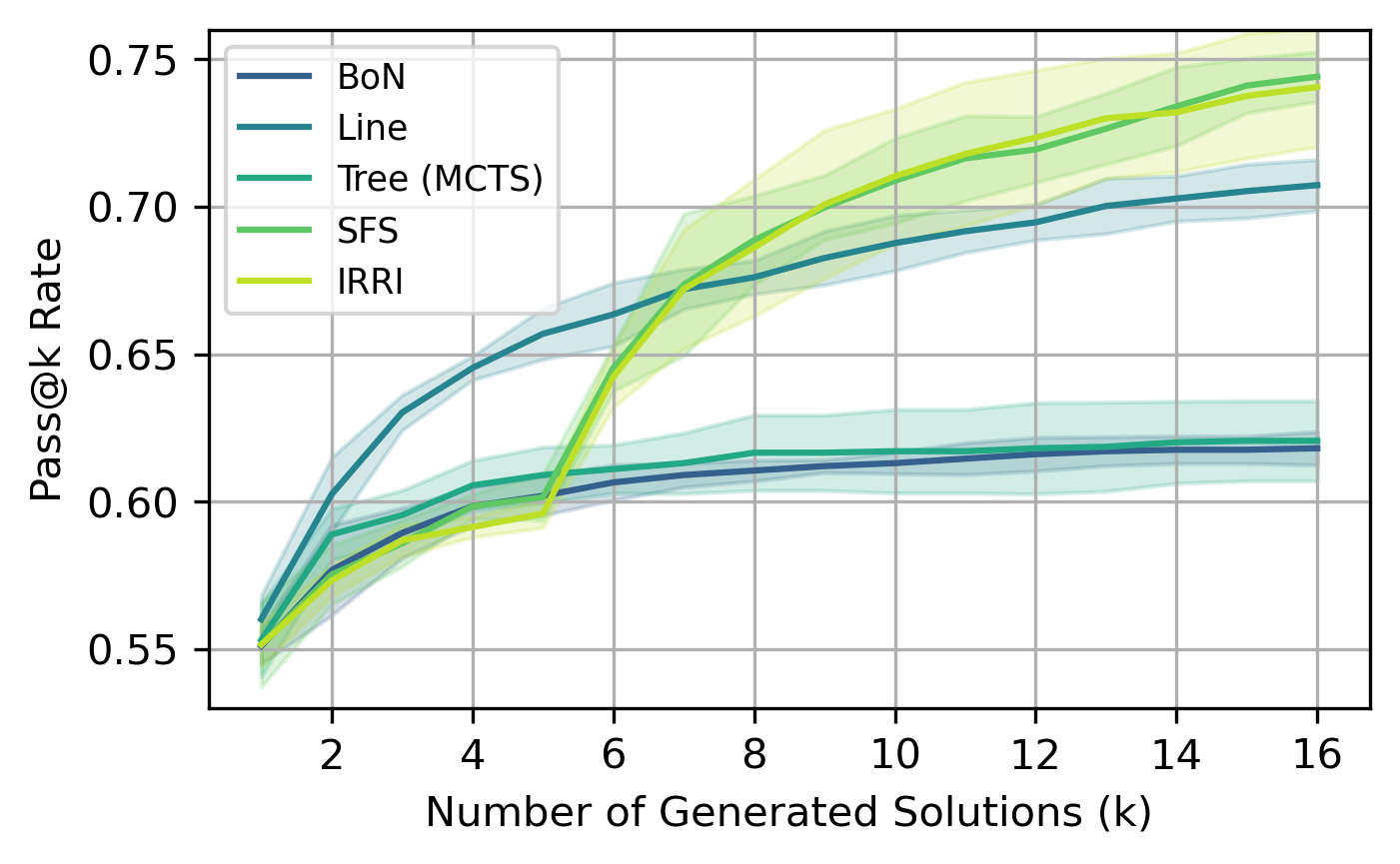}
    \end{minipage}
    \caption{\textbf{Scaling curves for different methods on APPS using gpt-4o-mini (left) and on MBPP using Llama-3.2-3B-Instruct (right).}}
    \label{fig:scaling-experiments}
\end{figure}

\subsection{Diversity of repair instructions}
\label{subsec:diversity-of-repair-instructions}
Theoretically, we show that the hypothesis space is non-increasing for IRRI and probabilistically non-increasing for Line and SFS. A non-increasing hypothesis space may suppress repeated repairs based on the same failures and provide more opportunities to explore different repair candidates, potentially leading to greater repair diversity. To examine this theoretical implication, we measured the cosine similarity between repair directions across different refinement steps. We used the ‘all-MiniLM-L6-v2’ model of SentenceTransformers \citep{reimers2019sentence} to obtain embeddings for these repair instructions and visualized the results as heatmaps. The results, as shown in Figures~\ref{fig:diversity-llama-3.1-8b-instruct-humaneval} and~\ref{fig:diversity-llama-3.1-8b-instruct-mbpp}, indicate that IRRI and SFS exhibit relatively lower cosine similarities than Line, despite requiring fewer repair steps, indicating that they generate more diverse repair candidates.

\begin{figure}[htbp]
    \centering
    \begin{minipage}{0.32\linewidth}
        \centering
        \includegraphics[width=\linewidth]{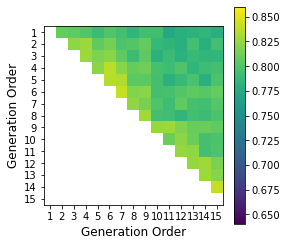}
        \subcaption{Line}
    \end{minipage}
    \begin{minipage}{0.32\linewidth}
        \centering
        \includegraphics[width=\linewidth]{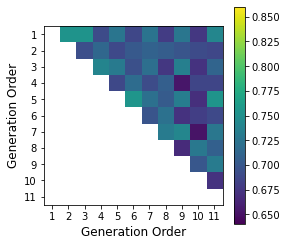}
        \subcaption{IRRI}
    \end{minipage}
    \begin{minipage}{0.32\linewidth}
        \centering
        \includegraphics[width=\linewidth]{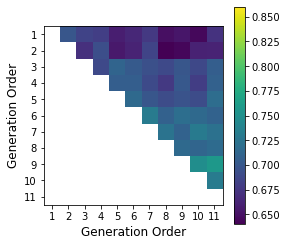}
        \subcaption{SFS}
    \end{minipage}
    \caption{\textbf{BERT similarity heatmaps on HumanEval using Llama-3.1-8B-Instruct.}}
    \label{fig:diversity-llama-3.1-8b-instruct-humaneval}
\end{figure}

\begin{figure}[htbp]
    \centering
    \begin{minipage}{0.32\linewidth}
        \centering
        \includegraphics[width=\linewidth]{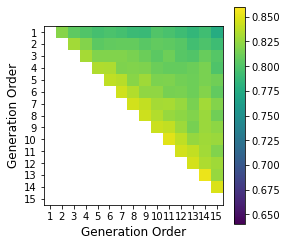}
        \subcaption{Line}
    \end{minipage}
    \begin{minipage}{0.32\linewidth}
        \centering
        \includegraphics[width=\linewidth]{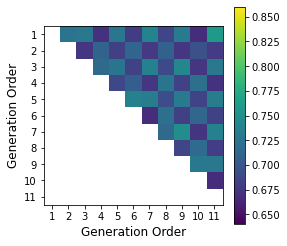}
        \subcaption{IRRI}
    \end{minipage}
    \begin{minipage}{0.32\linewidth}
        \centering
        \includegraphics[width=\linewidth]{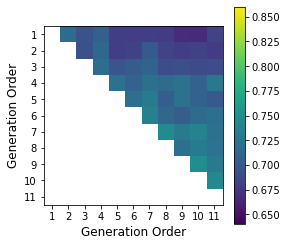}
        \subcaption{SFS}
    \end{minipage}
    \caption{\textbf{BERT similarity heatmaps on MBPP using Llama-3.1-8B-Instruct.}}
    \label{fig:diversity-llama-3.1-8b-instruct-mbpp}
\end{figure}

\section{Conclusion}
We have introduced two contrasting refinement methods, Line and IRRI, which progress the refinement process in the depth direction and the width direction, respectively, and used them as a basis for theoretically analyzing SFS from the perspective of version spaces in learning theory. Our theoretical analysis reveals that SFS exhibits refinement behavior closer to IRRI than to Line, and this theoretical characteristic is also confirmed experimentally. Experimental results comparing computational costs demonstrate that a simpler method, IRRI, may achieve a similar refinement process without relying on SFS's complex correction mechanism. These findings provide new insights into the mechanisms underlying improved inference performance in multi-turn program correction and suggest promising directions for future research.

\section*{AI use statement}
In this work, we used generative AI tools to help develop theoretical models or conceptual frameworks and to assist with translation. We have not used generative AI tools to formulate mathematical claims, provide critical ingredients for proving mathematical claims, assist in the writing of proofs, propose or refine hypotheses, design or provide feedback on research methodology or experiments, implement methods, and interpret results. Generating synthetic datasets, cleaning and reformatting datasets, and supporting qualitative and thematic data analysis are not applicable to this work. Additionally, we used generative AI tools to create or edit software code, create artifacts, draft parts of a research paper, summarize or analyze existing literature, brainstorm, search for information, edit a research paper to improve readability, identify relevant literature, and propose a title or keywords for a research paper. We have reviewed all AI-assisted work. AI-assisted brainstorming, including help developing theoretical models, was critically evaluated by the authors and cross-checked against relevant prior literature to confirm correctness. AI-assisted translations were reviewed and corrected by the authors. All AI-generated or AI-edited code was also tested and verified for correctness by the authors. AI-drafted text, summaries of literature, and suggested titles/keywords were rewritten and fact-checked by the authors, with all citations verified against original sources. We take responsibility for the final content of this work, including text, claims or artifacts produced with the aid of generative AI.

\bibliography{myref}
\bibliographystyle{iclr2027_conference}

\appendix

\section{Overview of problem setting}
\label{app:overview-of-problem-setting}

Figure~\ref{fig:overview-of-problem-setting} illustrates an overview of the problem setting described in Section~\ref{subsec:problem-setting}.

\begin{figure}[htbp]
    \centering
    \includegraphics[width=\linewidth]{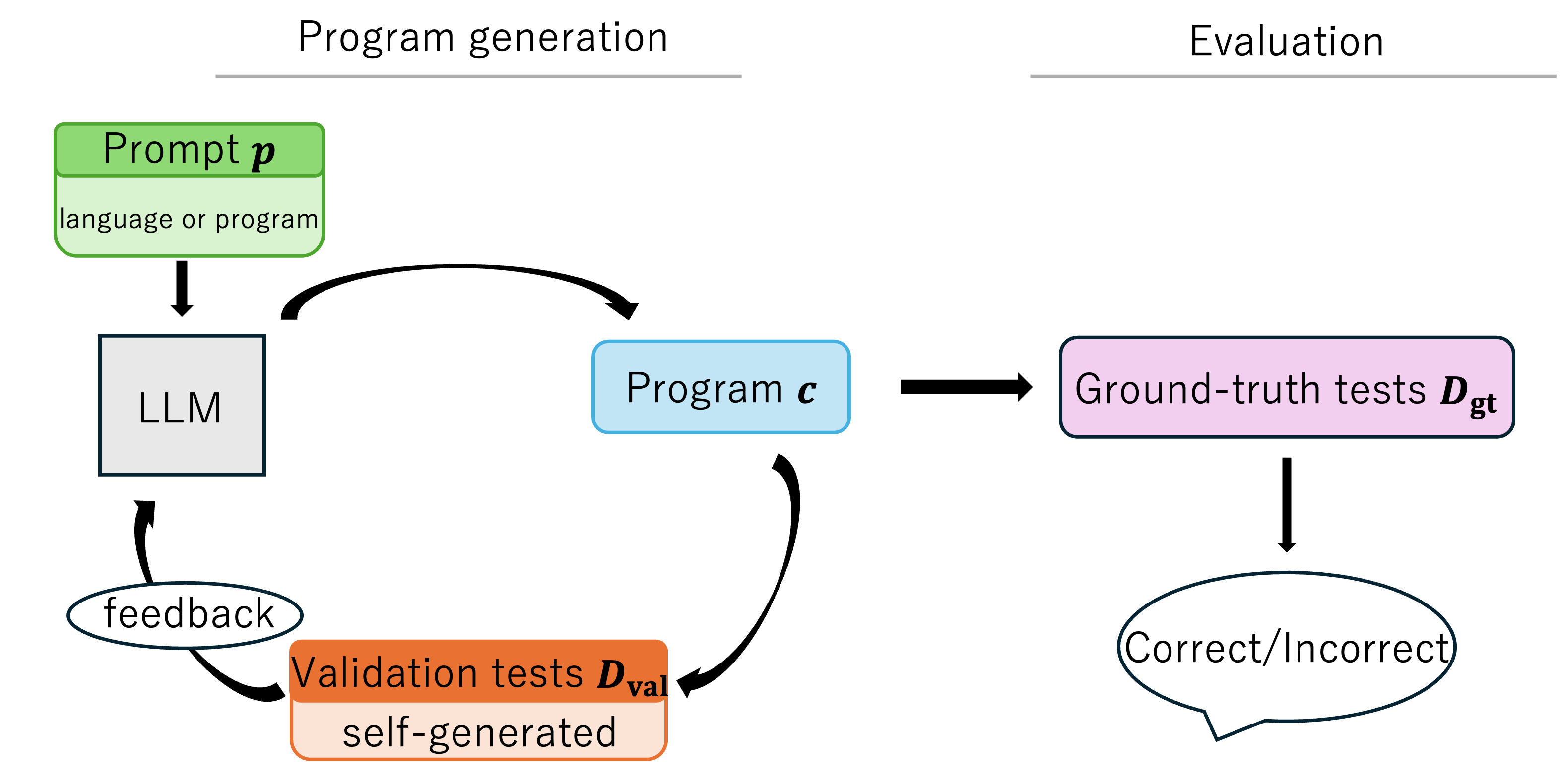}
    \caption{\textbf{Overview of problem setting.} Adapted from \citet{light2025sfs}. Given a prompt $p$, an LLM iteratively generates and refines programs $c$ on the basis of feedback from a validation test-case set $D_{\text{val}}$. The objective is to generate a program $c_{\text{true}}$ that passes all ground-truth test cases $D_{\text{gt}}$.}
    \label{fig:overview-of-problem-setting}
\end{figure}

\FloatBarrier
\section{Details of Scattered Forest Search (SFS)}
\label{app:details-of-sfs}

To enable diverse exploration of the solution space in program generation, \citet{light2025sfs} have proposed \textsc{Scattered Forest Search} (SFS). SFS builds on MCTS and incorporates techniques such as careful seed initialization and text-based optimization. This method encourages broader exploration of the solution space by combining LLM-driven repair instructions, the introduction of multiple initial programs, and the utilization of feedback and past search experience. Specifically, SFS consists of three main components: \textsc{Scattering}, \textsc{Foresting}, and \textsc{Scouting}, as illustrated in Figure~\ref{fig:overview-of-sfs}. 

\textbf{\textsc{Scattering}}: This component dynamically varies input prompts with different repair instructions when sampling revised programs. More concretely, given a program selected by MCTS along with its associated feedback, an LLM generates multiple repair instructions. Then, one of them is selected and incorporated into the prompt for generating the revised program.

\textbf{\textsc{Foresting}}: This component generates multiple initial programs that act as starting points for MCTS. In this process, initial programs are generated using either the same base prompt or diversified prompts obtained by randomly appending instruction statements.

\textbf{\textsc{Scouting}}: This component provides feedback on the repair instruction used during the refinement process and distills general insights that can be shared across search branches. These insights are incorporated into the prompt as auxiliary information when an LLM generates reflections and repair instructions. This enables the promotion of effective repair instructions and the suppression of unproductive ones.

\begin{figure}[htbp]
    \centering
    \includegraphics[width=\linewidth]{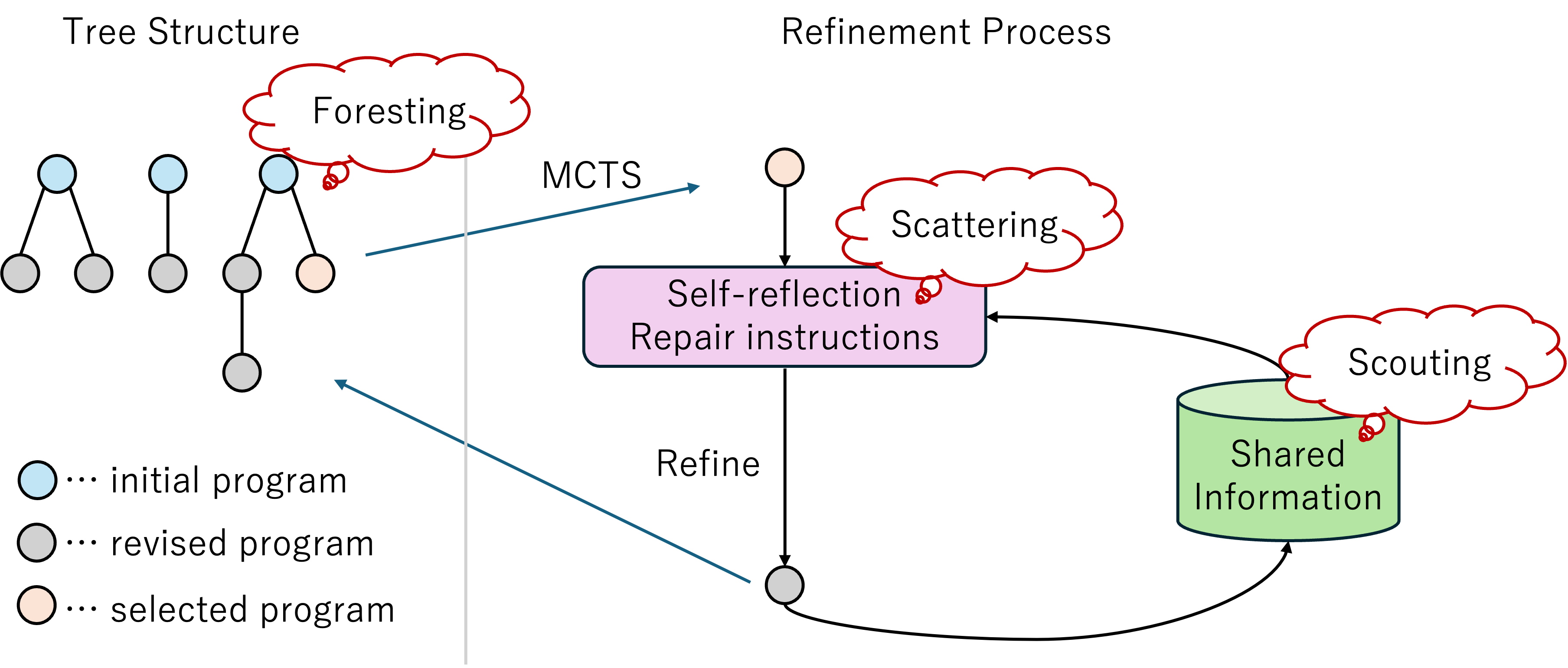}
    \caption{\textbf{Overview of SFS.} Adapted from \citet{light2025sfs}. SFS builds on MCTS and incorporates techniques such as careful seed initialization and text-based optimization. \textsc{Scattering} diversifies refinements via dynamic prompt variation, \textsc{Foresting} provides multiple starting programs for MCTS, and \textsc{Scouting} propagates insights about repair instructions across the search process.}
    \label{fig:overview-of-sfs}
\end{figure}

\FloatBarrier
\section{Overview of Iterative Refinement of Repair Instructions (IRRI)}
\label{app:overview-of-irri}

Algorithm~\ref{alg:irri} presents the pseudo code of IRRI, following the description in Section~\ref{subsec:introduce-irri}. Figure~\ref{fig:overview-of-irri} illustrates an overview of IRRI, and Figure~\ref{fig:difference-in-refinement-process} shows the difference between existing methods and IRRI.

\begin{algorithm}
    \caption{\textsc{Iterative Refinement of Repair Instructions} (IRRI)}
    \label{alg:irri}
    \begin{algorithmic}[1]
        \Require LLM $\mathcal{M}$, program generation task $x = \langle p, D_{\text{gt}} \rangle$
        \Require Number of initial programs $n$
        \Require Number of repair instructions generated per step $m$
        \Require Maximum number of refinements $T_{\max}$

        \State Generate a validation test-case set $D_{\text{val}} \gets \mathcal{M}(p)$
        
        \State \textcolor{darkblue}{\textbf{Initial Generation}}
        \State $i \gets 1$
        \For{$i \le n$}
            \State Generate initial program $c_i \gets \mathcal{M}(p)$
            \If {$c_i$ passes all validation test cases $D_{\text{val}}$}
                \State \textcolor{purple}{\textbf{Terminate the inference process}}
            \EndIf
            \State $i \gets i + 1$
        \EndFor

        \State \textcolor{darkblue}{\textbf{Refinement Process}}
        \State Shared information $I_1 \gets \varnothing$
        \State $t \gets 1$
        \While {$t \le T_{max}$}
            \For {$c_t \in \{c_1, \ldots, c_n\}$}
                \State Generate repair instructions $r_1, \ldots, r_m \gets \mathcal{M}(p, c_t, I_t)$
                \For {$r_t \in \{r_1, \ldots, r_m\}$}
                    \State Refine program $c_{\text{ref}} \gets \mathcal{M}(p, c_t, r_t)$
                    \If {$c_{\text{ref}}$ passes all validation test cases $D_{\text{val}}$}
                        \State \textcolor{purple}{\textbf{Terminate the inference process}}
                    \EndIf
                    \State $t \gets t + 1$
                    \If {$t > T_{max}$}
                        \State \textcolor{purple}{\textbf{Terminate the inference process}}
                    \EndIf
                    \State \textcolor{darkblue}{\textbf{Feedback for Repair Instructions}}
                    \State Update shared information $I_{t + 1} \gets \mathcal{M}(p, c_t, I_t, r_t, c_{\text{ref}})$
                \EndFor
            \EndFor
        \EndWhile
    \end{algorithmic}
\end{algorithm}

\begin{figure}[htbp]
    \centering
    \includegraphics[width=\linewidth]{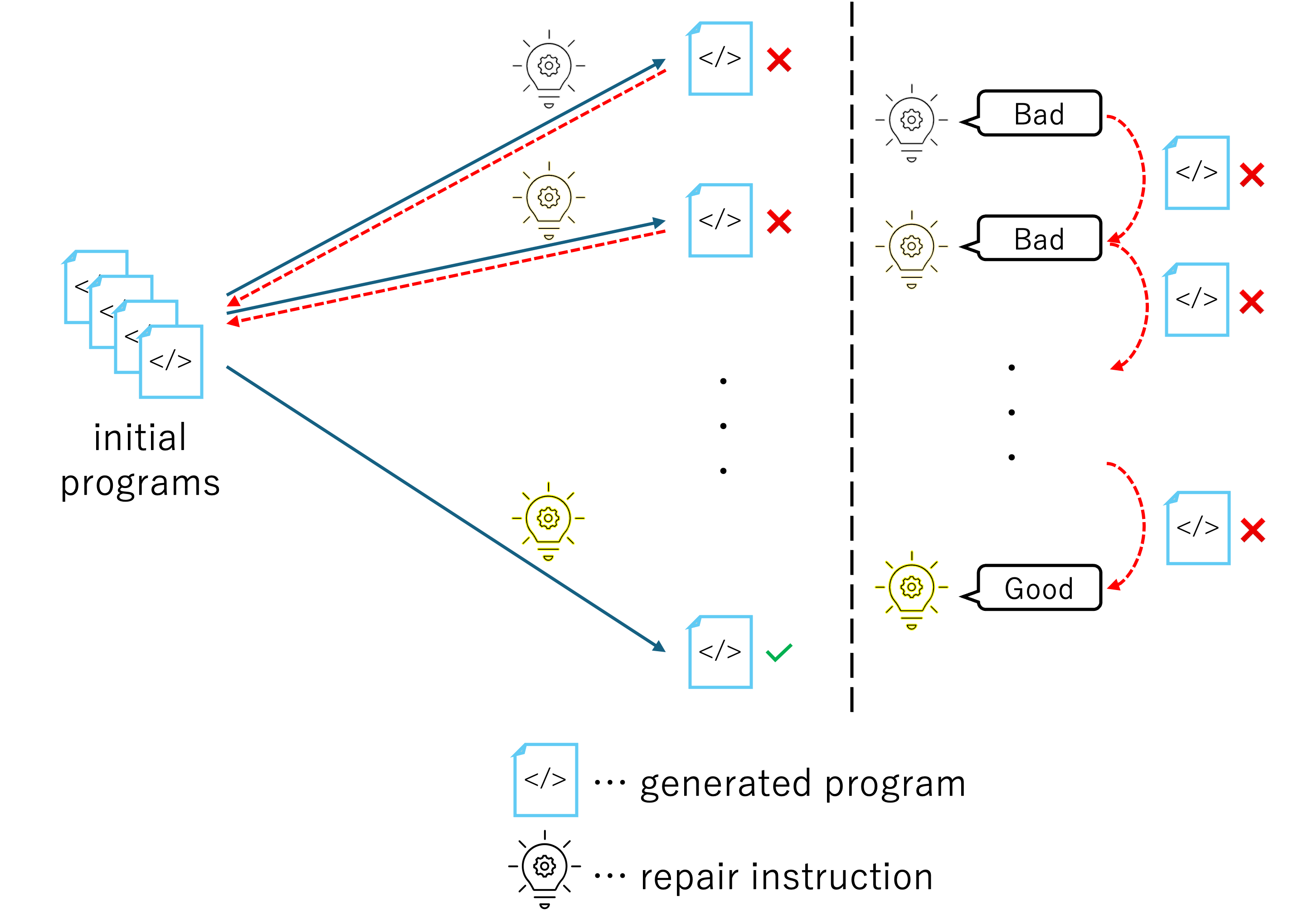}
    \caption{\textbf{Overview of IRRI.} To capture the characteristics of refinement in the width direction, we introduce a self-correction method that refines initial programs through iterative feedback on repair instructions.}
    \label{fig:overview-of-irri}
\end{figure}

\begin{figure}[htbp]
    \centering
    \begin{subfigure}[b]{0.49\linewidth}
        \centering
        \includegraphics[width=\linewidth]{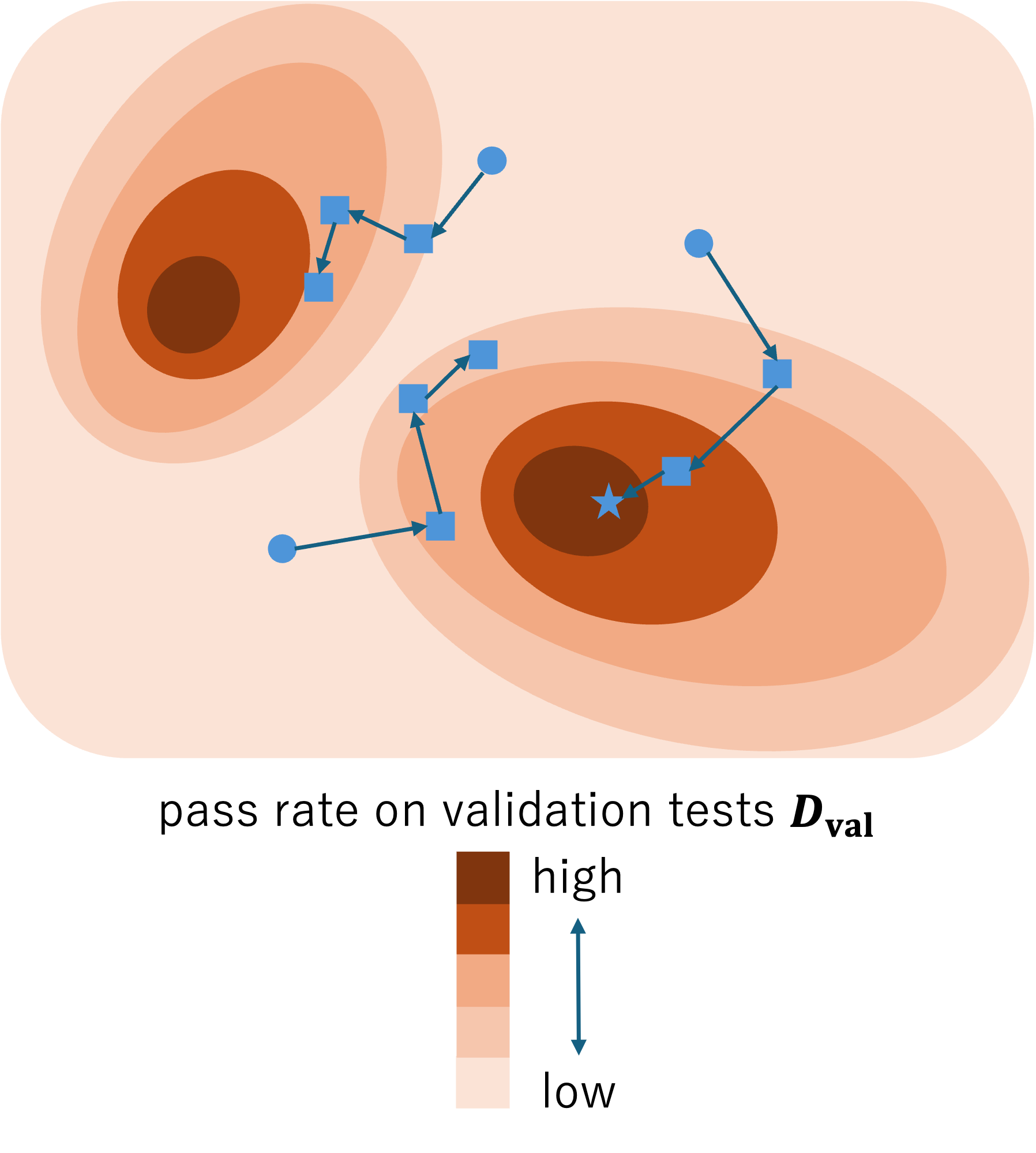}
        \caption{Existing methods}
        \label{fig:refinement-of-existing-methods}
    \end{subfigure}
    \begin{subfigure}[b]{0.49\linewidth}
        \centering
        \includegraphics[width=\linewidth]{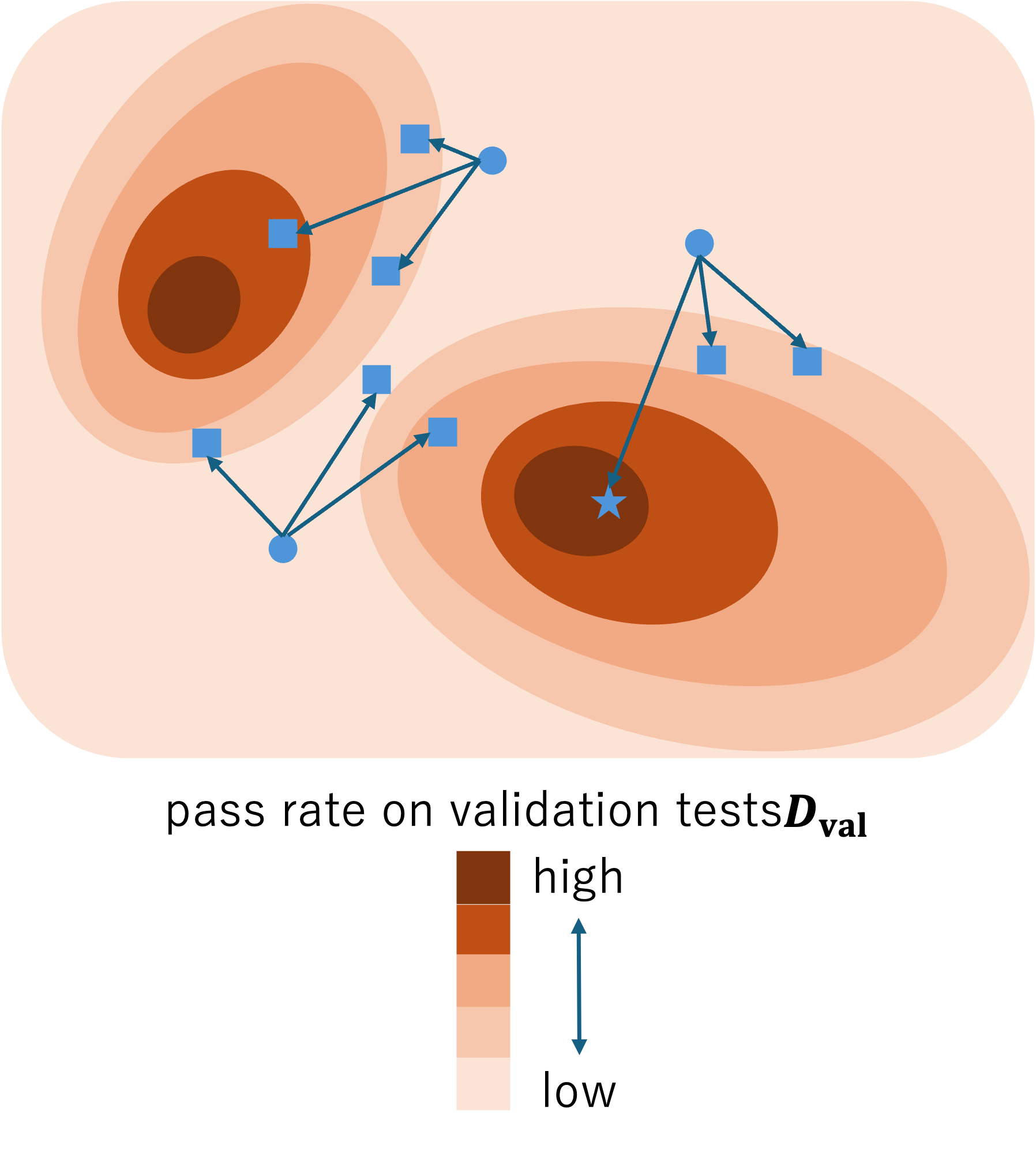}
        \caption{IRRI}
        \label{fig:refinement-of-irri}
    \end{subfigure}
    \caption{\textbf{Difference in the refinement process between existing methods and IRRI.} Circles denote initial programs, squares denote revised programs, and stars denote correct programs.}
    \label{fig:difference-in-refinement-process}
\end{figure}

\FloatBarrier
\section{Proofs}
\subsection{Proof of Theorem~\ref{thm:theorem-of-line}}
\label{subapp:proof-of-thm:theorem-of-line}

Using Assumption~\ref{asm:realizability-of-line}, the following lemma holds.
\begin{lemma}
    \label{lem:non-empty-of-line}
    Under Assumption~\ref{asm:realizability-of-line},
    \begin{align}
        \Pr[t \ge t_c] \ge 1 - (1 - q^{\text{Line}})\exp \left( - \pi^{\text{Line}} A_{\tau} \left(1 - e^{- \tau (t - 1)}\right) \right) \quad \left(A_{\tau} = \frac{e^{- 2 \tau}}{1 - e^{- \tau}}\right).
    \end{align}
\end{lemma}

\begin{proof}
    Let $\mathcal{Z}_t = \{c_t = \text{unknown}\} \cap \{ Z_{c_t} = 1 \}$. Then, under the condition $t < t_c$, its complement $\mathcal{Z}_t^c$ is given by $\mathcal{Z}_t^c = \{ Z_{c_t} = 0 \}$. Consider the probability $\Pr[t < t_c]$.
    \begin{align}
        \Pr[t < t_c] &= \Pr[\mathcal{Z}_1^c \cap \mathcal{Z}_2^c \cap \cdots \cap \mathcal{Z}_t^c \mid O_t] \\
        &= \Pr[\mathcal{Z}_1^c] \prod_{s = 2}^{t} \Pr[\mathcal{Z}_s^c \mid O_{s - 1}, \mathcal{Z}_1^c \cap \mathcal{Z}_2^c \cap \cdots \cap \mathcal{Z}_{s - 1}^c] \\
        &= \Pr[\mathcal{Z}_1^c] \prod_{s = 2}^{t} \left( 1 - \Pr[\mathcal{Z}_s \mid O_{s - 1}, \mathcal{Z}_1^c \cap \mathcal{Z}_2^c \cap \cdots \cap \mathcal{Z}_{s - 1}^c] \right) \\
        &= \Pr[H_{c_1}^{\star} = \varnothing] \prod_{s = 2}^{t} \left( 1 - \Pr[Z_{c_s} = 1, c_s = \text{unknown} \mid O_{s - 1}, Z_{c_1} = \cdots = Z_{c_{s - 1}} = 0] \right).
    \end{align}
    For each term, since $c_s$ is an unknown program, the independence in Assumption~\ref{asm:realizability-of-line} implies that
    \begin{align}
        \Pr[Z_{c_s} = 1, c_s = \text{unknown} \mid O_{s - 1}, Z_{c_1} = \cdots = Z_{c_{s - 1}} = 0] = \Pr[Z_{c_s} = 1, c_s = \text{unknown} \mid O_{s - 1}].
    \end{align}
    Hence,
    \begin{align}
        \Pr[t < t_c] &= \Pr[H_{c_1}^{\star} = \varnothing] \prod_{s = 2}^{t} \left( 1 - \Pr[Z_{c_s} = 1, c_s = \text{unknown} \mid O_{s - 1}] \right) \\
        &= \Pr[H_{c_1}^{\star} = \varnothing] \prod_{s = 2}^{t} \left( 1 - \Pr[Z_{c_s} = 1 \mid c_s = \text{unknown}, O_{s - 1}] \Pr[c_s = \text{unknown} \mid O_{s - 1}] \right) \\
        &= \Pr[H_{c_1}^{\star} = \varnothing] \prod_{s = 2}^{t} \left( 1 - \Pr[Z_{c_s} = 1 \mid c_s = \text{unknown}] (1 - q_s^{\text{old}}) \right) \\
        &= \Pr[H_{c_1}^{\star} = \varnothing] \prod_{s = 2}^{t} \left( 1 - \Pr[Z_{c_s} = 1 \mid c_s = \text{unknown}] \exp(- \tau s) \right).
    \end{align}
    For any time $t$, we have
    \begin{align}
        \Pr[Z_{c_t} = 1 \mid c_t = \text{unknown}] \ge \pi^{\text{Line}}
    \end{align}
    from Assumption~\ref{asm:realizability-of-line}, so for each term
    \begin{align}
        1 - \Pr[Z_{c_s} = 1 \mid c_s = \text{unknown}] \exp(- \tau s) \le 1 - \pi^{\text{Line}} \exp(- \tau s)
    \end{align}
    holds. It follows that
    \begin{align}
        \prod_{s = 2}^{t} \left( 1 - \Pr[Z_{c_s} = 1 \mid c_s = \text{unknown}] \exp(- \tau s) \right) \le \prod_{s = 2}^{t} \left( 1 - \pi^{\text{Line}} \exp(- \tau s) \right),
    \end{align}
    and hence
    \begin{align}
        \Pr[t \ge t_c] &= 1 - \Pr[t < t_c] \\
        &= 1 - \Pr[H_{c_1}^{\star} = \varnothing] \prod_{s = 2}^{t} \left( 1 - \Pr[Z_{c_s} = 1 \mid c_s = \text{unknown}] \exp(- \tau s) \right) \\
        &\ge 1 - (1 - q^{\text{Line}}) \prod_{s = 2}^{t} \left( 1 - \pi^{\text{Line}}\exp(- \tau s) \right).
    \end{align}
    Now, applying the inequality
    \begin{align}
        1 - x \le e^{- x}
    \end{align}
    to each term gives
    \begin{align}
        1 - \pi^{\text{Line}}\exp(- \tau s) \le \exp \left(- \pi^{\text{Line}} e^{- \tau s}\right).
    \end{align}
    Therefore,
    \begin{align}
        \prod_{s = 2}^{t} \left( 1 - \pi^{\text{Line}}\exp(- \tau s) \right) &\le \prod_{s = 2}^{t} \exp \left(- \pi^{\text{Line}} e^{- \tau s}\right) \\
        &= \exp \left( - \pi^{\text{Line}} \sum_{s = 2}^{t} e^{- \tau s} \right) \\
        &= \exp \left( - \pi^{\text{Line}} e^{- 2 \tau} \frac{1 - e^{- \tau(t - 1)}}{1 - e^{- \tau}} \right),
    \end{align}
    so that
    \begin{align}
        \Pr[t \ge t_c] \ge 1 - (1 - q^{\text{Line}}) \exp \left( - \pi^{\text{Line}} A_{\tau} \left(1 - e^{- \tau (t - 1)}\right) \right),
    \end{align}
    where
    \begin{align}
        A_{\tau} = \frac{e^{- 2 \tau}}{1 - e^{- \tau}}.
    \end{align}
\end{proof}

Using this lemma, Theorem~\ref{thm:theorem-of-line} can be proved as follows.
\begin{proof}
    1. At each time $t$, the refinement target is an unknown program with probability $1 - q_t^{\text{old}}$. In this case, regarding the hypothesis space that survives at time $t$,
    \begin{align}
        \mathcal{V}_t = \mathcal{V}_{t-1} \cup H_{c_t}(e_t) \supseteq \mathcal{V}_{t-1}
    \end{align}
    holds. \\
    On the other hand, the refinement target is a known program with probability $q_t^{\text{old}}$. In this case, the hypothesis space that survives at time $t$ is expressed as
    \begin{align}
        \mathcal{V}_t = (\mathcal{V}_{t-1} \cap H_{c_t}(e_t)) \cup (\mathcal{V}_{t-1} \setminus H_{c_t}).
    \end{align}
    For any program $c_t$ and any observation $e_t$,
    \begin{align}
        H_{c_t}(e_t) \subseteq H_{c_t}
    \end{align}
    and thus
    \begin{align}
        \mathcal{V}_{t-1} \cap H_{c_t}(e_t) \subseteq \mathcal{V}_{t-1} \cap H_{c_t}
    \end{align}
    holds. From this,
    \begin{align}
        \mathcal{V}_t &= (\mathcal{V}_{t-1} \cap H_{c_t}(e_t)) \cup (\mathcal{V}_{t-1} \setminus H_{c_t}) \\
        &\subseteq (\mathcal{V}_{t-1} \cap H_{c_t}) \cup (\mathcal{V}_{t-1} \setminus H_{c_t}) \\
        &= \mathcal{V}_{t-1}
    \end{align}
    follows. \\
    Furthermore, the probability that all refinement targets are known programs throughout the refinement process from time $1$ to time $t$ is given by
    \begin{align}
        \Pr[\text{all refinement targets are known}] &= \prod_{s = 1}^{t} \Pr[c_s = \text{known} \mid O_{s - 1}] \\
        &= \prod_{s = 1}^{t} q_s^{\text{old}} \\
        &= \prod_{s = 2}^{t} \left( 1 -  \exp(- \tau s) \right). \\
    \end{align}
    Applying the inequality
    \begin{align}
        1 - x \le e^{- x}
    \end{align}
    to each term gives
    \begin{align}
        1 -  \exp(- \tau s) \le \exp(- e^{- \tau s}).
    \end{align}
    Therefore,
    \begin{align}
        \Pr[\text{all refinement targets are known}] &\le \prod_{s = 2}^{t} \exp(- e^{- \tau s}) \\
        &= \exp \left(- \sum_{s = 2}^{t} e^{- \tau s}\right) \\
        &= \exp \left( - e^{- 2 \tau} \frac{1 - e^{- \tau (t - 1)}}{1 - e^{- \tau}} \right).
    \end{align}
    From this, with probability at most
    \begin{align}
        1 - \delta_{\text{Line}} = \exp \left( - A_{\tau} \left(1 - e^{- \tau (t - 1)}\right) \right) \quad \left( A_{\tau} = \frac{e^{- 2 \tau}}{1 - e^{- \tau}} \right),
    \end{align}
    all refinement targets $c_1, \ldots, c_t$ are known, in which case
    \begin{align}
        \mathcal{V}_t \subseteq \mathcal{V}_{t - 1} \subseteq \cdots \subseteq\mathcal{V}_0.
    \end{align}
    
    2. We prove this by mathematical induction. \\
    First, consider the case $t = t_c$.
    \begin{enumerate}[label=(\roman*)]
        \item Case $t_c = 0$. \\
        Since the true hypothesis space is non-empty at time $0$,
        \begin{align}
            \mathcal{T}_0 = H_{c_1}^{\star} (\neq \varnothing)
        \end{align}
        holds. Also, since
        \begin{align}
            \mathcal{V}_0 = H_{c_1}
        \end{align}
        and
        \begin{align}
            H_{c_1}^{\star} \subseteq H_{c_1},
        \end{align}
        \begin{align}
            \mathcal{T}_0 = H_{c_1}^{\star} \subseteq H_{c_1} = \mathcal{V}_0
        \end{align}
        holds.
        \item Case $t_c \ge 1$. \\
        Since the true hypothesis space first becomes non-empty at time $t_c$,
        \begin{align}
            \mathcal{T}_{t_c} = H_{c_{t_c}}^{\star} (\neq \varnothing)
        \end{align}
        holds. Also, since program $c_{t_c}$ should be unknown,
        \begin{align}
            H_{c_{t_c}}(e_{t_c})\subseteq \mathcal{V}_{t_c}
        \end{align}
        holds. Here, taking any $h^{\star} = (c_{t_c}, r^{\star}) \in H_{c_{t_c}}^{\star}$, by Assumption~\ref{asm:soundness},
        \begin{align}
            h^{\star} \in H_{c_{t_c}}(e_{t_c})
        \end{align}
        holds, and so
        \begin{align}
            H_{c_{t_c}}^{\star} \subseteq H_{c_{t_c}}(e_{t_c}).
        \end{align}
        Therefore,
        \begin{align}
            \mathcal{T}_{t_c} = H_{c_{t_c}}^{\star} \subseteq H_{c_{t_c}}(e_{t_c}) \subseteq \mathcal{V}_{t_c}
        \end{align}
        holds.
    \end{enumerate}
    Next, suppose that preservation of the true hypothesis holds at $t = s \, (s \ge t_c)$; that is,
    \begin{align}
        \mathcal{T}_s \subseteq \mathcal{V}_s.
    \end{align}
    \begin{enumerate}[label=(\roman*)]
        \item Case where the refinement target $c_{s + 1}$ at time $s + 1$ is known. \\
        $\mathcal{T}_{s + 1} = \mathcal{T}_s$. Take any $h^{\star} = (c^{\star}, r^{\star}) \in \mathcal{T}_{s + 1}$. If $c^{\star} = c_{s + 1}$, then by Assumption~\ref{asm:soundness},
        \begin{align}
            h^{\star} \in H_{c_{s + 1}}(e_{s + 1})
        \end{align}
        holds. Also, by the induction hypothesis,
        \begin{align}
            h^{\star} \in \mathcal{V}_s
        \end{align}
        holds as well, so
        \begin{align}
            h^{\star} \in \mathcal{V}_s \cap H_{c_{s + 1}}(e_{s + 1}) \subseteq \mathcal{V}_{s + 1}.
        \end{align}
        If $c^{\star} \neq c_{s + 1}$, then
        \begin{align}
            h^{\star} \notin H_{c_{s + 1}}.
        \end{align}
        On the other hand, by the induction hypothesis,
        \begin{align}
            h^{\star} \in \mathcal{V}_s,
        \end{align}
        so
        \begin{align}
            h^{\star} \in \mathcal{V}_s \setminus H_{c_{s + 1}} \subseteq \mathcal{V}_{s + 1}.
        \end{align}
        Thus, for any $h^{\star} \in \mathcal{T}_{s + 1}$,
        \begin{align}
            h^{\star} \in \mathcal{V}_{s + 1}
        \end{align}
        holds, and hence
        \begin{align}
            \mathcal{T}_{s + 1} = \mathcal{T}_s \subseteq \mathcal{V}_{s + 1}
        \end{align}
        holds.
        \item Case where the refinement target $c_{s + 1}$ at time $s + 1$ is unknown. \\
        $\mathcal{T}_{s + 1} = \mathcal{T}_s \cup H_{c_{s + 1}}^{\star}$. Taking any $h^{\star} = (c_{s + 1}, r^{\star}) \in H_{c_{s + 1}}^{\star}$, by Assumption~\ref{asm:soundness},
        \begin{align}
            h^{\star} \in H_{c_{s + 1}}(e_{s + 1})
        \end{align}
        holds, and so
        \begin{align}
            H_{c_{s + 1}}^{\star} \subseteq H_{c_{s + 1}}(e_{s + 1})
        \end{align}
        holds. Also, by the induction hypothesis,
        \begin{align}
            \mathcal{T}_s \subseteq \mathcal{V}_s.
        \end{align}
        Thus,
        \begin{align}
            \mathcal{T}_{s + 1} = \mathcal{T}_s \cup H_{c_{s + 1}}^{\star} \subseteq \mathcal{V}_s \cup H_{c_{s + 1}}(e_{s + 1}) = \mathcal{V}_{s + 1}.
        \end{align}
    \end{enumerate}
    Therefore, the claim also holds for $t = s + 1$.
    By the above, preservation of the true hypothesis holds.

    3. By Property 2, as events, $\{t \ge t_c\} = \{\mathcal{T}_t \neq \varnothing\} \subseteq \{\mathcal{V}_t \neq \varnothing\}$ holds. Therefore, by Lemma~\ref{lem:non-empty-of-line} and monotonicity of probability,
    \begin{align}
        \Pr[\mathcal{V}_t \neq \varnothing] \geq \Pr[\mathcal{T}_t \neq \varnothing] = \Pr[t \ge t_c] \ge 1 - (1 - q^{\text{Line}})\exp \left( - \pi^{\text{Line}} A_{\tau} \left(1 - e^{- \tau (t - 1)}\right) \right)
    \end{align}
    follows, where
    \begin{align}
        A_{\tau} = \frac{e^{- 2 \tau}}{1 - e^{- \tau}}.
    \end{align}
\end{proof}

\subsection{Proof of Theorem~\ref{thm:theorem-of-irri}}
\label{subapp:proof-of-thm:theorem-of-irri}

Under Assumption~\ref{asm:realizability-of-irri}, the following lemma holds.
\begin{lemma}
    \label{lem:non-empty-of-irri}
    Under Assumption~\ref{asm:realizability-of-irri}, if $|C_{\text{init}}| = N$, then
    \begin{align}
        \Pr[\mathcal{T}_0 \neq \varnothing] \ge 1 - (1 - q^{\text{IRRI}})^N.
    \end{align}
\end{lemma}

\begin{proof}
    $\mathcal{T}_0 \neq \varnothing$ holds when at least one $c \in C_{\text{init}}$ satisfies $Y_c = 1$. Thus, by Assumption~\ref{asm:realizability-of-irri} and considering the complementary event, we obtain
    \begin{align}
        \Pr[\mathcal{T}_0 \neq \varnothing] = \Pr\left[\bigcup_{c \in C_{\text{init}}} \{Y_c = 1\}\right] = 1 - \prod_{c \in C_{\text{init}}} \Pr[Y_c = 0] \ge 1 - (1 - q^{\text{IRRI}})^N.
    \end{align}
\end{proof}

Using this lemma, Theorem~\ref{thm:theorem-of-irri} can be proved as follows.
\begin{proof}
    1. For any program $c_t$ and any observation $e_t$,
    \begin{align}
        H_{c_t}(e_t) \subseteq H_{c_t}
    \end{align}
    implies
    \begin{align}
        \mathcal{V}_{t-1} \cap H_{c_t}(e_t) \subseteq \mathcal{V}_{t-1} \cap H_{c_t}.
    \end{align}
    From this, we obtain
    \begin{align}
        \mathcal{V}_t &= (\mathcal{V}_{t-1} \cap H_{c_t}(e_t)) \cup (\mathcal{V}_{t-1} \setminus H_{c_t}) \\
        &\subseteq (\mathcal{V}_{t-1} \cap H_{c_t}) \cup (\mathcal{V}_{t-1} \setminus H_{c_t}) \\
        &= \mathcal{V}_{t-1}.
    \end{align}
    Hence, monotonic non-increasing holds.
    
    2. The case $\mathcal{T}_0 = \varnothing$ holds trivially, so we consider only the case $\mathcal{T}_0 \neq \varnothing$ and prove the claim by mathematical induction. \\
    First, at $t = 0$, since for any $c \in C_{\text{init}}$,
    \begin{align}
        H_c^{\star} \subseteq H_c,
    \end{align}
    we have
    \begin{align}
        \mathcal{T}_0 = \bigcup_{c \in C_{\text{init}}} H_c^{\star} \subseteq \bigcup_{c \in C_{\text{init}}} H_c = \mathcal{V}_0,
    \end{align}
    so the claim holds. \\
    Next, suppose that preservation of the true hypothesis holds at $t = s - 1 \, (s\ge1)$; that is,
    \begin{align}
        \mathcal{T}_0 \subseteq \mathcal{V}_{s - 1}.
    \end{align}
    Since $\mathcal{T}_0 \neq \varnothing$, take any $h^{\star} = (c^{\star}, r^{\star}) \in \mathcal{T}_0$. If $c^{\star} = c_s$, then by Assumption~\ref{asm:soundness},
    \begin{align}
        h^{\star} \in H_{c_s}(e_s)
    \end{align}
    holds. Also, by the induction hypothesis,
    \begin{align}
        h^{\star} \in \mathcal{V}_{s - 1}
    \end{align}
    holds as well, so
    \begin{align}
        h^{\star} \in \mathcal{V}_{s - 1} \cap H_{c_s}(e_s) \subseteq \mathcal{V}_s.
    \end{align}
    If $c^{\star} \neq c_s$, then
    \begin{align}
        h^{\star} \notin H_{c_s}.
    \end{align}
    On the other hand, by the induction hypothesis,
    \begin{align}
        h^{\star} \in \mathcal{V}_{s - 1},
    \end{align}
    so
    \begin{align}
        h^{\star} \in \mathcal{V}_{s - 1} \setminus H_{c_s} \subseteq \mathcal{V}_s.
    \end{align}
    Thus, for any $h^{\star} \in \mathcal{T}_0$,
    \begin{align}
        h^{\star} \in \mathcal{V}_s
    \end{align}
    holds, so the claim also holds for $t = s$. \\
    By the above, preservation of the true hypothesis holds.

    3. By Property 2, as events, $\{\mathcal{T}_0 \neq \varnothing\} \subseteq \{\mathcal{V}_t \neq \varnothing\}$ holds. Therefore, by Lemma~\ref{lem:non-empty-of-irri} and monotonicity of probability,
    \begin{align}
        \Pr[\mathcal{V}_t \neq \varnothing] \geq \Pr[\mathcal{T}_0 \neq \varnothing] \ge 1 - (1 - q^{\text{IRRI}})^N
    \end{align}
    follows.
\end{proof}

\subsection{Proof of Theorem~\ref{thm:theorem-of-sfs}}
\label{subapp:proof-of-thm:theorem-of-sfs}

Under Assumption~\ref{asm:realizability-of-sfs}, the following lemma holds.
\begin{lemma}
    \label{lem:non-empty-of-sfs}
    Under Assumption~\ref{asm:realizability-of-sfs}, if $|C_{\text{init}}| = N$, then
    \begin{align}
        \Pr[t \ge t_c] \ge 1 - (1 - q^{\text{SFS}})^N \exp \left( - \pi^{\text{SFS}} \theta^{\text{SFS}} \left(t - \frac{1 - e^{- \gamma t}}{1 - e^{- \gamma}} \right) \right).
    \end{align}
\end{lemma}

\begin{proof}
    Let $\mathcal{Y}_{\text{init}} = \bigcap_{c \in C_{\text{init}}} \{Y_c = 0\}$ and $\mathcal{Z}_t = \{c_t = \text{unknown}\} \cap \{ Z_{c_t} = 1 \}$. Then, under the condition $t < t_c$, its complement $\mathcal{Z}_t^c$ is given by $\mathcal{Z}_t^c = \{ Z_{c_t} = 0 \}$. Consider the probability $\Pr[t < t_c]$.
    \begin{align}
        \Pr[t < t_c] &= \Pr[\mathcal{Y}_{\text{init}} \cap \mathcal{Z}_1^c \cap \mathcal{Z}_2^c \cap \cdots \cap \mathcal{Z}_t^c \mid O_{t}] \\
        &= \Pr[\mathcal{Y}_{\text{init}}] \prod_{s = 1}^{t} \Pr[\mathcal{Z}_s^c \mid O_{s - 1}, \mathcal{Y}_{\text{init}}, \mathcal{Z}_1^c \cap \cdots \cap \mathcal{Z}_{s - 1}^c] \\
        &= \Pr[\mathcal{Y}_{\text{init}}] \prod_{s = 1}^{t} \left( 1 - \Pr[\mathcal{Z}_s \mid O_{s - 1}, \mathcal{Y}_{\text{init}}, \mathcal{Z}_1^c \cap \cdots \cap \mathcal{Z}_{s - 1}^c] \right) \\
        &= \Pr[\mathcal{Y}_{\text{init}}] \prod_{s = 1}^{t} \left( 1 - \Pr[Z_{c_s} = 1, c_s = \text{unknown} \mid O_{s - 1}, Z_{c_1} = \cdots = Z_{c_{s - 1}} = 0] \right).
    \end{align}
    For each term, since $c_s$ is an unknown program, the independence in Assumption~\ref{asm:realizability-of-sfs} implies that
    \begin{align}
        \Pr[Z_{c_s} = 1, c_s = \text{unknown} \mid O_{s - 1}, Z_{c_1} = \cdots = Z_{c_{s - 1}} = 0] = \Pr[Z_{c_s} = 1, c_s = \text{unknown} \mid O_{s - 1}].
    \end{align}
    Hence,
    \begin{align}
        \Pr[t < t_c] &= \Pr[\mathcal{Y}_{\text{init}}] \prod_{s = 1}^{t} \left( 1 - \Pr[Z_{c_s} = 1, c_s = \text{unknown} \mid O_{s - 1}] \right) \\
        &= \underbrace{\Pr[\mathcal{Y}_{\text{init}}]}_{(a)} \underbrace{\prod_{s = 1}^{t} \left( 1 - \Pr[Z_{c_s} = 1 \mid c_s = \text{unknown}] \Pr[c_s = \text{unknown} \mid O_{s - 1}] \right)}_{(b)}.
    \end{align}
    For (a), by the uniform lower bound in Assumption~\ref{asm:realizability-of-sfs}, we have
    \begin{align}
        \Pr[Y_c = 1] \ge q^{\text{SFS}}
    \end{align}
    for any $c \in C_{\text{init}}$, and therefore
    \begin{align}
        \Pr[\mathcal{Y}_{\text{init}}] &= \prod_{c \in C_{init}} \Pr[Y_c = 0] \\
        &= \prod_{c \in C_{init}} (1 - \Pr[Y_c = 1]) \\
        &\le (1 - q^{\text{SFS}})^N.
    \end{align}
    For each term in (b), by the uniform lower bound in Assumption~\ref{asm:realizability-of-sfs},
    \begin{align}
        \Pr[Z_{c_s} = 1 \mid c_s = \text{unknown}] \ge \pi^{\text{SFS}}.
    \end{align}
    Moreover, by Assumption~\ref{asm:realizability-of-sfs}, for any time $t$ and any depth $k \, (\ge 1)$,
    \begin{align}
        q_t^{\text{old}}(k) \le 1 - \theta^{\text{SFS}}
    \end{align}
    holds, and hence
    \begin{align}
        \Pr[c_s = \text{unknown} \mid O_{s - 1}] &= \sum_{k = 1}^{K_{s - 1}} q_s^{\text{dep}}(k) \left(1 - q_s^{\text{old}}(k)\right) \\
        &\ge \theta^{\text{SFS}} \sum_{k = 1}^{K_{s - 1}} q_s^{\text{dep}}(k) \\
        &= \theta^{\text{SFS}} \left(1 - q_s^{\text{dep}}(0)\right) \\
        &= \theta^{\text{SFS}} \left(1 - e^{- \gamma (s - 1)} \right).
    \end{align}
    Combining these, we obtain the inequality
    \begin{align}
        1 - \Pr[Z_{c_s} = 1 \mid c_s = \text{unknown}]\Pr[c_s = \text{unknown} \mid O_{s - 1}] \le 1 - \pi^{\text{SFS}} \theta^{\text{SFS}} \left(1 - e^{- \gamma (s - 1)} \right)
    \end{align}
    for each term. Applying the inequality
    \begin{align}
        1 - x \le e^{- x},
    \end{align}
    we further obtain
    \begin{align}
        1 - \Pr[Z_{c_s} = 1 \mid c_s = \text{unknown}]\Pr[c_s = \text{unknown} \mid O_{s - 1}] \le \exp \left( - \pi^{\text{SFS}} \theta^{\text{SFS}} \left(1 - e^{- \gamma (s - 1)} \right) \right).
    \end{align}
    Altogether, (b) satisfies
    \begin{align}
        (b) &\le \prod_{s = 1}^{t} \exp \left( - \pi^{\text{SFS}} \theta^{\text{SFS}} \left(1 - e^{- \gamma (s - 1)} \right) \right) \\
        &= \exp \left( - \pi^{\text{SFS}} \theta^{\text{SFS}} \sum_{s = 1}^{t} \left(1 - e^{- \gamma (s - 1)} \right) \right) \\
        &= \exp \left( - \pi^{\text{SFS}} \theta^{\text{SFS}} \left(t - \frac{1 - e^{- \gamma t}}{1 - e^{- \gamma}} \right) \right).
    \end{align}
    Therefore,
    \begin{align}
        \Pr[t < t_c] \le (1 - q^{\text{SFS}})^N \exp \left( - \pi^{\text{SFS}} \theta^{\text{SFS}} \left(t - \frac{1 - e^{- \gamma t}}{1 - e^{- \gamma}} \right) \right),
    \end{align}
    and hence
    \begin{align}
        \Pr[t \ge t_c] = 1 - \Pr[t < t_c] \ge 1 - (1 - q^{\text{SFS}})^N \exp \left( - \pi^{\text{SFS}} \theta^{\text{SFS}} \left(t - \frac{1 - e^{- \gamma t}}{1 - e^{- \gamma}} \right) \right)
    \end{align}
    holds.
\end{proof}

Using this lemma, Theorem~\ref{thm:theorem-of-sfs} can be proved as follows.
\begin{proof}
    1. At each time $t$, the probability that the refinement target is unknown is given by
    \begin{align}
        \sum_{k = 1}^{K_{t - 1}} q_t^{\text{dep}}(k) \left(1 - q_t^{\text{old}}(k)\right).
    \end{align}
    In this case, regarding the hypothesis space that survives at time $t$,
    \begin{align}
        \mathcal{V}_t = \mathcal{V}_{t-1} \cup H_{c_t}(e_t) \supseteq \mathcal{V}_{t-1}
    \end{align}
    holds. \\
    On the other hand, the probability that the refinement target is known is given by
    \begin{align}
        \sum_{k = 0}^{K_{t - 1}} q_t^{\text{dep}}(k) q_t^{\text{old}}(k).
    \end{align}
    In this case, the hypothesis space that survives at time $t$ is expressed as
    \begin{align}
        \mathcal{V}_t = (\mathcal{V}_{t-1} \cap H_{c_t}(e_t)) \cup (\mathcal{V}_{t-1} \setminus H_{c_t}).
    \end{align}
    For any program $c_t$ and any observation $e_t$,
    \begin{align}
        H_{c_t}(e_t) \subseteq H_{c_t}
    \end{align}
    and thus
    \begin{align}
        \mathcal{V}_{t-1} \cap H_{c_t}(e_t) \subseteq \mathcal{V}_{t-1} \cap H_{c_t}
    \end{align}
    holds. From this,
    \begin{align}
        \mathcal{V}_t &= (\mathcal{V}_{t-1} \cap H_{c_t}(e_t)) \cup (\mathcal{V}_{t-1} \setminus H_{c_t}) \\
        &\subseteq (\mathcal{V}_{t-1} \cap H_{c_t}) \cup (\mathcal{V}_{t-1} \setminus H_{c_t}) \\
        &= \mathcal{V}_{t-1}
    \end{align}
    follows. \\
    Furthermore, the probability that all refinement targets are known programs throughout the refinement process from time $1$ to time $t$ satisfies
    \begin{align}
        \Pr[\text{all refinement targets are known}] &= \prod_{s = 1}^{t} \Pr[c_s = \text{known} \mid O_{s - 1}] \\
        &= \prod_{s = 1}^{t} \left( \sum_{k = 0}^{K_{s - 1}} q_s^{\text{dep}}(k) q_s^{\text{old}}(k) \right) \\
        &= \prod_{s = 1}^{t} \left( q_s^{\text{dep}}(0) q_s^{\text{old}}(0) + \sum_{k = 1}^{K_{s - 1}} q_s^{\text{dep}}(k) q_s^{\text{old}}(k) \right) \\
        &\ge \prod_{s = 1}^{t} q_s^{\text{dep}}(0) q_s^{\text{old}}(0) \\
        &= \prod_{s = 1}^{t} q_s^{\text{dep}}(0) \\
        & = \prod_{s = 1}^t \exp \left(- \gamma (s - 1)\right) \\
        &= \exp \left(- \gamma \sum_{s = 1}^{t} (s - 1)\right) \\
        &= \exp \left(- \gamma \frac{t(t - 1)}{2}\right).
    \end{align}
    From this, in the refinement process from time $1$ to time $t$, we obtain
    \begin{align}
        \mathcal{V}_t \subseteq \mathcal{V}_{t - 1} \subseteq \cdots \subseteq\mathcal{V}_0
    \end{align}
    with probability at least
    \begin{align}
        1 - \delta_{\text{SFS}} = \exp \left(- \gamma \frac{t(t - 1)}{2}\right).
    \end{align}
    
    2. We prove this by mathematical induction. \\
    First, consider the case $t = t_c$.
    \begin{enumerate}[label=(\roman*)]
        \item Case $t_c = 0$. \\
        Since for any $c \in C_{\text{init}}$
        \begin{align}
            H_c^{\star} \subseteq H_c,
        \end{align}
        we have
        \begin{align}
            \mathcal{T}_0 = \bigcup_{c \in C_{\text{init}}} H_c^{\star} \subseteq \bigcup_{c \in C_{\text{init}}} H_c = \mathcal{V}_0,
        \end{align}
        so the claim holds.
        \item Case $t_c \ge 1$. \\
        Since the true hypothesis space first becomes non-empty at time $t_c$,
        \begin{align}
            \mathcal{T}_{t_c} = H_{c_{t_c}}^{\star} (\neq \varnothing)
        \end{align}
        holds. Also, since program $c_{t_c}$ should be unknown,
        \begin{align}
            H_{c_{t_c}}(e_{t_c})\subseteq \mathcal{V}_{t_c}
        \end{align}
        holds. Here, taking any $h^{\star} = (c_{t_c}, r^{\star}) \in H_{c_{t_c}}^{\star}$, by Assumption~\ref{asm:soundness},
        \begin{align}
            h^{\star} \in H_{c_{t_c}}(e_{t_c})
        \end{align}
        holds, and so
        \begin{align}
            H_{c_{t_c}}^{\star} \subseteq H_{c_{t_c}}(e_{t_c}).
        \end{align}
        Therefore,
        \begin{align}
            \mathcal{T}_{t_c} = H_{c_{t_c}}^{\star} \subseteq H_{c_{t_c}}(e_{t_c}) \subseteq \mathcal{V}_{t_c}
        \end{align}
        holds.
    \end{enumerate}
    Next, suppose that preservation of the true hypothesis holds at $t = s \, (s \ge t_c)$; that is,
    \begin{align}
        \mathcal{T}_s \subseteq \mathcal{V}_s.
    \end{align}
    \begin{enumerate}[label=(\roman*)]
        \item Case where the refinement target $c_{s + 1}$ at time $s + 1$ is known. \\
        $\mathcal{T}_{s + 1} = \mathcal{T}_s$. Take any $h^{\star} = (c^{\star}, r^{\star}) \in \mathcal{T}_{s + 1}$. If $c^{\star} = c_{s + 1}$, then by Assumption~\ref{asm:soundness},
        \begin{align}
            h^{\star} \in H_{c_{s + 1}}(e_{s + 1})
        \end{align}
        holds. Also, by the induction hypothesis,
        \begin{align}
            h^{\star} \in \mathcal{V}_s
        \end{align}
        holds as well, so
        \begin{align}
            h^{\star} \in \mathcal{V}_s \cap H_{c_{s + 1}}(e_{s + 1}) \subseteq \mathcal{V}_{s + 1}.
        \end{align}
        If $c^{\star} \neq c_{s + 1}$, then
        \begin{align}
            h^{\star} \notin H_{c_{s + 1}}.
        \end{align}
        On the other hand, by the induction hypothesis,
        \begin{align}
            h^{\star} \in \mathcal{V}_s,
        \end{align}
        so
        \begin{align}
            h^{\star} \in \mathcal{V}_s \setminus H_{c_{s + 1}} \subseteq \mathcal{V}_{s + 1}.
        \end{align}
        Thus, for any $h^{\star} \in \mathcal{T}_{s + 1}$,
        \begin{align}
            h^{\star} \in \mathcal{V}_{s + 1}
        \end{align}
        holds, and hence
        \begin{align}
            \mathcal{T}_{s + 1} = \mathcal{T}_s \subseteq \mathcal{V}_{s + 1}
        \end{align}
        holds.
        \item Case where the refinement target $c_{s + 1}$ at time $s + 1$ is unknown. \\
        $\mathcal{T}_{s + 1} = \mathcal{T}_s \cup H_{c_{s + 1}}^{\star}$. Taking any $h^{\star} = (c_{s + 1}, r^{\star}) \in H_{c_{s + 1}}^{\star}$, by Assumption~\ref{asm:soundness},
        \begin{align}
            h^{\star} \in H_{c_{s + 1}}(e_{s + 1})
        \end{align}
        holds, and so
        \begin{align}
            H_{c_{s + 1}}^{\star} \subseteq H_{c_{s + 1}}(e_{s + 1})
        \end{align}
        holds. Also, by the induction hypothesis,
        \begin{align}
            \mathcal{T}_s \subseteq \mathcal{V}_s.
        \end{align}
        Thus,
        \begin{align}
            \mathcal{T}_{s + 1} = \mathcal{T}_s \cup H_{c_{s + 1}}^{\star} \subseteq \mathcal{V}_s \cup H_{c_{s + 1}}(e_{s + 1}) = \mathcal{V}_{s + 1}.
        \end{align}
    \end{enumerate}
    Therefore, the claim also holds for $t = s + 1$.
    By the above, preservation of the true hypothesis holds.

    3. By Property 2, as events, $\{t \ge t_c\} = \{\mathcal{T}_t \neq \varnothing\} \subseteq \{\mathcal{V}_t \neq \varnothing\}$ holds. Therefore, by Lemma~\ref{lem:non-empty-of-sfs} and monotonicity of probability,
    \begin{align}
        \Pr[\mathcal{V}_t \neq \varnothing] \geq \Pr[\mathcal{T}_t \neq \varnothing] = \Pr[t \ge t_c] \ge 1 - (1 - q^{\text{SFS}})^N \exp \left( - \pi^{\text{SFS}} \theta^{\text{SFS}} \left(t - \frac{1 - e^{- \gamma t}}{1 - e^{- \gamma}} \right) \right)
    \end{align}
    follows.
\end{proof}

\FloatBarrier
\section{Further analysis and discussion}
\subsection{Discussion about Assumption~\ref{asm:soundness}}
\label{subapp:discussion-about-asm:soundness}

Assumption~\ref{asm:soundness} requires that the true hypothesis be consistent with all observations. In practical multi-turn program refinement, these observations are based on self-generated validation test cases, so the feasibility of the assumption depends on the validity of those tests. LLM-generated test cases may satisfy only part of the specification captured by the ground-truth test cases, or may impose requirements that are not included in the ground-truth tests. In the former case, a program that should be rejected by the ground-truth tests may be judged correct by the validation tests (a false positive). This may prevent us from eliminating faulty programs from the hypothesis space, but it does not cause the true solution to be discarded. In the latter case, a program that would not be rejected by the ground-truth tests may be judged incorrect by the validation tests (a false negative). This may cause the true solution to be discarded, which can affect the feasibility of the assumption.

We therefore computed a confusion matrix to evaluate how accurately the self-generated validation test cases predict the correctness of a solution. The results are reported as the average over ten runs (Figures~\ref{fig:confusion-matrix-gpt-4o-mini-apps}-~\ref{fig:confusion-matrix-llama-3.2-3b-instruct-humaneval-and-mbpp}). The false negative rate ranges from approximately $12.5\%$ to $32.3\%$ depending on the model and dataset. Some false negatives exist in every setting, so Assumption~\ref{asm:soundness} does not always hold strictly. However, the false negative rate is at most about $30\%$, which indicates that most validation tests do not reject correct solutions and that the assumption holds reasonably well in practice.

\begin{figure}[htbp]
    \centering
    \begin{minipage}{0.49\linewidth}
        \centering
        \includegraphics[width=\linewidth]{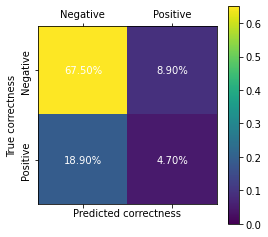}
    \end{minipage}
    \caption{\textbf{Confusion matrix for self-generated validation tests on APPS using gpt-4o-mini.}}
    \label{fig:confusion-matrix-gpt-4o-mini-apps}
\end{figure}

\begin{figure}[htbp]
    \centering
    \begin{minipage}{0.49\linewidth}
        \centering
        \includegraphics[width=\linewidth]{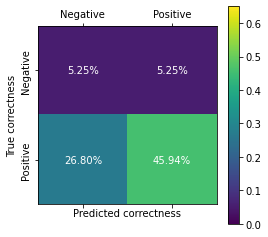}
    \end{minipage}
    \begin{minipage}{0.49\linewidth}
        \centering
        \includegraphics[width=\linewidth]{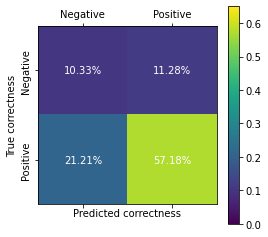}
    \end{minipage}
    \caption{\textbf{Confusion matrix for self-generated validation tests on HumanEval using gpt-4o-mini (left) and on MBPP using gpt-4o-mini (right).}}
    \label{fig:confusion-matrix-gpt-4o-mini-humaneval-and-mbpp}
\end{figure}

\begin{figure}[htbp]
    \centering
    \begin{minipage}{0.49\linewidth}
        \centering
        \includegraphics[width=\linewidth]{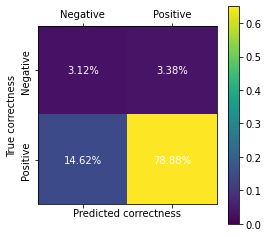}
    \end{minipage}
    \begin{minipage}{0.49\linewidth}
        \centering
        \includegraphics[width=\linewidth]{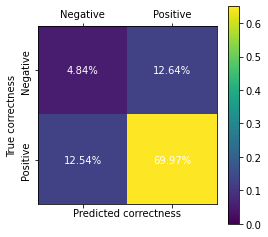}
    \end{minipage}
    \caption{\textbf{Confusion matrix for self-generated validation tests on HumanEval using gpt-4.1-mini (left) and on MBPP using gpt-4.1-mini (right).}}
    \label{fig:confusion-matrix-gpt-4.1-mini-humaneval-and-mbpp}
\end{figure}

\begin{figure}[htbp]
    \centering
    \begin{minipage}{0.49\linewidth}
        \centering
        \includegraphics[width=\linewidth]{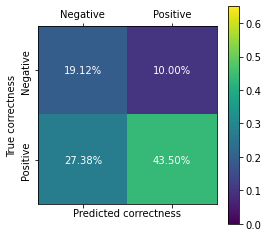}
    \end{minipage}
    \begin{minipage}{0.49\linewidth}
        \centering
        \includegraphics[width=\linewidth]{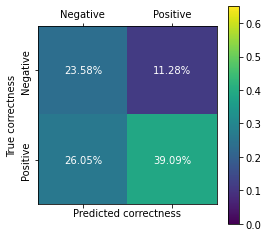}
    \end{minipage}
    \caption{\textbf{Confusion matrix for self-generated validation tests on HumanEval using Llama-3.1-8B-Instruct (left) and on MBPP using Llama-3.1-8B-Instruct (right).}}
    \label{fig:confusion-matrix-llama-3.1-8b-instruct-humaneval-and-mbpp}
\end{figure}

\begin{figure}[htbp]
    \centering
    \begin{minipage}{0.49\linewidth}
        \centering
        \includegraphics[width=\linewidth]{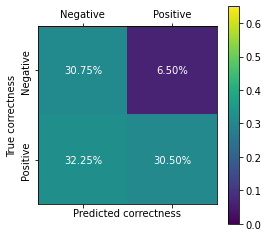}
    \end{minipage}
    \begin{minipage}{0.49\linewidth}
        \centering
        \includegraphics[width=\linewidth]{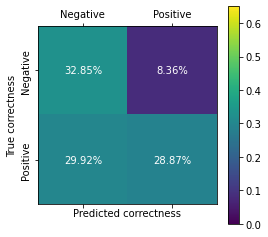}
    \end{minipage}
    \caption{\textbf{Confusion matrix for self-generated validation tests on HumanEval using Llama-3.2-3B-Instruct (left) and on MBPP using Llama-3.2-3B-Instruct (right).}}
    \label{fig:confusion-matrix-llama-3.2-3b-instruct-humaneval-and-mbpp}
\end{figure}

\FloatBarrier
\subsection{Effect of false negatives on the theoretical analysis}
\label{subapp:effect-of-false-negatives-on-the-theoretical-analysis}

Here, we consider the impact of false negatives on the theoretical analysis. As discussed in Appendix~\ref{subapp:discussion-about-asm:soundness}, when the validation tests are generated by an LLM, false negatives may occur, resulting in the rejection of correct hypotheses. To account for this possibility, we modify Assumption~\ref{asm:soundness} as follows.
\begin{assumption}
    \label{asm:probabilistic-soundness}
    (Probabilistic Soundness) For any history $O_t = ((h_1, e_1),(h_2, e_2), \cdots, (h_t, e_t))$ and $t_b, t_e$ with $1\le t_b \le t_e \le t$, there exists $\rho_{O_t, t_b:t_e} \in [0, 1)$ such that for every $t_b \le s \le t_e$ and $c_s = c^{\star}$,
    \begin{align}
        \Pr[H^{\star}_{c^{\star}} \subseteq H_{c^{\star}}(e_s)] \ge 1 - \rho_{O_t, t_b:t_e}.
    \end{align}
\end{assumption}

Under this assumption, let $\mathcal{G}_{O_t, t_b:t_e}$ denote the event that no false negatives have occurred at any step from time $t_b$ to time $t_e$ for which $c_s = c^{\star}$. Furthermore, let $\mathcal{N}_{O_t, t_b:t_e}$ denote the set of steps from time $t_b$ to time $t_e$ at which $c_s = c^{\star}$. Then, by the assumption, we have
\begin{align}
    \Pr[\mathcal{G}_{O_t, t_b:t_e}] = \prod_{s \in \mathcal{N}_{O_t, t_b:t_e}} \Pr[H^{\star}_{c_s} \subseteq H_{c_s}(e_s)] \ge (1 - \rho_{O_t, t_b:t_e})^{|\mathcal{N}_{O_t, t_b:t_e}|}.
\end{align}

Consequently, theorems can be reformulated as follows. Compared with the original theorems, there are two differences. First, the preservation of the true hypothesis is no longer deterministic but holds only probabilistically. Second, since the original guarantees rely on this property, they are weakened accordingly. The statements that previously held deterministically now hold only probabilistically, and the lower bound on the non-emptiness probability becomes looser.
\begin{theorem}
    \label{thm:theorem-of-line-fnr}
    Under Assumptions~\ref{asm:realizability-of-line} and~\ref{asm:probabilistic-soundness} , for any history $O_t = ((h_1, e_1),(h_2, e_2), \cdots, (h_t, e_t))$ and time $t_c = \min\{t \ge 0 \mid \mathcal{T}_t \neq \varnothing\}$, the following properties hold.
    \begin{enumerate}
        \item (($1 - \delta_{\text{Line}}$)-monotonic non-increasing) \\
        At each time $t$,
        \begin{itemize}
            \item With probability $1 - q_t^{\text{old}}$, the refinement target $c_t$ is unknown, in which case $\mathcal{V}_t \supseteq \mathcal{V}_{t - 1}$.
            \item With probability $q_t^{\text{old}}$, the refinement target $c_t$ is known, in which case $\mathcal{V}_t \subseteq \mathcal{V}_{t - 1}$.
        \end{itemize}
        Moreover, with probability at most
        \begin{align}
            1 - \delta_{\text{Line}} = \exp \left( - A_{\tau} \left(1 - e^{- \tau (t - 1)}\right) \right) \quad \left(A_{\tau} = \frac{e^{- 2 \tau}}{1 - e^{- \tau}}\right),
        \end{align}
        all refinement targets $c_1, \ldots, c_t$ are known, in which case $\mathcal{V}_t \subseteq \mathcal{V}_{t - 1} \subseteq \cdots \subseteq\mathcal{V}_0$.
        \item (Preservation of the true hypothesis) With probability at least $(1 - \rho_{O_t, t_c:t})^{|\mathcal{N}_{O_t, t_c:t}|}$,
        \begin{align}
            \mathcal{T}_t \subseteq \mathcal{V}_t \quad (t \ge t_c).
        \end{align}
        \item (Non-emptiness)
        \begin{align}
            \Pr[\mathcal{V}_t \neq \varnothing] \ge (1 - \rho_{O_t, t_c:t})^{|\mathcal{N}_{O_t, t_c:t}|} - (1 - q^{\text{Line}})\exp \left( - \pi^{\text{Line}} A_{\tau} \left(1 - e^{- \tau (t - 1)}\right) \right),
        \end{align}
        where
        \begin{align}
            A_{\tau} = \frac{e^{- 2 \tau}}{1 - e^{- \tau}}.
        \end{align}
    \end{enumerate}
\end{theorem}

\begin{theorem}
    \label{thm:theorem-of-irri-fnr}
    Under Assumptions~\ref{asm:realizability-of-irri} and~\ref{asm:probabilistic-soundness}, for any history $O_t = ((h_1, e_1),(h_2, e_2), \cdots, (h_t, e_t))$, the following properties hold.
    \begin{enumerate}
        \item (Monotonic non-increasing) $\mathcal{V}_t \subseteq \mathcal{V}_{t-1}$.
        \item (Preservation of the true hypothesis) With probability at least $(1 - \rho_{O_t, 1:t})^{|\mathcal{N}_{O_t, 1:t}|}$,
        \begin{align}
            \mathcal{T}_0 \subseteq \mathcal{V}_t.
        \end{align}
        \item (Non-emptiness) $\Pr[\mathcal{V}_t \neq \varnothing] \ge (1 - \rho_{O_t, 1:t})^{|\mathcal{N}_{O_t, 1:t}|} - (1 - q^{\text{IRRI}})^N$.
    \end{enumerate}
\end{theorem}

\begin{theorem}
    \label{thm:theorem-of-sfs-fnr}
    Under Assumptions~\ref{asm:realizability-of-sfs} and~\ref{asm:probabilistic-soundness}, for any history $O_t = ((h_1, e_1),(h_2, e_2), \cdots, (h_t, e_t))$ and time $t_c = \min\{t \ge 0 \mid \mathcal{T}_t \neq \varnothing\}$, the following properties hold.
    \begin{enumerate}
        \item ((1 - $\delta_{\text{SFS}}$)-monotonic non-increasing) \\
        At each time $t$,
        \begin{itemize}
            \item With probability $\sum_{k = 1}^{K_{t - 1}} q_t^{\text{dep}}(k) \left(1 - q_t^{\text{old}}(k)\right)$, the refinement target $c_t$ is unknown, in which case
            $\mathcal{V}_t \supseteq \mathcal{V}_{t - 1}$.
            \item With probability $\sum_{k = 0}^{K_{t - 1}} q_t^{\text{dep}}(k) q_t^{\text{old}}(k)$, the refinement target $c_t$ is known, in which case $\mathcal{V}_t \subseteq \mathcal{V}_{t - 1}$. 
        \end{itemize}
        Moreover, with probability at least
        \begin{align}
            1 - \delta_{\text{SFS}} = \exp \left(- \gamma \frac{t(t - 1)}{2}\right),
        \end{align}
        all refinement targets $c_1, \ldots, c_t$ are known,, in which case $\mathcal{V}_t \subseteq \mathcal{V}_{t - 1} \subseteq \cdots \subseteq \mathcal{V}_0$.
        \item (Preservation of the true hypothesis) With probability at least $(1 - \rho_{O_t, t_c:t})^{|\mathcal{N}_{O_t, t_c:t}|}$,
        \begin{align}
            \mathcal{T}_t \subseteq \mathcal{V}_t \quad (t \ge t_c).
        \end{align}
        \item (Non-emptiness)
        \begin{align}
            \Pr[\mathcal{V}_t \neq \varnothing] \ge (1 - \rho_{O_t, t_c:t})^{|\mathcal{N}_{O_t, t_c:t}|} - (1 - q^{\text{SFS}})^N \exp \left( - \pi^{\text{SFS}} \theta^{\text{SFS}} \left(t - \frac{1 - e^{- \gamma t}}{1 - e^{- \gamma}} \right) \right).
        \end{align}
    \end{enumerate}
\end{theorem}

\FloatBarrier
\subsection{Sequential self-refinement method (Line)}
\label{subapp:further-discussion-of-line}

From Theorem~\ref{thm:theorem-of-line}, the hypothesis space $\mathcal{V}_t$ in Line can be non-decreasing or non-increasing depending on whether the refinement target $c_t$ at time $t$ is an unknown or known program. Furthermore, over the refinement process from time $1$ to time $t$, all refinement targets are known, and consequently the hypothesis space becomes monotonically non-increasing with probability at most
\begin{align}
    \exp \left( - A_{\tau} \left(1 - e^{- \tau (t - 1)}\right) \right).
\end{align}

This indicates that there is a limit to the extent to which Line can properly narrow down the hypothesis space. Moreover, the upper bound is monotonically decreasing in $t$ and approaches
\begin{align}
    \exp \left( - A_{\tau} \right)
\end{align}
as $t \rightarrow \infty$. Therefore, the probability that all refinement targets are known, and consequently the hypothesis space is monotonically non-increasing, in Line decreases as $t$ increases.

Theorem~\ref{thm:theorem-of-line} also states that Line satisfies non-emptiness with probability at least a certain threshold. Here, we discuss the condition on time $t$ required for the probability $\Pr[\mathcal{V}_t \neq \varnothing]$ to be at least $1 - \xi_{\text{Line}}$. Before doing so, we first organize the condition that $\xi_{\text{Line}}$ must satisfy. The lower bound on the probability $\Pr[\mathcal{V}_t \neq \varnothing]$,
\begin{align}
    1 - (1 - q^{\text{Line}})\exp \left( - \pi^{\text{Line}} A_{\tau} \left(1 - e^{- \tau (t - 1)}\right) \right),
\end{align}
is monotonically increasing in $t$ and approaches
\begin{align}
    1 - (1 - q^{\text{Line}})\exp \left( - \pi^{\text{Line}} A_{\tau} \right)
\end{align}
as $t \rightarrow \infty$. From this, we must have
\begin{align}
    \xi_{\text{Line}} \ge (1 - q^{\text{Line}})\exp \left( - \pi^{\text{Line}} A_{\tau} \right).
\end{align}

This means that no matter how good Line behaves, the lower bound on the probability $\Pr[\mathcal{V}_t \neq \varnothing]$ never reaches $1$. In other words, this analysis does not guarantee that $\mathcal{V}_t$ is non-empty with probability $1$. Next, considering the condition on time $t$,
\begin{align}
    t \ge 1 - \frac{1}{\tau} \log \left(1 + \frac{1}{\pi^{\text{Line}} A_{\tau}} \log \frac{\xi_{\text{Line}}}{1 - q^{\text{Line}}}\right).
\end{align}

Thus, the minimum time $t_{\min}$ at which the probability $\Pr[\mathcal{V}_t \neq \varnothing]$ becomes at least $1 - \xi_{\text{Line}}$ is obtained as
\begin{align}
    t_{\min} = \left\lceil 1 - \frac{1}{\tau} \log \left(1 + \frac{1}{\pi^{\text{Line}} A_{\tau}} \log \frac{\xi_{\text{Line}}}{1 - q^{\text{Line}}}\right) \right\rceil.
\end{align}

Under the condition imposed on $\xi_{\text{Line}}$, as the value of $\xi_{\text{Line}}$ decreases, the right-hand side of the above expression increases. This shows that, in order for non-emptiness to hold with sufficiently high probability, the time $t_{\min}$ needs to be relatively large, requiring more correction steps. 

\subsection{Parameter and upper-bound estimation in Line}
\label{subapp:parameter-and-upper-bound-estimation-in-line}

Recall that we modeled the probability that the refinement target $c_t$ is a previously known program as
\begin{align}
    q_t^{\text{old}} = 1 - \exp (- \tau t) \quad (t \ge 2).
\end{align}
We estimated the parameter $\tau$ appearing in this model empirically by applying Line to real datasets. Specifically, we used gpt-4o-mini as the underlying model and applied Line to four datasets: HumanEval \citep{chen2021evaluating}, MBPP \citep{austin2021program}, APPS \citep{hendrycks2021measuring} and CodeContests \citep{li2022competition}. For each task, Line generates a single initial program and then performs up to $15$ rounds of revision.

At each time step $t \in \{2, \cdots, 15\}$, we determined whether the program being revised had appeared previously and used these observations to compute the empirical probability $\hat{q}_t^{\text{old}}$ that the refinement target $c_t$ is known. We then fit the model $q_t^{\text{old}} = 1 - \exp (- \tau t)$ to the resulting data points $\{(t, \hat{q}_t^{\text{old}})\}_{t = 2}^{15}$ via non-linear least squares to obtain an estimate of $\tau$. Using this estimated value of $\tau$, we computed the upper bound on the probability that all refinement targets are known.

The results are shown in Tables~\ref{tab:estimated-values-of-tau} and~\ref{tab:estimated-values-of-upper-bound}. The extremely small estimated upper bound shows that, in Line, all refinement targets are almost never known.

\begin{table}[htbp]
    \centering
    \begin{minipage}{0.49\linewidth}
        \caption{\textbf{Estimated values of} $\bm{\tau}$.}
        \label{tab:estimated-values-of-tau}
        \small
        \centering
        \begin{tabular}{c|c}
            \toprule
            \textbf{Dataset} & $\bm{\tau}$ \\
            \midrule
            HumanEval & $0.1497 \pm 0.0168$ \\
            MBPP & $0.1490 \pm 0.0113$ \\
            APPS & $0.1109 \pm 0.0057$ \\
            CodeContests & $0.1011 \pm 0.0062$ \\
            \bottomrule
        \end{tabular}
    \end{minipage}
    \hfill
    \begin{minipage}{0.49\linewidth}
        \caption{\textbf{Estimated values of} $\bm{1 - \delta_{\textbf{Line}}}$.}
        \label{tab:estimated-values-of-upper-bound}
        \small
        \centering
        \begin{tabular}{c|c}
            \toprule
            \textbf{Dataset} & $\bm{1 - \delta_{\textbf{Line}} \quad (t = 15)}$ \\
            \midrule
            HumanEval & $9.325 \times 10^{-3}$ \\
            MBPP & $9.134 \times 10^{-3}$ \\
            APPS & $2.442 \times 10^{-3}$ \\
            CodeContests & $1.607 \times 10^{-3}$ \\
            \bottomrule
        \end{tabular}
    \end{minipage}
\end{table}

\FloatBarrier
\subsection{Scattered Forest Search (SFS)}
\label{subapp:further-discussion-of-sfs}

From Theorem~\ref{thm:theorem-of-sfs}, the hypothesis space $\mathcal{V}_t$ in SFS can be non-decreasing or non-increasing depending on whether the refinement target $c_t$ at time $t$ is an unknown or known program. Furthermore, over the refinement process from time $1$ to time $t$, it becomes monotonically non-increasing with probability at least
\begin{align}
    \exp \left(- \gamma \frac{t(t - 1)}{2}\right).
\end{align}

The condition for this probability to be at least $1 - \epsilon_{\text{SFS}}$ is
\begin{align}
    t \le \frac{1 + \sqrt{1 + \frac{8}{\gamma} \log\frac{1}{1 - \epsilon_{\text{SFS}}}}}{2}.
\end{align}

This implies that, for a relatively short refinement process satisfying the above equation, the hypothesis space may approximately exhibit monotonic non-increasing behavior. In other words, this suggests that SFS can properly narrow down the hypothesis space within a small number of revision steps. Moreover, when the refinement process is relatively short, MCTS prioritizes breadth-wise exploration, so the refinement targets are drawn from the set of initial programs with high probability. This means that SFS exhibits refinement behavior similar to that of IRRI.

Theorem~\ref{thm:theorem-of-sfs} also states that SFS satisfies non-emptiness with probability at least a certain threshold. Let
\begin{align}
    G_t(N) = 1 - (1 - q^{\text{SFS}})^N \exp \left( - \pi^{\text{SFS}} \theta^{\text{SFS}} \left(t - \frac{1 - e^{- \gamma t}}{1 - e^{- \gamma}} \right) \right).
\end{align}

Since $q^{\text{SFS}} \in (0, 1)$, we have $0 < 1 - q^{\text{SFS}} < 1$. Therefore, $G_t(N)$is monotonically increasing in $N$ and satisfies
\begin{align}
    \lim_{N \rightarrow \infty} G_t(N) = 1 - \lim_{N \rightarrow \infty} (1 - q^{\text{SFS}})^N \exp \left( - \pi^{\text{SFS}} \theta^{\text{SFS}} \left(t - \frac{1 - e^{- \gamma t}}{1 - e^{- \gamma}} \right) \right)  = 1.
\end{align}

This indicates that increasing the number of initial programs generated has the benefit of increasing the probabilities that $\mathcal{T}_t \neq \varnothing$ and $\mathcal{V}_t \neq \varnothing$.

Compared with Theorem~\ref{thm:theorem-of-irri}, an additional factor
\begin{align}
    \exp \left( - \pi^{\text{SFS}} \theta^{\text{SFS}} \left(t - \frac{1 - e^{- \gamma t}}{1 - e^{- \gamma}} \right) \right)
\end{align}
is multiplied onto the second term in SFS. Here, we consider the function
\begin{align}
    J(t) = t - \frac{1 - e^{- \gamma t}}{1 - e^{- \gamma}},
\end{align}
noting that $\gamma > 0$. From
\begin{align}
    J^{\prime}(t) &= 1 - \frac{\gamma e^{- \gamma t}}{1 - e^{- \gamma}}, \\
    J^{\prime\prime}(t) &= \frac{\gamma^2 e^{- \gamma t}}{1 - e^{- \gamma}} > 0,
\end{align}
the function $J(t)$ is convex. Since $J(t) = 0$ at $t = 0, 1$, it follows that for $t > 1$,
\begin{align}
    t - \frac{1 - e^{- \gamma t}}{1 - e^{- \gamma}} > 0
\end{align}
holds. Taking into account that $\pi^{\text{SFS}} > 0$ and $\theta^{\text{SFS}} > 0$, we obtain
\begin{align}
    \exp \left( - \pi^{\text{SFS}} \theta^{\text{SFS}} \left(t - \frac{1 - e^{- \gamma t}}{1 - e^{- \gamma}} \right) \right) < 1,
\end{align}
and hence
\begin{align}
    1 - (1 - q^{\text{SFS}})^N \exp \left( - \pi^{\text{SFS}} \theta^{\text{SFS}} \left(t - \frac{1 - e^{- \gamma t}}{1 - e^{- \gamma}} \right) \right) > 1 - (1 - q^{\text{SFS}})^N .
\end{align}
This reflects the exploratory capability of SFS in the refinement steps and demonstrates its effectiveness. However, unless $t$ is sufficiently large, no clear difference emerges between the two.

Here, we discuss the condition on time $t$ required for the probability $\Pr[\mathcal{V}_t \neq \varnothing]$ to be at least $1 - \xi_{\text{SFS}}$. Before doing so, we first organize the condition that $\xi_{\text{SFS}}$ must satisfy. From the preceding discussion, we see that the function $J(t)$ is monotonically increasing for $t \ge 1$ and satisfies $J(t) \rightarrow \infty$ as $t \rightarrow \infty$. It follows that the lower bound on the probability $\Pr[\mathcal{V}_t \neq \varnothing]$ is monotonically increasing for $t \ge 1$ and converges to $1$ as $t \rightarrow \infty$. This means that $\xi_{\text{SFS}}$ can be taken to be arbitrarily small. Next, considering the condition on time $t$,
\begin{align}
    t - \frac{1 - e^{- \gamma t}}{1 - e^{- \gamma}} \ge - \frac{1}{\pi^{\text{SFS}} \theta^{\text{SFS}}} \log \frac{\xi_{\text{SFS}}}{(1 - q^{\text{SFS}})^N}.
\end{align}

Let
\begin{align}
    L = \frac{1}{1 - e^{- \gamma}}, \quad \kappa = - \frac{1}{\pi^{\text{SFS}} \theta^{\text{SFS}}} \log \frac{\xi_{\text{SFS}}}{(1 - q^{\text{SFS}})^N},
\end{align}
and consider solving the equation
\begin{align}
    t - L(1 - e^{- \gamma t}) = \kappa \, \Longleftrightarrow \, t = \kappa + L - L e^{- \gamma t}.
\end{align}

Since we expect $\xi_{\text{SFS}} \le (1 - q^{\text{SFS}})^N$ with respect to $\xi_{\text{SFS}}$, we consider the case $\kappa \ge 0$. In addition, since $t \ge 0$, we seek only the larger of the two solutions. Applying the change of variables $\hat{t} = \gamma t - \gamma(\kappa + L)$, we obtain
\begin{align}
    \frac{\hat{t}}{\gamma} = - L e^{- \hat{t}} e^{- \gamma(\kappa + L)} \, \Longleftrightarrow \, \hat{t} e^{\hat{t}} = - \gamma L e^{- \gamma(\kappa + L)}.
\end{align}

Hence, using the Lambert $W$ function, this can be expressed as
\begin{align}
    \hat{t} &= W_0 \left(- \gamma L e^{- \gamma(\kappa + L)}\right) \\
    &= W_0 \left(- \frac{\gamma}{1 - e^{- \gamma}} \exp \left(- \gamma \left(\kappa + \frac{1}{1 - e^{- \gamma}}\right)\right)\right),
\end{align}
so that the solution to the equation is given by
\begin{align}
    t = \kappa + \frac{1}{1 - e^{- \gamma}} + \frac{1}{\gamma} W_0 \left(- \frac{\gamma}{1 - e^{- \gamma}} \exp \left(- \gamma \left(\kappa + \frac{1}{1 - e^{- \gamma}}\right)\right)\right).
\end{align}

Given that the function $J(t)$ is monotonically increasing for $t \ge 1$, the condition for the probability $\Pr[\mathcal{V}_t \neq \varnothing]$ to be at least $1 - \xi_{\text{SFS}}$ is given by
\begin{align}
    t \ge \kappa + \frac{1}{1 - e^{- \gamma}} + \frac{1}{\gamma} W_0 \left(- \frac{\gamma}{1 - e^{- \gamma}} \exp \left(- \gamma \left(\kappa + \frac{1}{1 - e^{- \gamma}}\right)\right)\right).
\end{align}

Thus, the minimum time $t_{\min}$ at which the probability $\Pr[\mathcal{V}_t \neq \varnothing]$ becomes at least $1 - \xi_{\text{SFS}}$ is obtained as
\begin{align}
    t_{\min} = \left\lceil \kappa + \frac{1}{1 - e^{- \gamma}} + \frac{1}{\gamma} W_0 \left(- \frac{\gamma}{1 - e^{- \gamma}} \exp \left(- \gamma \left(\kappa + \frac{1}{1 - e^{- \gamma}}\right)\right)\right) \right\rceil,
\end{align}
where
\begin{align}
    \kappa = - \frac{1}{\pi^{\text{SFS}} \theta^{\text{SFS}}} \log \frac{\xi_{\text{SFS}}}{(1 - q^{\text{SFS}})^N}.
\end{align}

As the value of $\xi_{\text{SFS}}$decreases, the value of $\kappa$ increases, and the value of $t_{\min}$ also increases. This shows that, in order for non-emptiness to hold with sufficiently high probability, the time $t_{\min}$ needs to be relatively large, requiring more correction steps. However, SFS allows the value of $\kappa$ and $t_{\min}$ to be decreased by increasing the number of initial programs generated, unlike Line. This means that, by enlarging the set of initial programs, the probability that non-emptiness is satisfied becomes higher even within a relatively small number of correction steps.

\subsection{Parameter and lower-bound estimation in SFS}
\label{subapp:parameter-and-lower-bound-estimation-in-sfs}

Recall that we modeled the probability that the refinement target $c_t$ is a node at depth $0$ as
\begin{align}
    q_t^{\text{dep}}(0) = \exp \left(- \gamma (t - 1)\right) \quad (t \ge 1).
\end{align}
We estimated the parameter $\gamma$ appearing in this model empirically by applying SFS to real datasets. Specifically, we used gpt-4o-mini as the underlying model and applied Line to four datasets: HumanEval \citep{chen2021evaluating}, MBPP \citep{austin2021program}, APPS \citep{hendrycks2021measuring} and CodeContests \citep{li2022competition}. For each task, SFS generates five initial programs and then performs up to $11$ rounds of revision.

At each time step $t \in \{1, \cdots, 11\}$, we determined whether the program being revised was a node at depth $0$ and used these observations to compute the empirical probability $\hat{q}_t^{\text{dep}}(0)$ that the refinement target $c_t$ is a node at depth $0$. We then fit the model $q_t^{\text{dep}}(0) = \exp \left(- \gamma (t - 1)\right)$ to the resulting data points $\{(t, \hat{q}_t^{\text{old}})\}_{t = 1}^{11}$ via non-linear least squares to obtain an estimate of $\gamma$. Using this estimated value of $\gamma$, we computed the lower bound on the probability that the hypothesis space is monotonically non-increasing up to time $t = 11$.

The results are shown in Tables~\ref{tab:estimated-values-of-gamma} and~\ref{tab:estimated-values-of-lower-bound}. Since the estimated lower bound is non-negligible, we find that the hypothesis space in SFS is monotonically non-increasing with a certain probability. This suggests that SFS is more amenable to maintaining a monotonically non-increasing hypothesis space.

\begin{table}[htbp]
    \centering
    \begin{minipage}{0.49\linewidth}
        \caption{\textbf{Estimated values of} $\bm{\gamma}$.}
        \label{tab:estimated-values-of-gamma}
        \small
        \centering
        \begin{tabular}{c|c}
            \toprule
            \textbf{Dataset} & $\bm{\gamma}$ \\
            \midrule
            HumanEval & $0.01338 \pm 0.00149$ \\
            MBPP & $0.03989 \pm 0.00178$ \\
            APPS & $0.01832 \pm 0.00231$ \\
            CodeContests & $0.01084 \pm 0.00209$ \\
            \bottomrule
        \end{tabular}
    \end{minipage}
    \hfill
    \begin{minipage}{0.49\linewidth}
        \caption{\textbf{Estimated values of} $\bm{1 - \delta_{\textbf{SFS}}}$.}
        \label{tab:estimated-values-of-lower-bound}
        \small
        \centering
        \begin{tabular}{c|c}
            \toprule
            \textbf{Dataset} & $\bm{1 - \delta_{\textbf{SFS}} \quad (t = 11)}$ \\
            \midrule
            HumanEval & $0.4791$ \\
            MBPP & $0.1115$ \\
            APPS & $0.3651$ \\
            CodeContests & $0.5508$ \\
            \bottomrule
        \end{tabular}
    \end{minipage}
\end{table}

\FloatBarrier
\section{Additional results for inference scaling}
\label{app:additional-results-for-inference-scaling}

The results are shown in Figures~\ref{fig:scaling-gpt-4o-mini-humaneval-and-mbpp}-~\ref{fig:scaling-llama-3.2-3b-instruct-humaneval}.

\begin{figure}[htbp]
    \centering
    \begin{minipage}{0.49\linewidth}
        \centering
        \includegraphics[width=\linewidth]{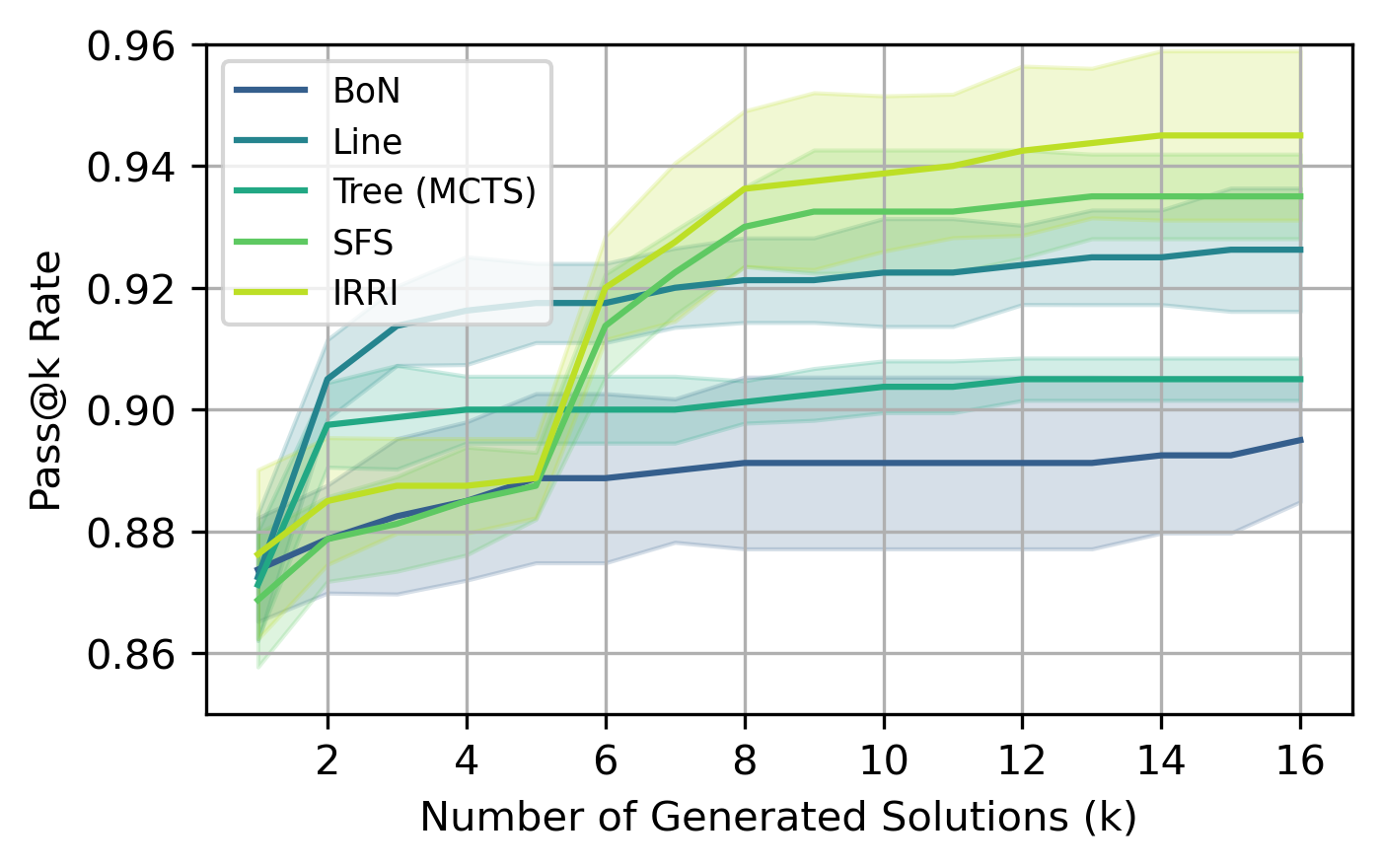}
    \end{minipage}
    \begin{minipage}{0.49\linewidth}
        \centering
        \includegraphics[width=\linewidth]{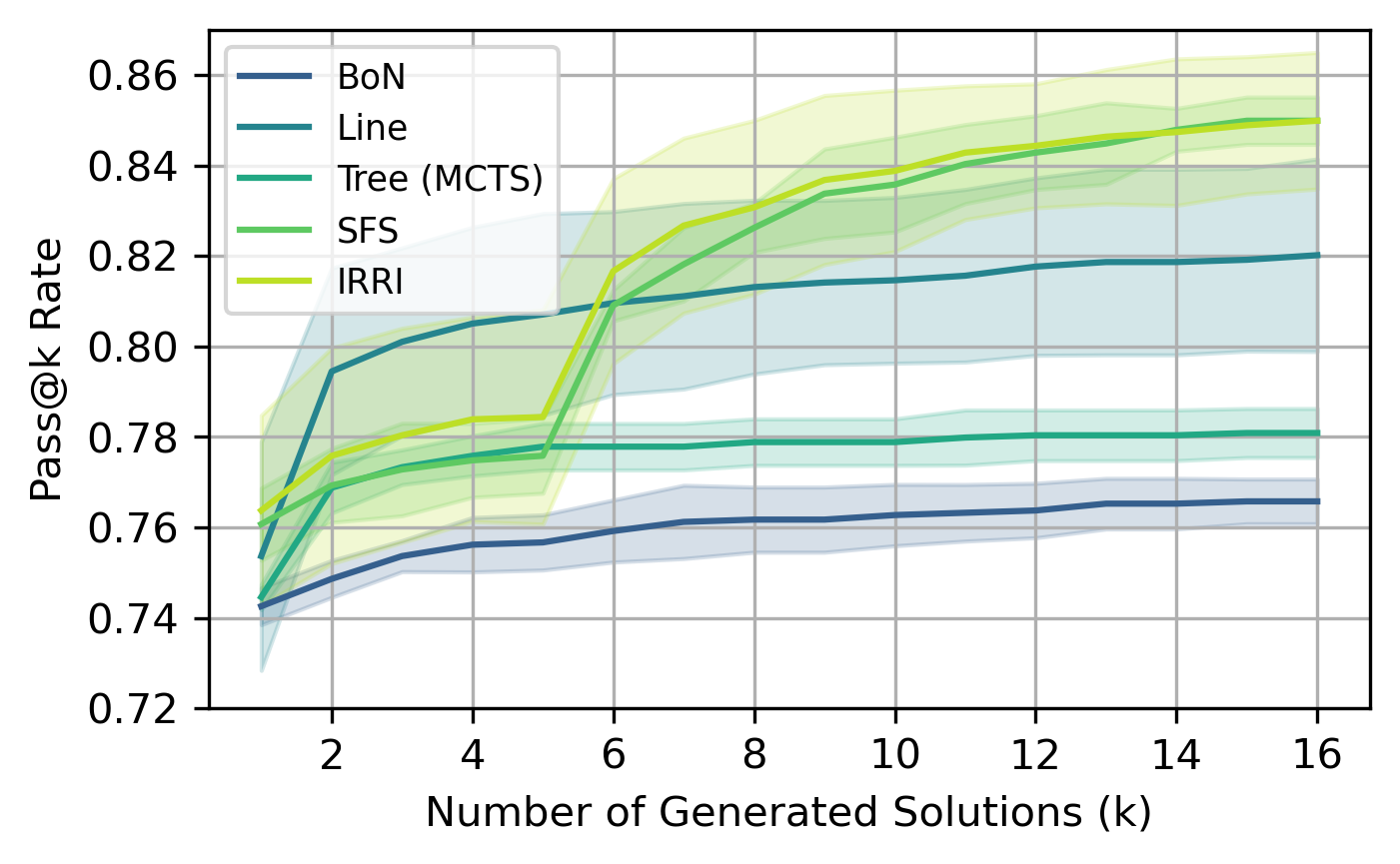}
    \end{minipage}
    \caption{\textbf{Scaling curves for different methods on HumanEval using gpt-4o-mini (left) and on MBPP using gpt-4o-mini (right).}}
    \label{fig:scaling-gpt-4o-mini-humaneval-and-mbpp}
\end{figure}

\begin{figure}[htbp]
    \centering
    \begin{minipage}{0.49\linewidth}
        \centering
        \includegraphics[width=\linewidth]{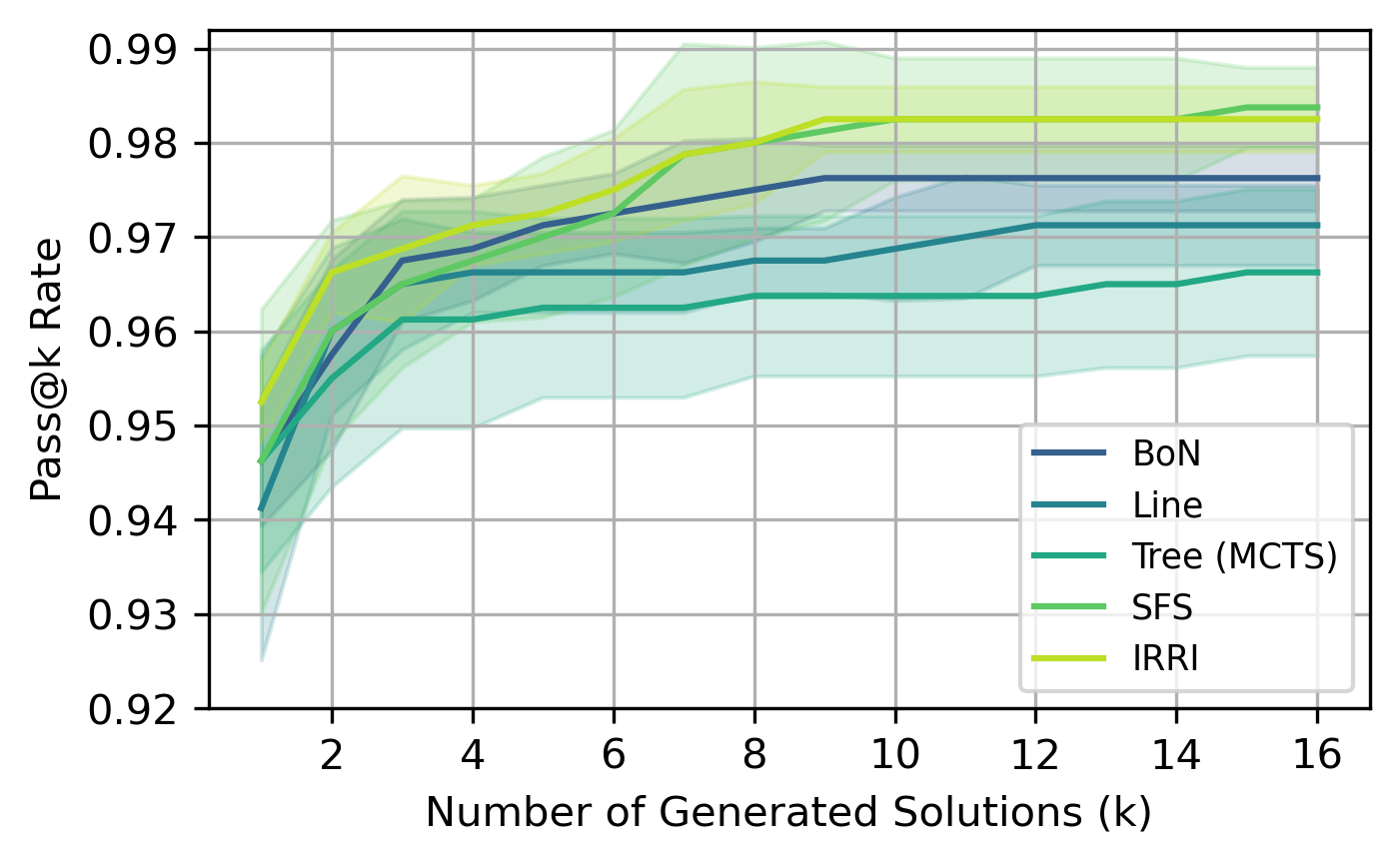}
    \end{minipage}
    \begin{minipage}{0.49\linewidth}
        \centering
        \includegraphics[width=\linewidth]{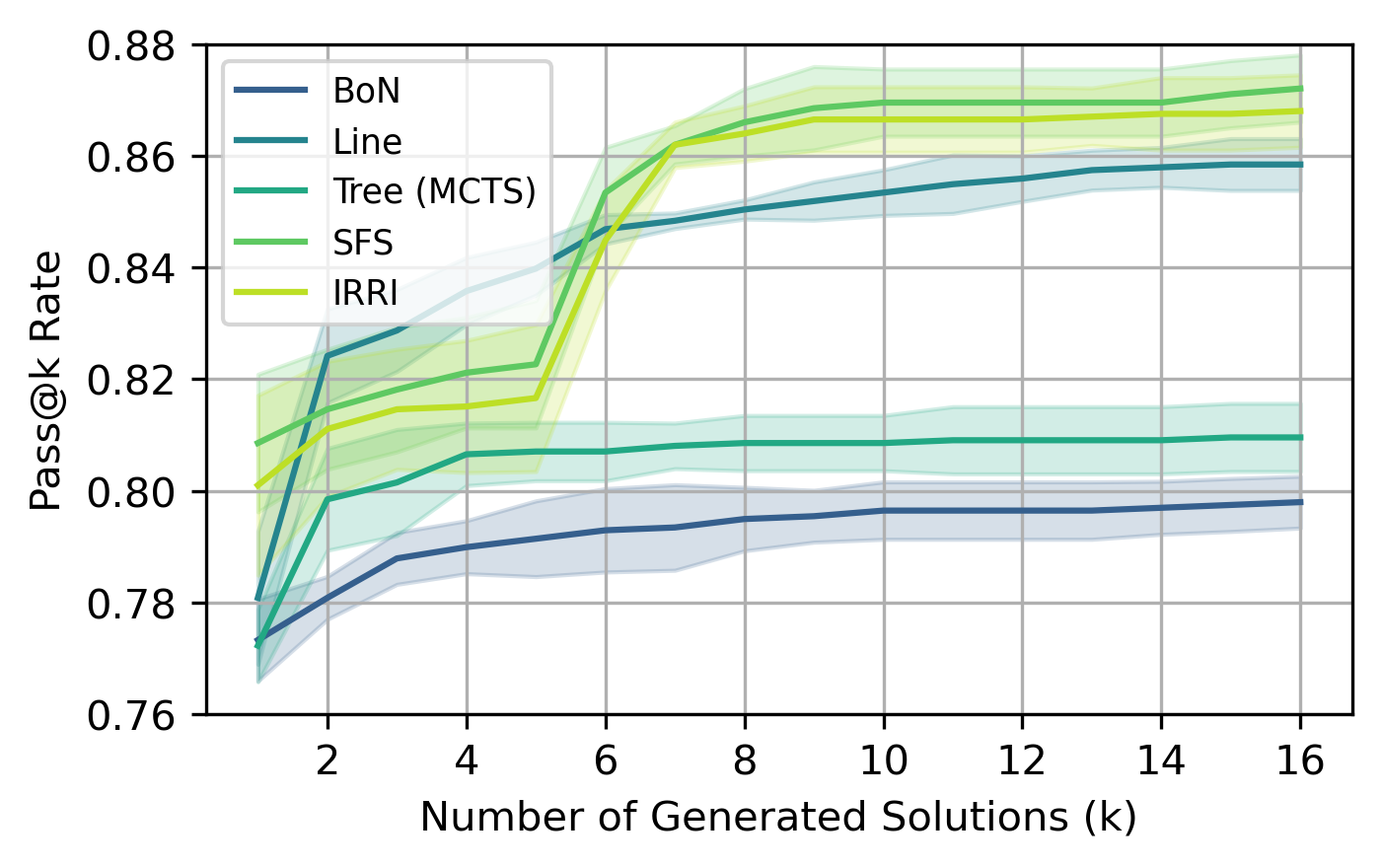}
    \end{minipage}
    \caption{\textbf{Scaling curves for different methods on HumanEval using gpt-4.1-mini (left) and on MBPP using gpt-4.1-mini (right).}}
    \label{fig:scaling-gpt-4.1-mini-humaneval-and-mbpp}
\end{figure}

\begin{figure}[htbp]
    \centering
    \begin{minipage}{0.49\linewidth}
        \centering
        \includegraphics[width=\linewidth]{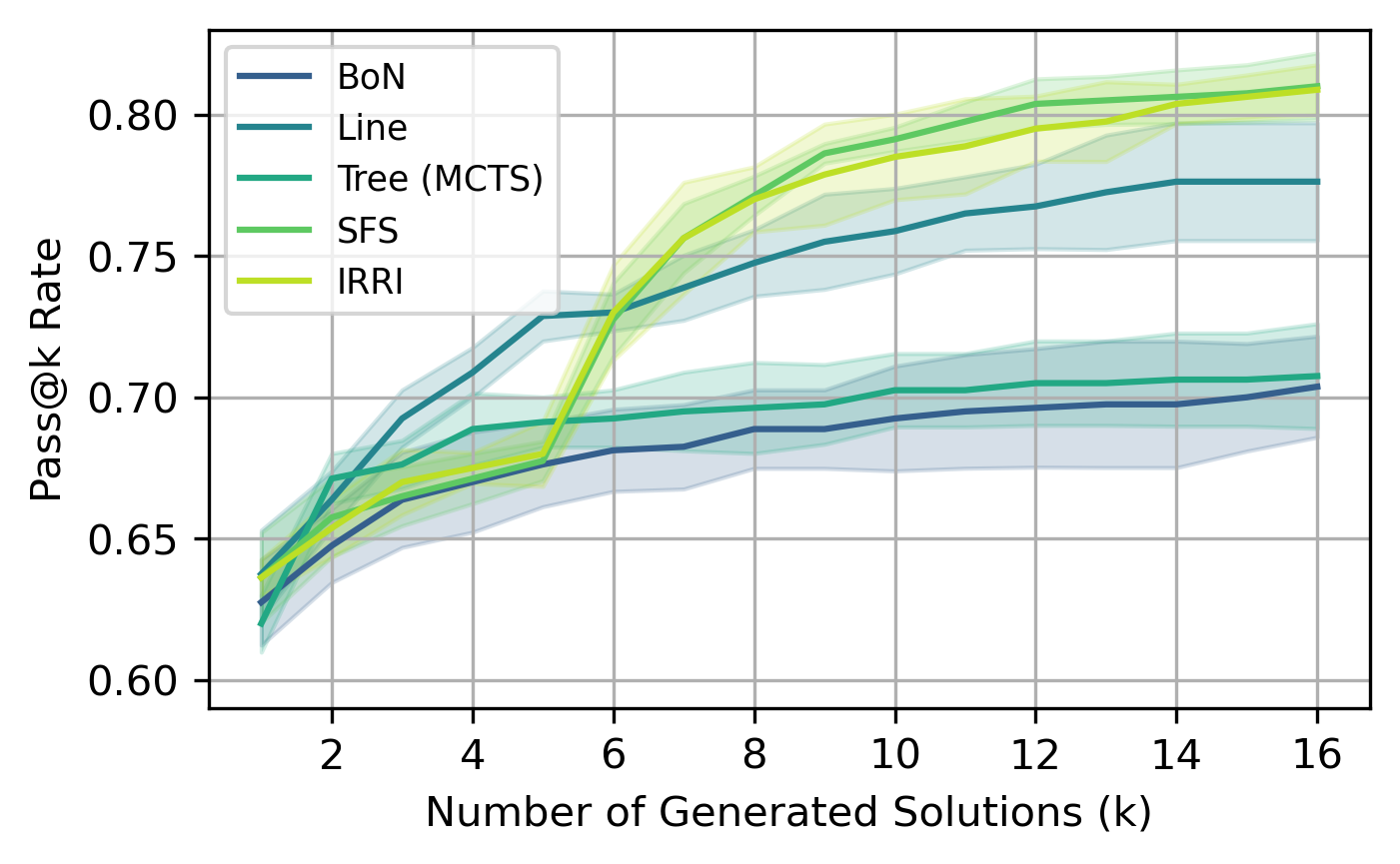}
    \end{minipage}
    \begin{minipage}{0.49\linewidth}
        \centering
        \includegraphics[width=\linewidth]{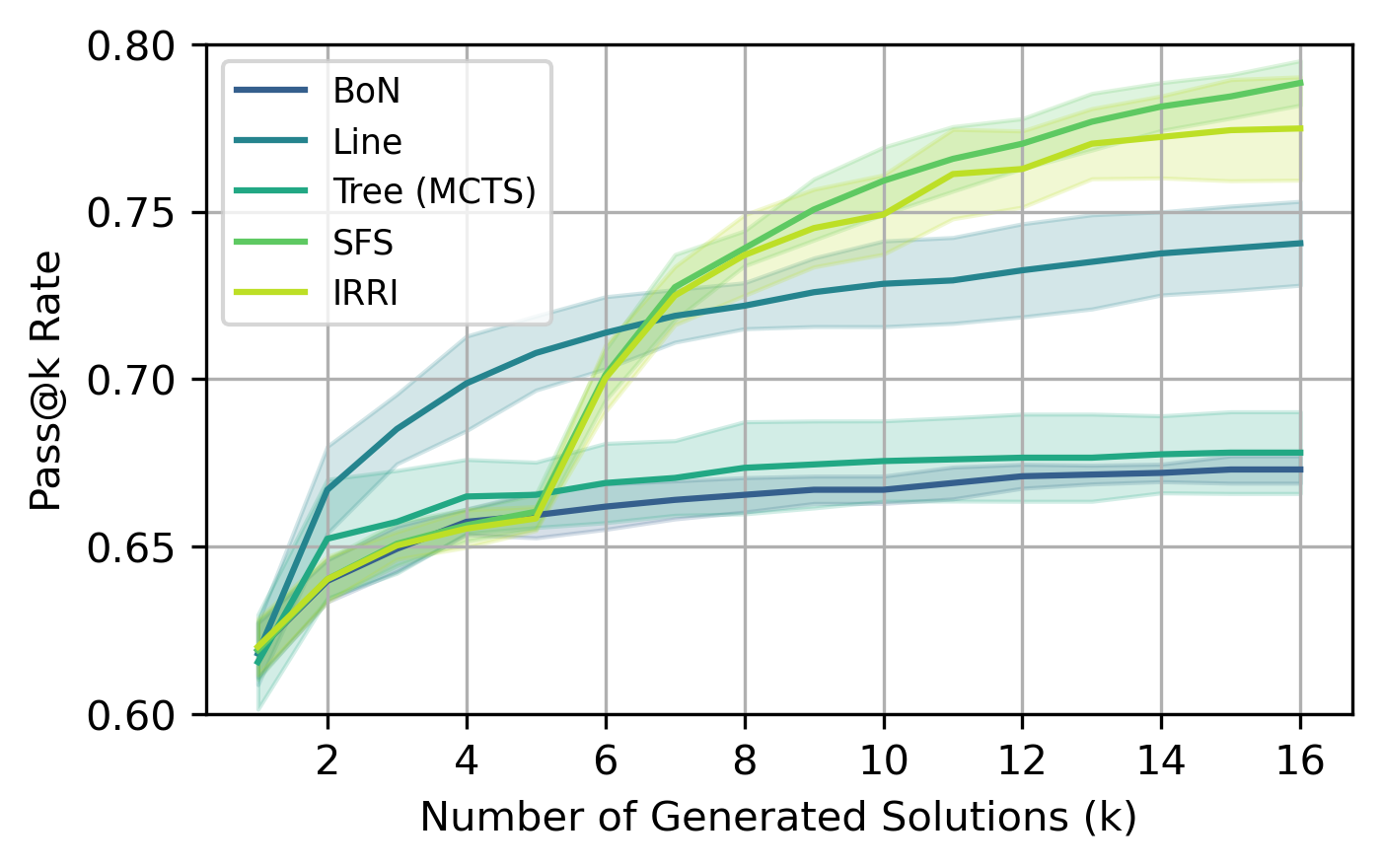}
    \end{minipage}
    \caption{\textbf{Scaling curves for different methods on HumanEval using Llama-3.1-8B-Instruct (left) and on MBPP using Llama-3.1-8B-Instruct (right).}}
    \label{fig:scaling-llama-3.1-8b-instruct-humaneval-and-mbpp}
\end{figure}

\begin{figure}[htbp]
    \centering
    \begin{minipage}{0.49\linewidth}
        \centering
        \includegraphics[width=\linewidth]{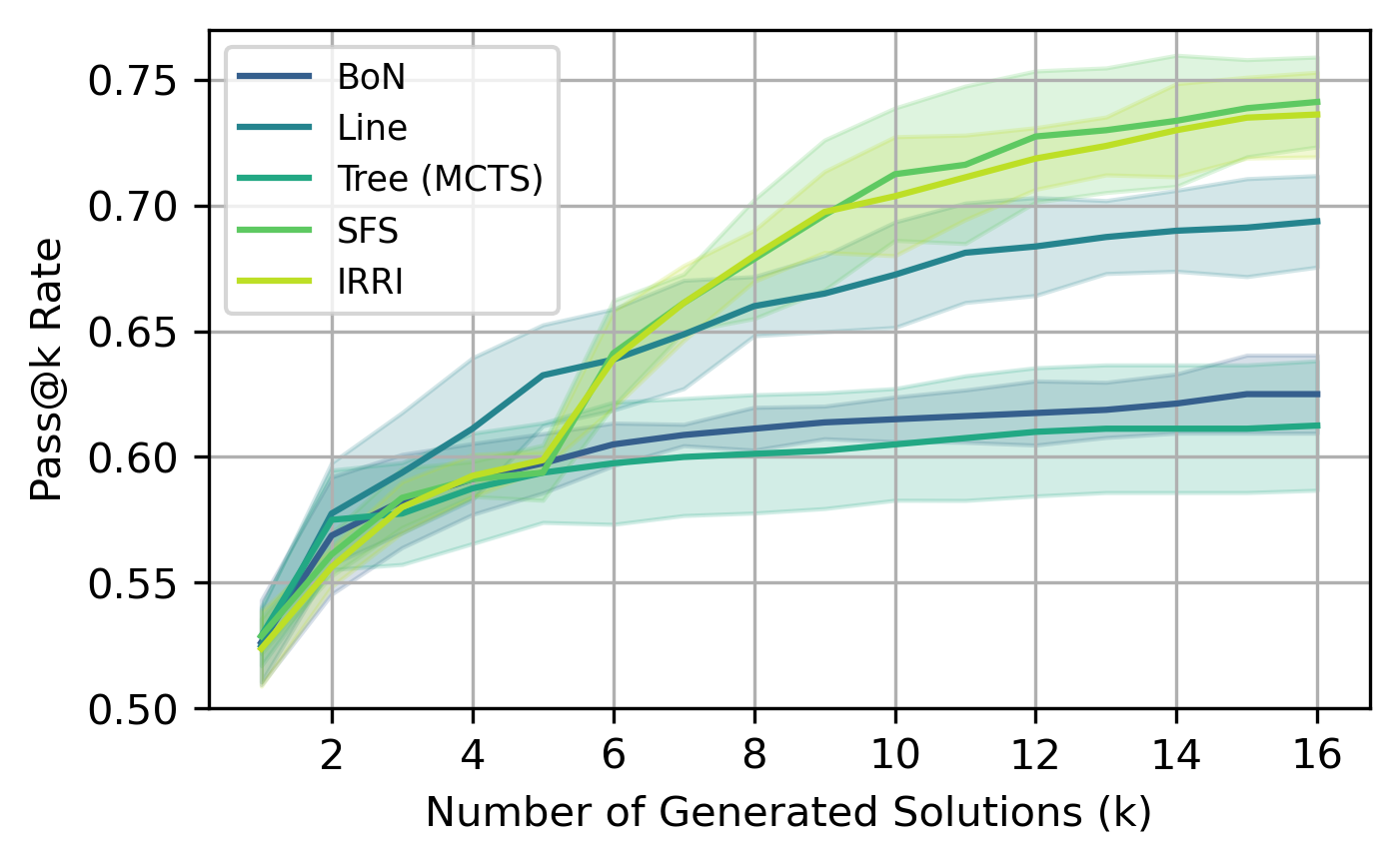}
    \end{minipage}
    \caption{\textbf{Scaling curves for different methods on HumanEval using Llama-3.2-3B-Instruct.}}
    \label{fig:scaling-llama-3.2-3b-instruct-humaneval}
\end{figure}

\FloatBarrier
\section{Additional results for diversity of repair instructions}
\label{app:additional-results-for-dicersity-of-repair-instructions}

The results are shown in Figures~\ref{fig:diversity-llama-3.2-3b-instruct-humaneval} and~\ref{fig:diversity-llama-3.2-3b-instruct-mbpp}.

\begin{figure}[htbp]
    \centering
    \begin{minipage}{0.32\linewidth}
        \centering
        \includegraphics[width=\linewidth]{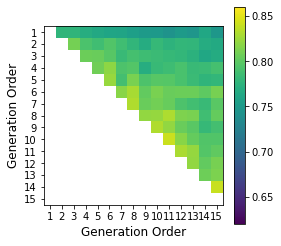}
        \subcaption{Line}
    \end{minipage}
    \begin{minipage}{0.32\linewidth}
        \centering
        \includegraphics[width=\linewidth]{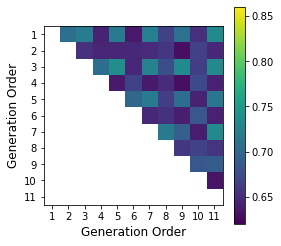}
        \subcaption{IRRI}
    \end{minipage}
    \begin{minipage}{0.32\linewidth}
        \centering
        \includegraphics[width=\linewidth]{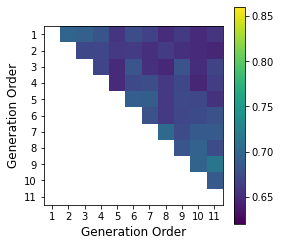}
        \subcaption{SFS}
    \end{minipage}
    \caption{\textbf{BERT similarity heatmaps on HumanEval using Llama-3.2-3B-Instruct.}}
    \label{fig:diversity-llama-3.2-3b-instruct-humaneval}
\end{figure}

\begin{figure}[htbp]
    \centering
    \begin{minipage}{0.32\linewidth}
        \centering
        \includegraphics[width=\linewidth]{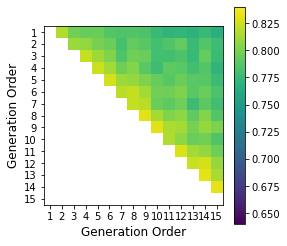}
        \subcaption{Line}
    \end{minipage}
    \begin{minipage}{0.32\linewidth}
        \centering
        \includegraphics[width=\linewidth]{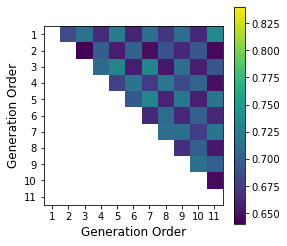}
        \subcaption{IRRI}
    \end{minipage}
    \begin{minipage}{0.32\linewidth}
        \centering
        \includegraphics[width=\linewidth]{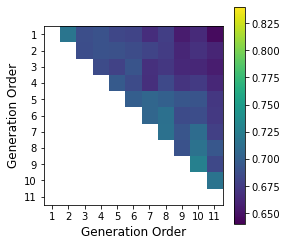}
        \subcaption{SFS}
    \end{minipage}
    \caption{\textbf{BERT similarity heatmaps on MBPP using Llama-3.2-3B-Instruct.}}
    \label{fig:diversity-llama-3.2-3b-instruct-mbpp}
\end{figure}

\FloatBarrier
\section{Ablation studies}
\subsection{Other benchmarks}
We compared the inference scaling performance of SFS and IRRI on CodeContests \citep{li2022competition} with gpt-4o-mini as the base model. We also compared the inference scaling performance of SFS and IRRI on Leetcode \citep{guo2024deepseek} with Llama-3.1-8B-Instruct as the base model. Each curve represents the mean of three runs, and each shaded area indicates 95\% confidence intervals based on the $t$-distribution. From Figure~\ref{fig:scaling-ablations}, we found that IRRI achieves inference scaling performance comparable to SFS. Moreover, to improve the interpretability of the CodeContests result, we evaluated the performance difference between SFS and IRRI (Figure~\ref{fig:scaling-diff-ablation}). We found that the average performance difference between the two methods was small, at approximately $2 \%$ or less. On the other hand, we observed that the variability increased as inference progressed.

\begin{figure}[htbp]
    \centering
    \begin{minipage}{0.49\linewidth}
        \centering
        \includegraphics[width=\linewidth]{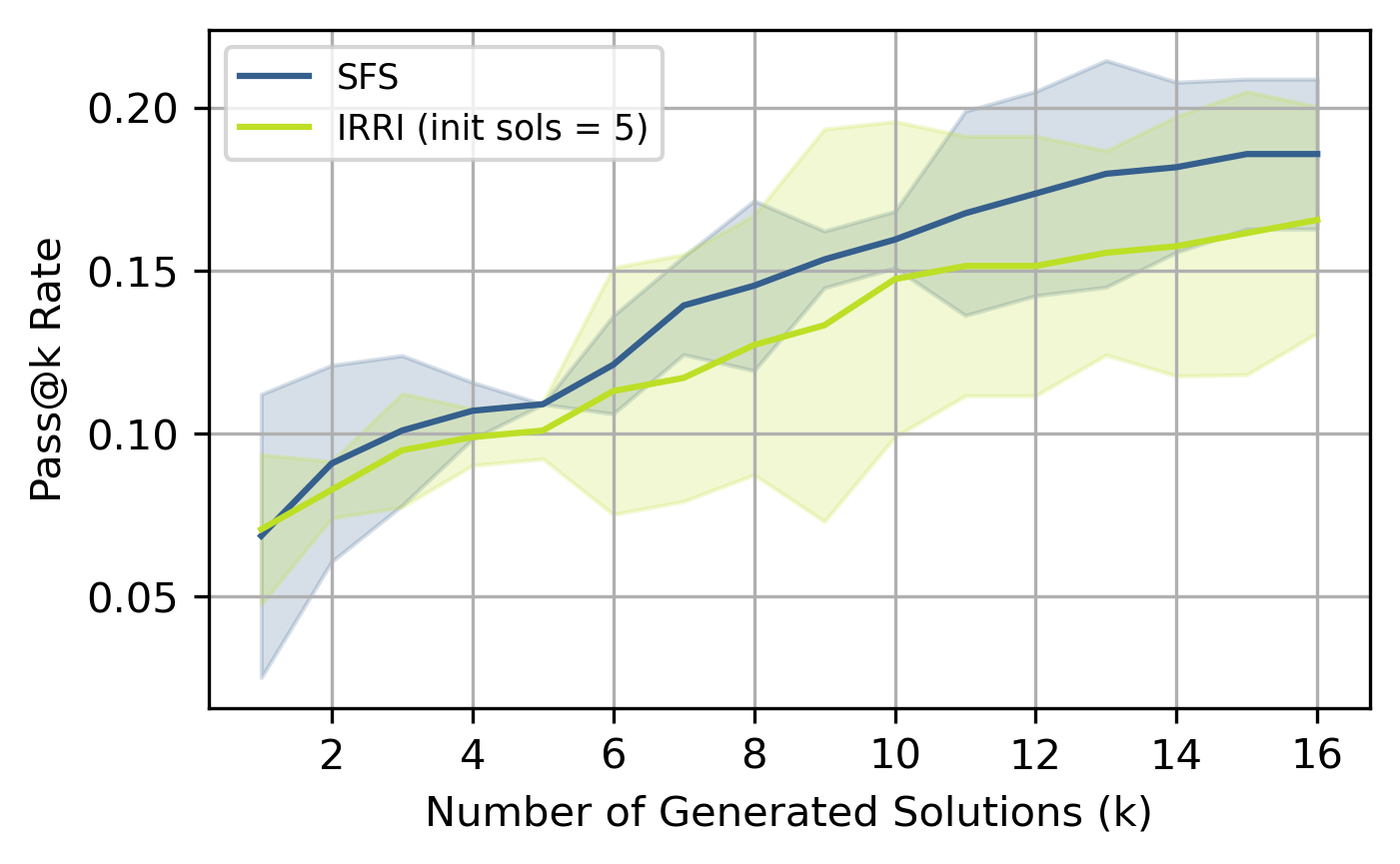}
    \end{minipage}
    \begin{minipage}{0.49\linewidth}
        \centering
        \includegraphics[width=\linewidth]{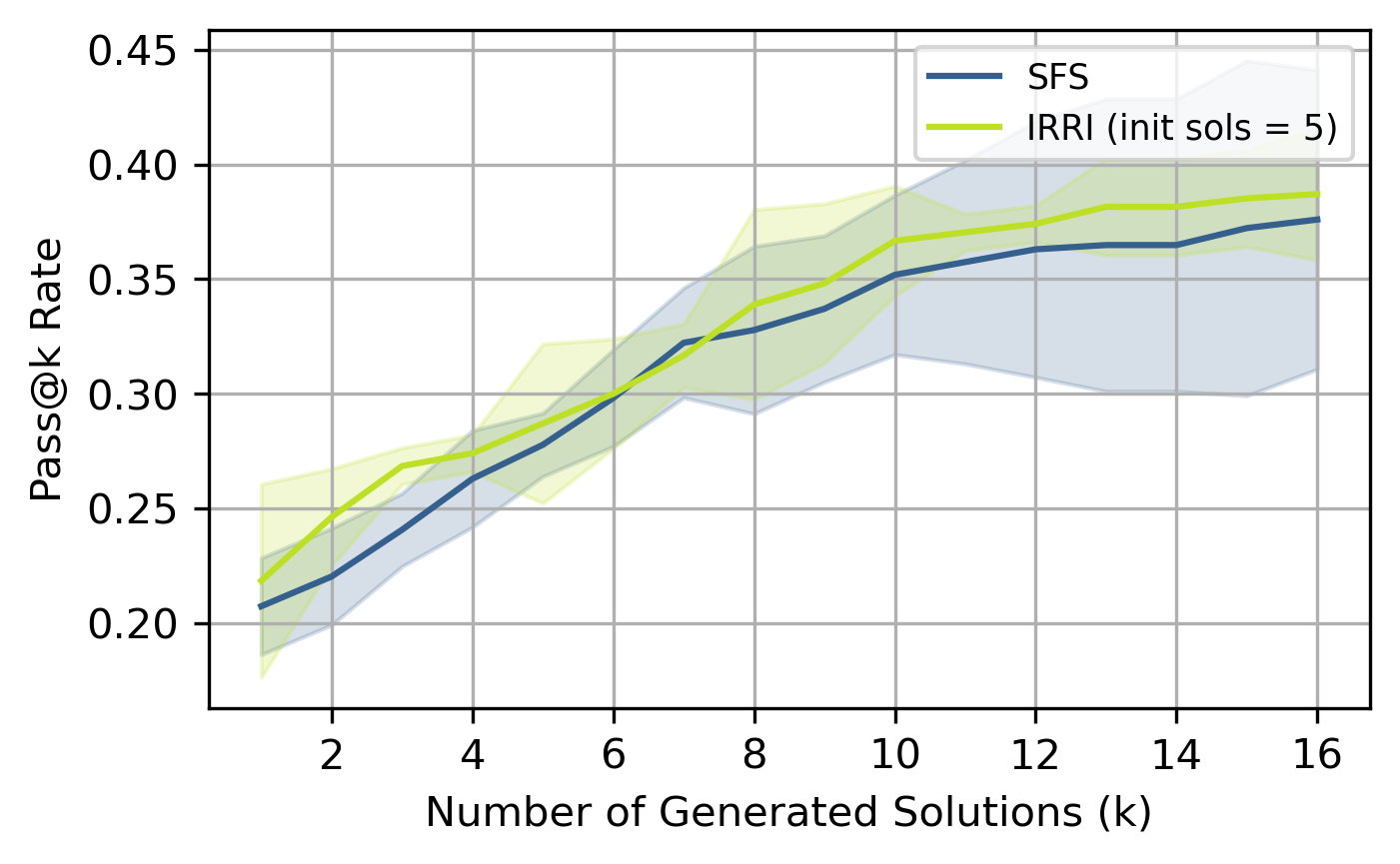}
    \end{minipage}
    \caption{\textbf{Scaling curves for different methods on CodeContests using gpt-4o-mini (left) and on Leetcode using Llama-3.1-8B-Instruct (right).}}
    \label{fig:scaling-ablations}
\end{figure}

\begin{figure}[htbp]
    \centering
    \begin{minipage}{0.49\linewidth}
        \centering
        \includegraphics[width=\linewidth]{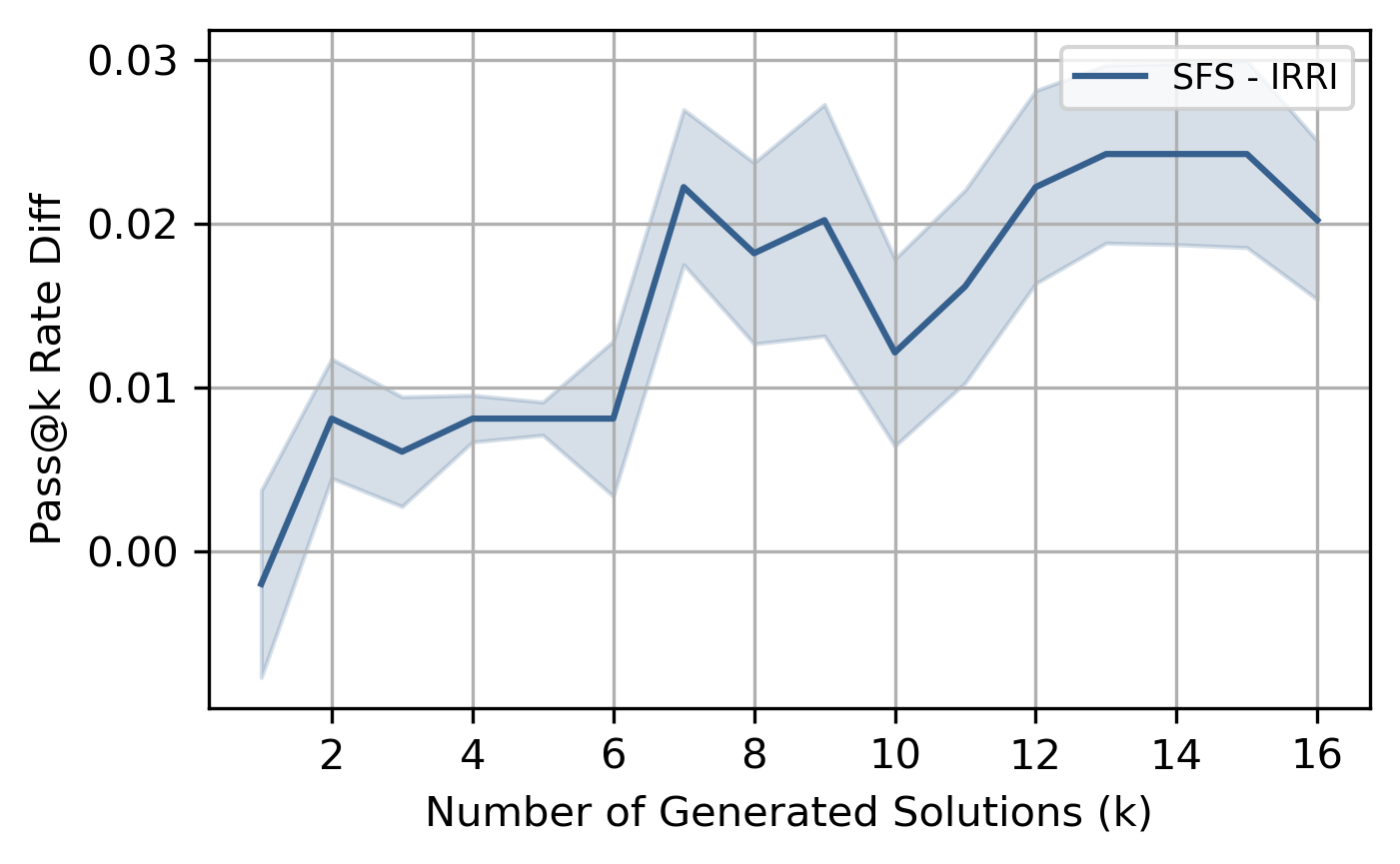}
    \end{minipage}
    \caption{\textbf{Performance difference between SFS and IRRI on CodeContests using gpt-4o-mini.}}
    \label{fig:scaling-diff-ablation}
\end{figure}

\FloatBarrier
\subsection{Computational cost comparison}
\label{subapp:computational-cost-comparison}

To reveal the efficiency of IRRI, we compared SFS and IRRI in terms of the number of LLM calls, the number of output tokens, and the total number of tokens, reporting the average over five runs (Tables~\ref{tab:cost-gpt-4o-mini-apps}-~\ref{tab:cost-llama-3.2-3b-instruct-humaneval-and-mbpp}). The LLM calls, output tokens, and total tokens reported in all tables include every auxiliary call, such as generating repair instructions and updating the shared information. We found that IRRI requires fewer LLM calls and consumes fewer tokens than SFS, highlighting the advantages of IRRI.

\begin{table}[htbp]
    \centering
    \caption{\textbf{Computational cost comparison on APPS using gpt-4o-mini.}}
    \label{tab:cost-gpt-4o-mini-apps}
    \begin{minipage}{0.49\linewidth}
        \small
        \centering
        \begin{tabular}{c|cc}
            \toprule
            \textbf{Metric} & \textbf{SFS} & \textbf{IRRI} \\
            \midrule
            Average LLM calls & 30.95 & 29.64 \\
            Average output tokens & 8228 & 7792 \\
            Average total tokens & 56348 & 52900 \\
            \bottomrule
        \end{tabular}
    \end{minipage}
\end{table}

\begin{table}[htbp]
    \centering
    \caption{\textbf{Computational cost comparison on HumanEval using gpt-4o-mini (left) and on MBPP using gpt-4o-mini (right).}}
    \label{tab:cost-gpt-4o-mini-humaneval-and-mbpp}
    \begin{minipage}{0.49\linewidth}
        \small
        \centering
        \begin{tabular}{c|cc}
            \toprule
            \textbf{Metric} & \textbf{SFS} & \textbf{IRRI} \\
            \midrule
            Average LLM calls & 15.47 & 14.04 \\
            Average output tokens & 3356 & 3009 \\
            Average total tokens & 20714 & 18290 \\
            \bottomrule
        \end{tabular}
    \end{minipage}
    \hfill
    \begin{minipage}{0.49\linewidth}
        \small
        \centering
        \begin{tabular}{c|cc}
            \toprule
            \textbf{Metric} & \textbf{SFS} & \textbf{IRRI} \\
            \midrule
            Average LLM calls & 14.23 & 13.11 \\
            Average output tokens & 2725 & 2463 \\
            Average total tokens & 16339 & 14574 \\
            \bottomrule
        \end{tabular}
    \end{minipage}
\end{table}

\begin{table}[htbp]
    \centering
    \caption{\textbf{Computational cost comparison on HumanEval using gpt-4.1-mini (left) and on MBPP using gpt-4.1-mini (right).}}
    \label{tab:cost-gpt-4.1-mini-humaneval-and-mbpp}
    \begin{minipage}{0.49\linewidth}
        \small
        \centering
        \begin{tabular}{c|cc}
            \toprule
            \textbf{Metric} & \textbf{SFS} & \textbf{IRRI} \\
            \midrule
            Average LLM calls & 9.01 & 8.11 \\
            Average output tokens & 2384 & 2118 \\
            Average total tokens & 12396 & 10753 \\
            \bottomrule
        \end{tabular}
    \end{minipage}
    \hfill
    \begin{minipage}{0.49\linewidth}
        \small
        \centering
        \begin{tabular}{c|cc}
            \toprule
            \textbf{Metric} & \textbf{SFS} & \textbf{IRRI} \\
            \midrule
            Average LLM calls & 9.51 & 8.24 \\
            Average output tokens & 2186 & 1826 \\
            Average total tokens & 11078 & 8986 \\
            \bottomrule
        \end{tabular}
    \end{minipage}
\end{table}

\begin{table}[htbp]
    \centering
    \caption{\textbf{Computational cost comparison on HumanEval using Llama-3.1-8B-Instruct (left) and on MBPP using Llama-3.1-8B-Instruct (right).}}
    \label{tab:cost-llama-3.1-8b-instruct-humaneval-and-mbpp}
    \begin{minipage}{0.49\linewidth}
        \small
        \centering
        \begin{tabular}{c|cc}
            \toprule
            \textbf{Metric} & \textbf{SFS} & \textbf{IRRI} \\
            \midrule
            Average LLM calls & 19.79 & 18.02 \\
            Average output tokens & 5140 & 4907 \\
            Average total tokens & 27037 & 23286 \\
            \bottomrule
        \end{tabular}
    \end{minipage}
    \hfill
    \begin{minipage}{0.49\linewidth}
        \small
        \centering
        \begin{tabular}{c|cc}
            \toprule
            \textbf{Metric} & \textbf{SFS} & \textbf{IRRI} \\
            \midrule
            Average LLM calls & 20.97 & 19.37 \\
            Average output tokens & 5109 & 4621 \\
            Average total tokens & 26069 & 23286 \\
            \bottomrule
        \end{tabular}
    \end{minipage}
\end{table}

\begin{table}[htbp]
    \centering
    \caption{\textbf{Computational cost comparison on HumanEval using Llama-3.2-3B-Instruct (left) and on MBPP using Llama-3.2-3B-Instruct (right).}}
    \label{tab:cost-llama-3.2-3b-instruct-humaneval-and-mbpp}
    \begin{minipage}{0.49\linewidth}
        \small
        \centering
        \begin{tabular}{c|cc}
            \toprule
            \textbf{Metric} & \textbf{SFS} & \textbf{IRRI} \\
            \midrule
            Average LLM calls & 24.49 & 23.02 \\
            Average output tokens & 8181 & 6281 \\
            Average total tokens & 36797 & 31806 \\
            \bottomrule
        \end{tabular}
    \end{minipage}
    \hfill
    \begin{minipage}{0.49\linewidth}
        \small
        \centering
        \begin{tabular}{c|cc}
            \toprule
            \textbf{Metric} & \textbf{SFS} & \textbf{IRRI} \\
            \midrule
            Average LLM calls & 24.41 & 23.22 \\
            Average output tokens & 6339 & 5829 \\
            Average total tokens & 31390 & 28850 \\
            \bottomrule
        \end{tabular}
    \end{minipage}
\end{table}

\FloatBarrier
\subsection{Inference scaling by token usage}
\label{subapp:inference-scaling-by-token-usage}

To compare efficiency across different methods, we evaluated inference scaling performance with cumulative token usage on the x-axis. The results are shown in Figures~\ref{fig:scaling-token-gpt-4o-mini-apps}-~\ref{fig:scaling-token-llama-3.2-3b-instruct-humaneval-and-mbpp}. We found that SFS and IRRI perform on par with prior methods at small budgets, but scale better as the token budget increases. Notably, the two methods exhibit similar behavior across the entire budget range.

\begin{figure}[htbp]
    \centering
    \begin{minipage}{0.49\linewidth}
        \centering
        \includegraphics[width=\linewidth]{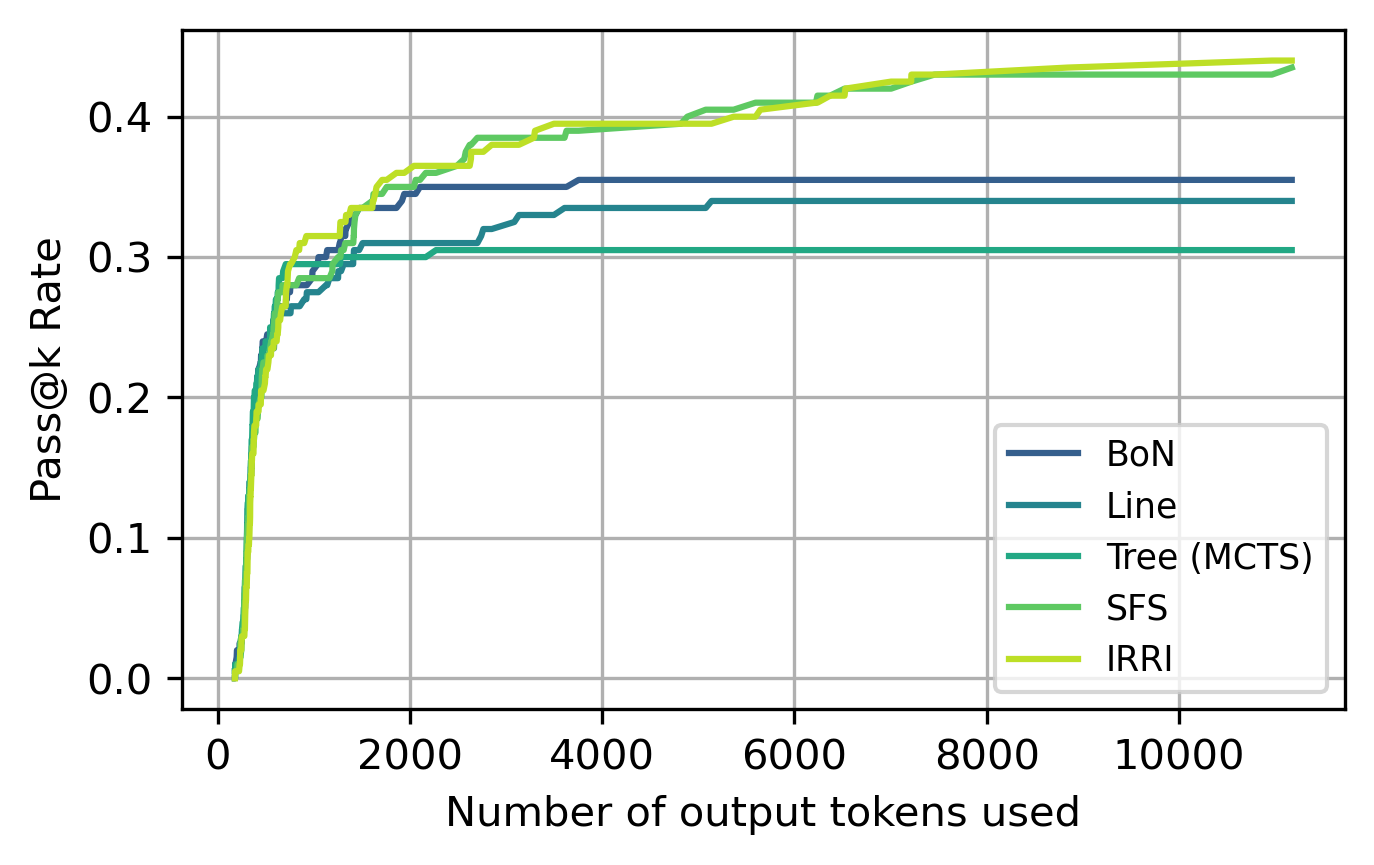}
    \end{minipage}
    \caption{\textbf{Scaling curves by token usage for different methods on APPS using gpt-4o-mini.}}
    \label{fig:scaling-token-gpt-4o-mini-apps}
\end{figure}

\begin{figure}[htbp]
    \centering
    \begin{minipage}{0.49\linewidth}
        \centering
        \includegraphics[width=\linewidth]{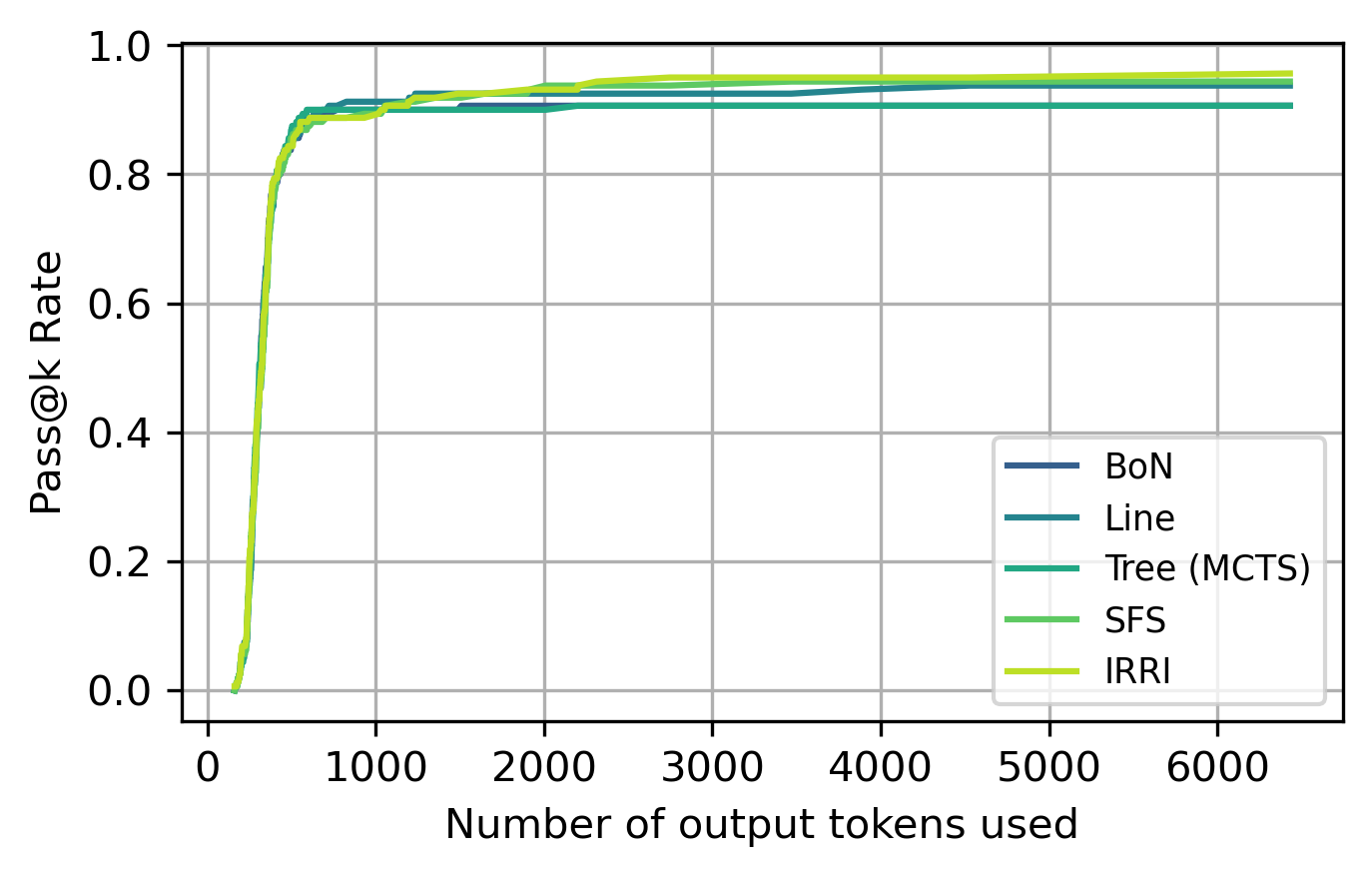}
    \end{minipage}
    \begin{minipage}{0.49\linewidth}
        \centering
        \includegraphics[width=\linewidth]{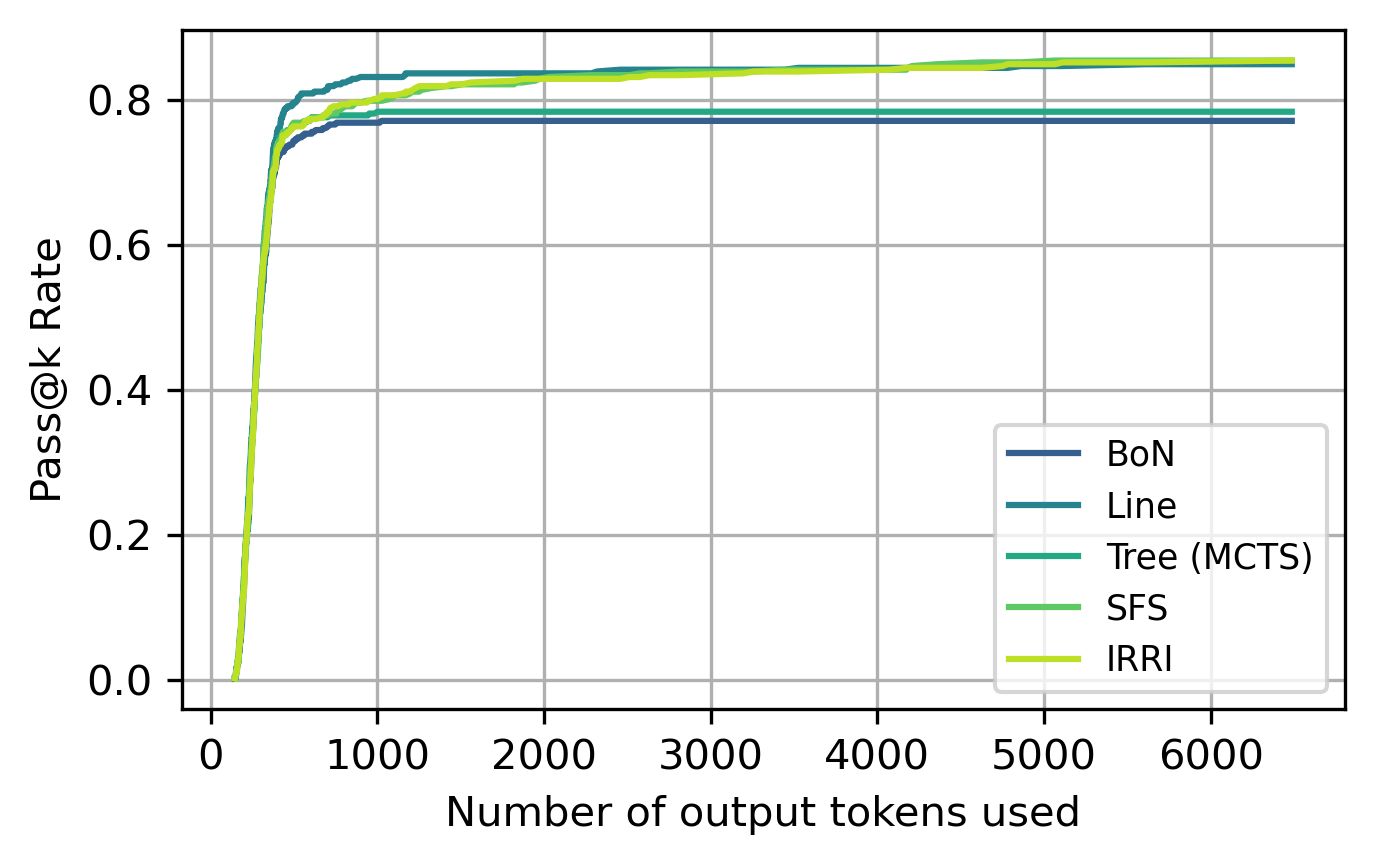}
    \end{minipage}
    \caption{\textbf{Scaling curves by token usage for different methods on HumanEval using gpt-4o-mini (left) and on MBPP using gpt-4o-mini (right).}}
    \label{fig:scaling-token-gpt-4o-mini-humaneval-and-mbpp}
\end{figure}

\begin{figure}[htbp]
    \centering
    \begin{minipage}{0.49\linewidth}
        \centering
        \includegraphics[width=\linewidth]{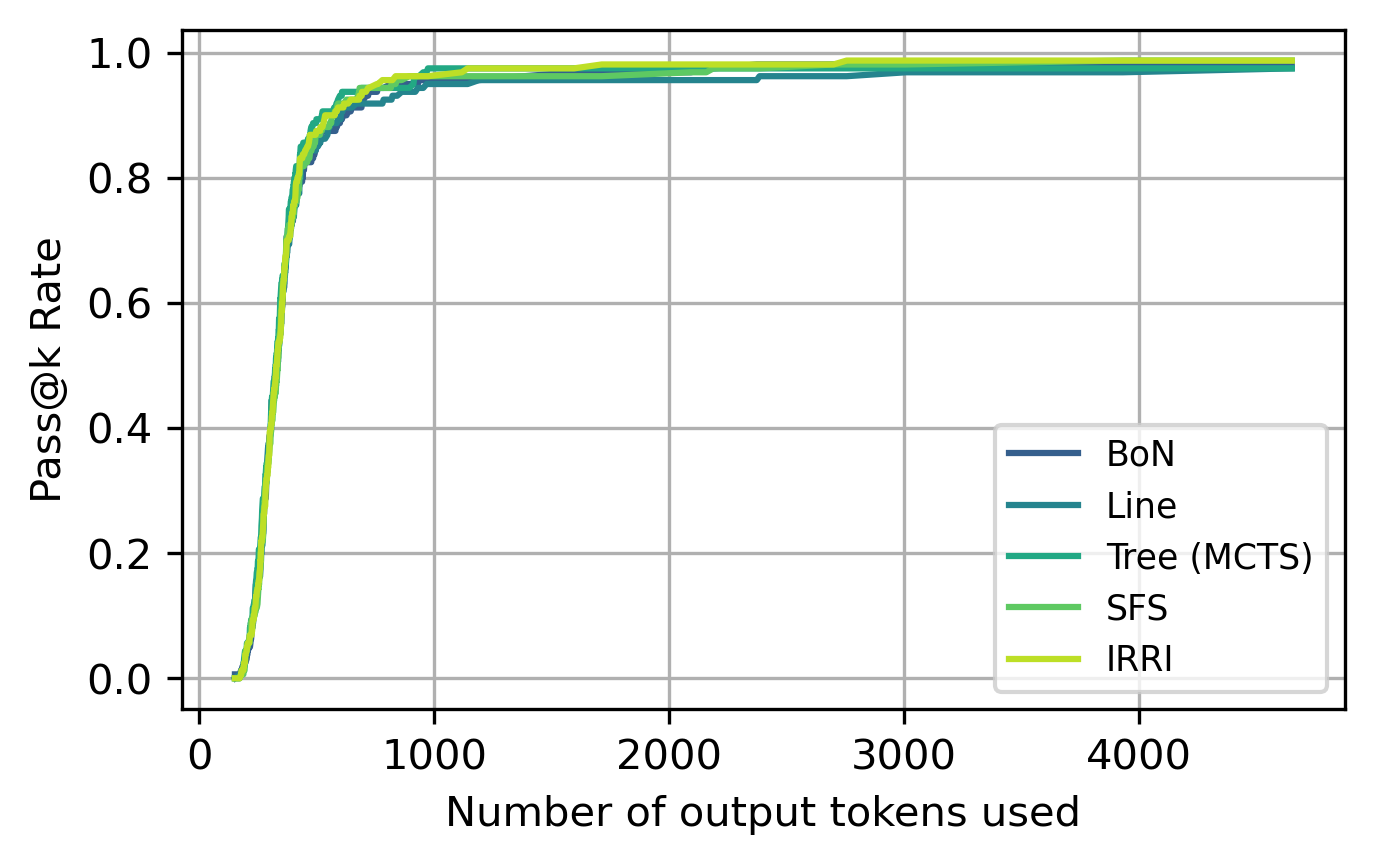}
    \end{minipage}
    \begin{minipage}{0.49\linewidth}
        \centering
        \includegraphics[width=\linewidth]{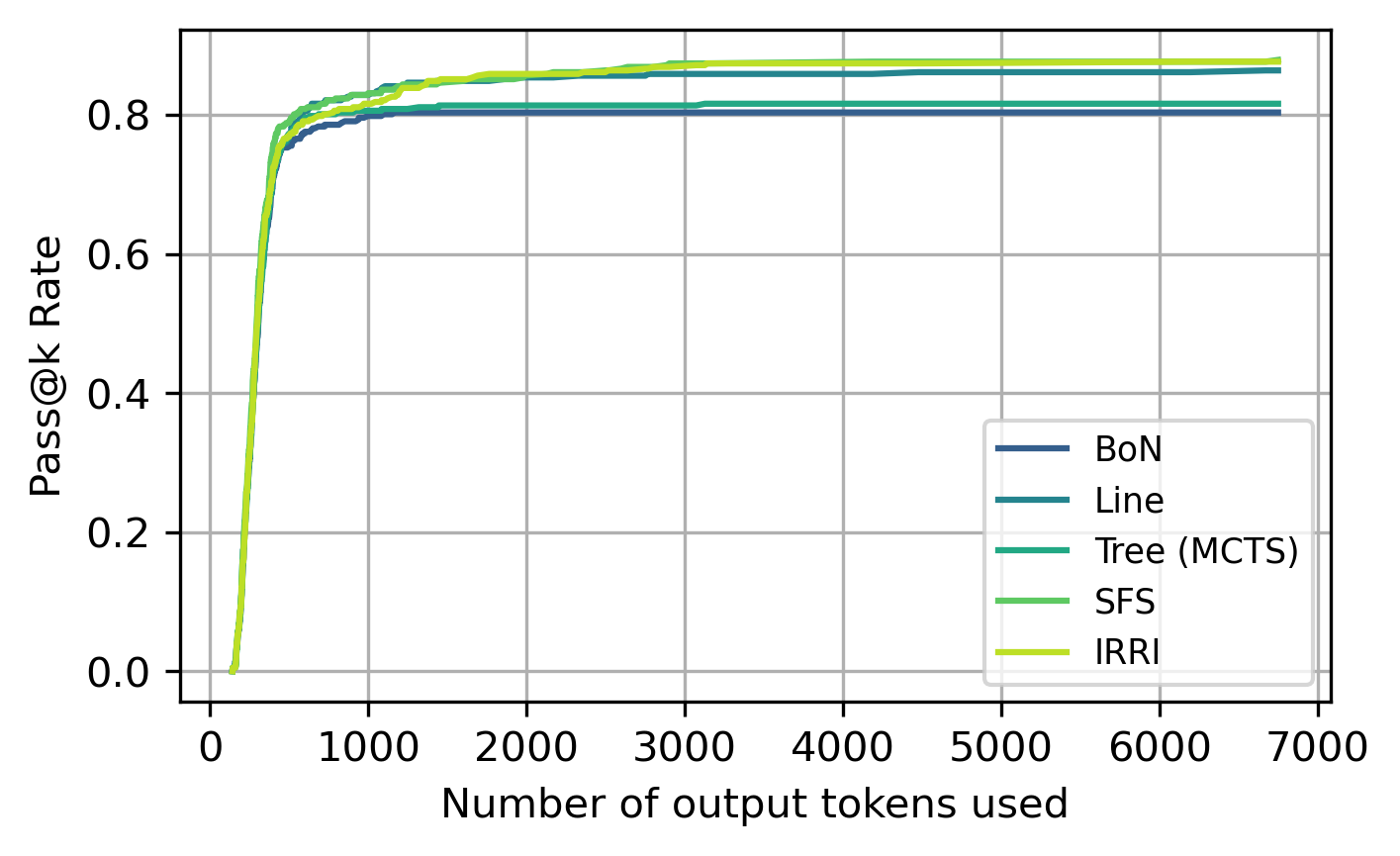}
    \end{minipage}
    \caption{\textbf{Scaling curves by token usage for different methods on HumanEval using gpt-4.1-mini (left) and on MBPP using gpt-4.1-mini (right).}}
    \label{fig:scaling-token-gpt-4.1-mini-humaneval-and-mbpp}
\end{figure}

\begin{figure}[htbp]
    \centering
    \begin{minipage}{0.49\linewidth}
        \centering
        \includegraphics[width=\linewidth]{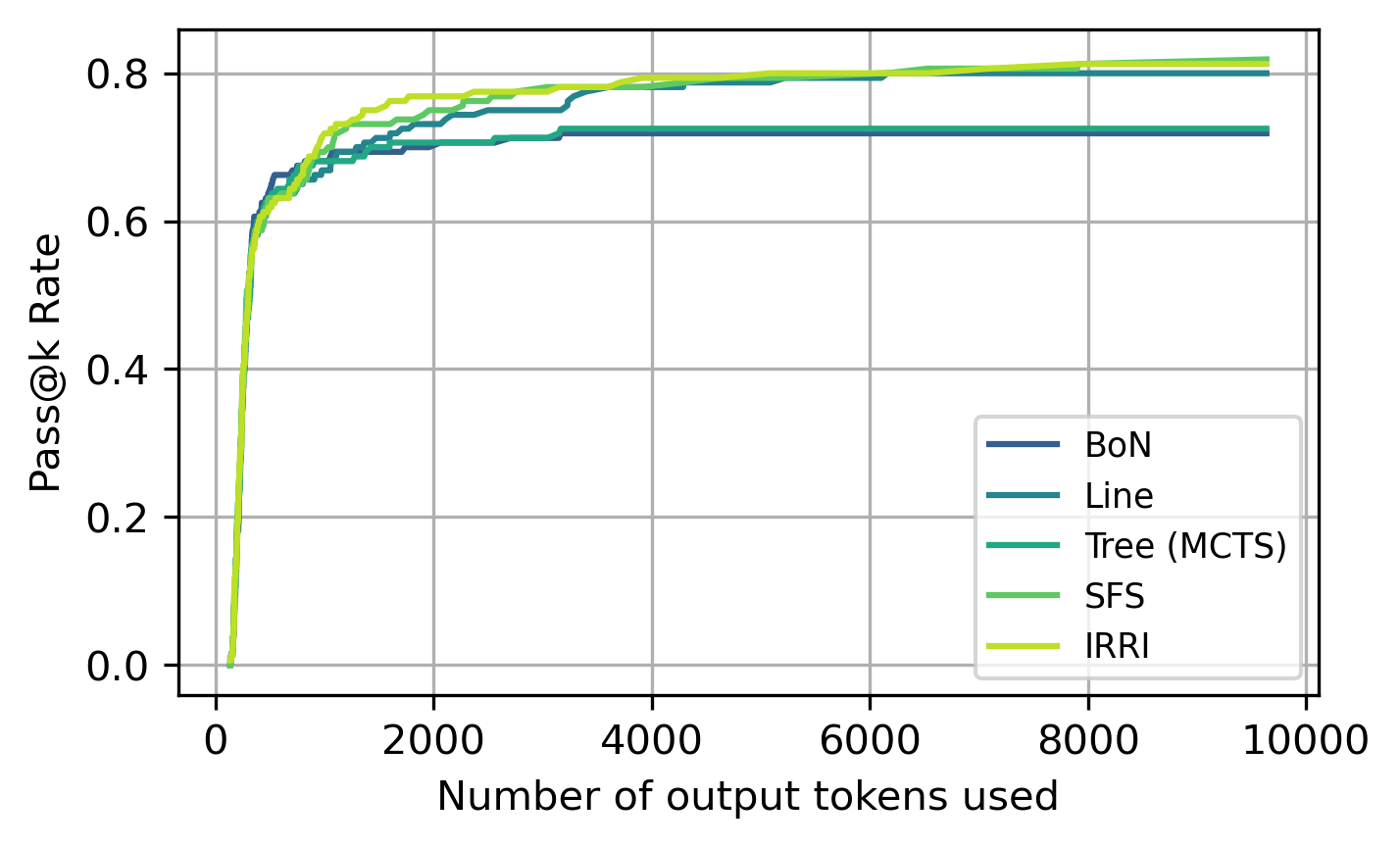}
    \end{minipage}
    \begin{minipage}{0.49\linewidth}
        \centering
        \includegraphics[width=\linewidth]{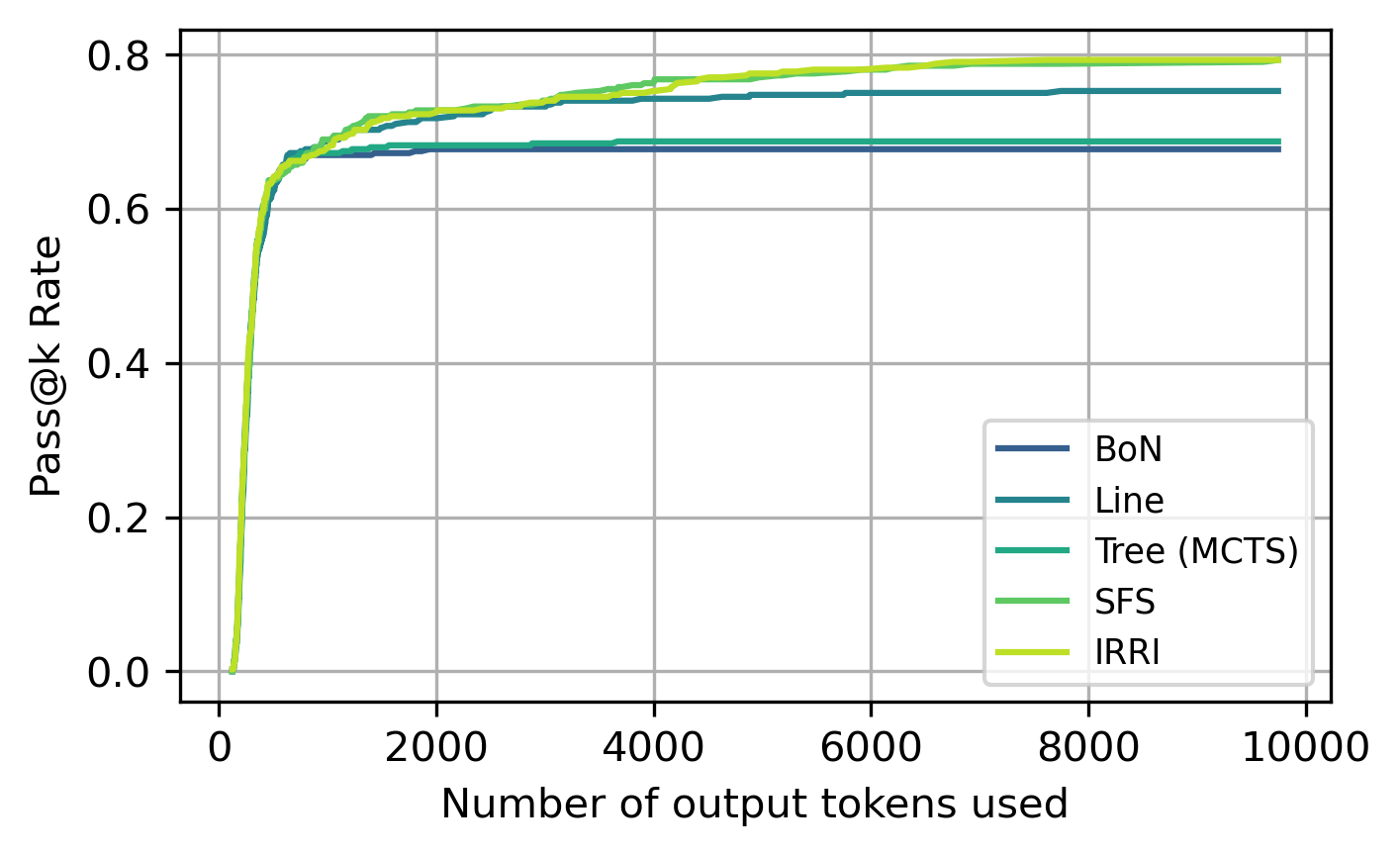}
    \end{minipage}
    \caption{\textbf{Scaling curves by token usage for different methods on HumanEval using Llama-3.1-8B-Instruct (left) and on MBPP using Llama-3.1-8B-Instruct (right).}}
    \label{fig:scaling-token-llama-3.1-8b-instruct-humaneval-and-mbpp}
\end{figure}

\begin{figure}[htbp]
    \centering
    \begin{minipage}{0.49\linewidth}
        \centering
        \includegraphics[width=\linewidth]{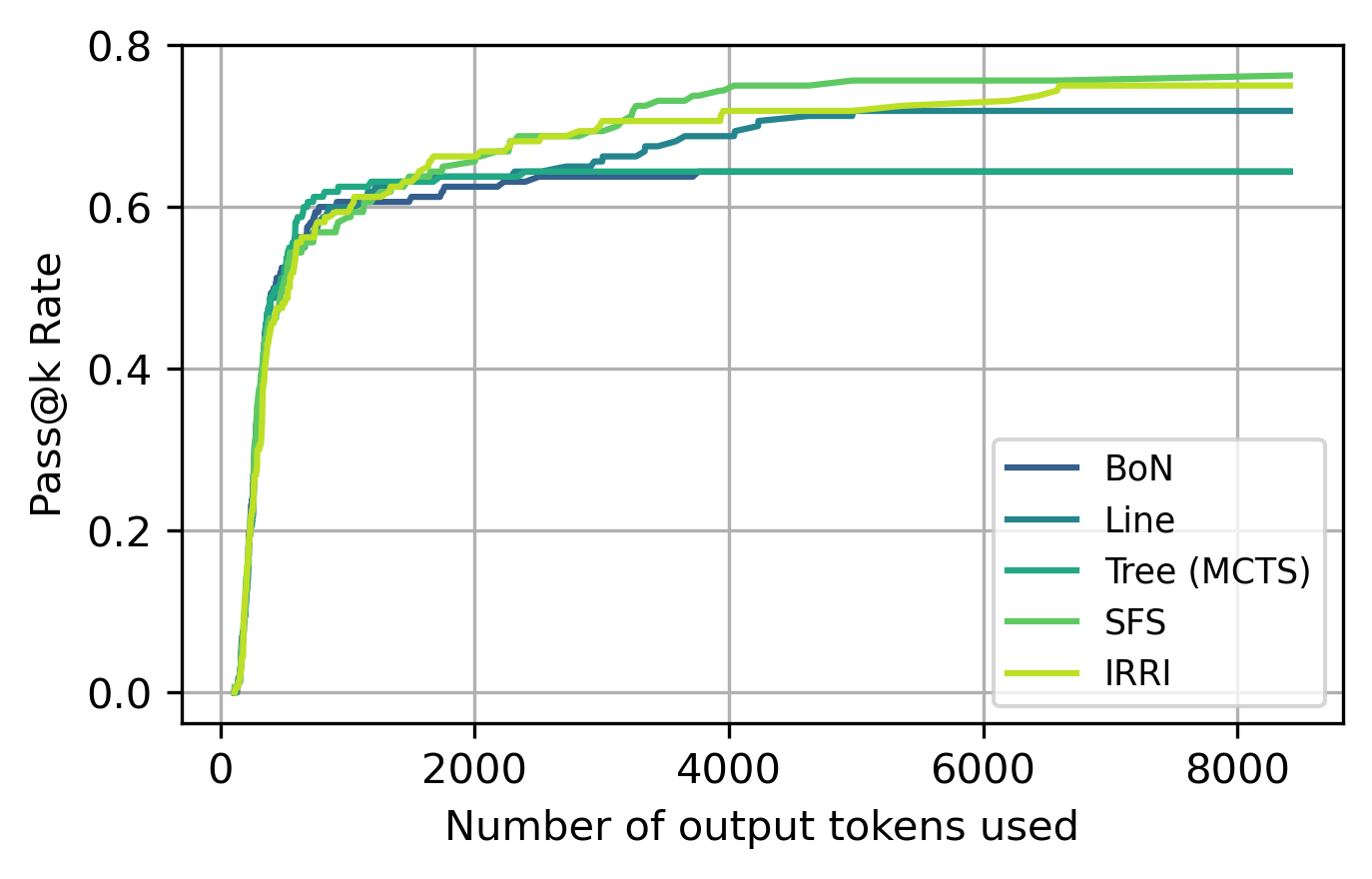}
    \end{minipage}
    \begin{minipage}{0.49\linewidth}
        \centering
        \includegraphics[width=\linewidth]{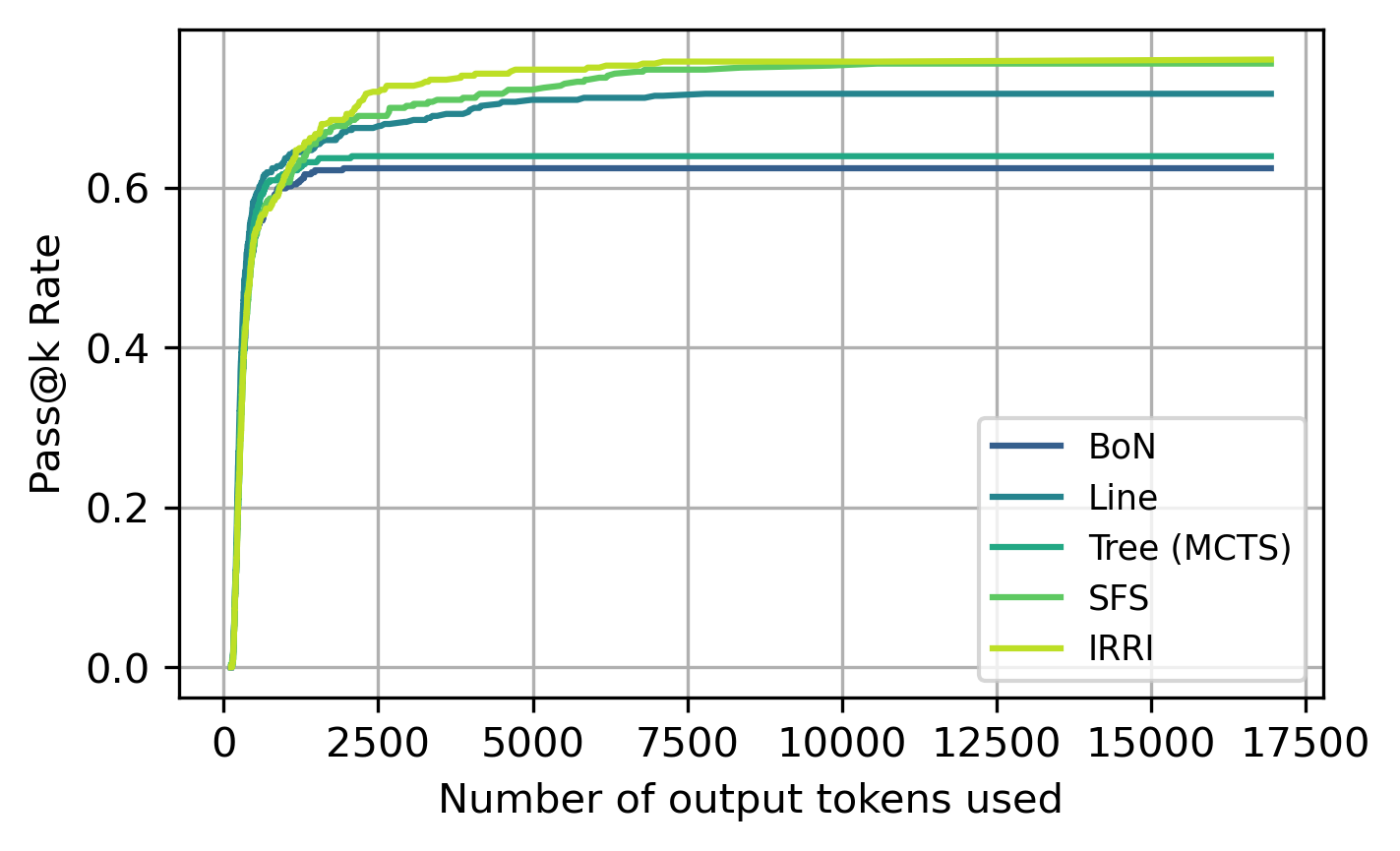}
    \end{minipage}
    \caption{\textbf{Scaling curves by token usage for different methods on HumanEval using Llama-3.2-3B-Instruct (left) and on MBPP using Llama-3.2-3B-Instruct (right).}}
    \label{fig:scaling-token-llama-3.2-3b-instruct-humaneval-and-mbpp}
\end{figure}

\FloatBarrier
\section{Example of learning from failed Repair instructions}
\label{app:example-of-learning-from-failed-repair-instructions}

We present a concrete example of how feedback from an unsuccessful refinement is distilled into shared information and improves a subsequent repair instruction.

\begin{problembox}{Problem Summary}
    A storehouse initially holds $n$ planks with given lengths ($1 \le n \le 10^5$). It then undergoes $q$ events ($1 \le q \le 10^5$): \texttt{+ x} adds a plank of length $x$, and \texttt{- x} removes one plank of length $x$ (guaranteed to exist). After every event, we must decide whether two storages can be built from the current planks: one square (four planks of equal length) and one rectangle, possibly also a square (two pairs of equal-length planks). Each side uses exactly one plank, so eight distinct planks are needed in total. Output \texttt{YES} if both storages can be built, and \texttt{NO} otherwise.
\end{problembox}

\begin{attemptbox}{Attempt 1: Unsuccessful Refinement}
    \textbf{Repair Instruction} \\
    Modify the \texttt{can\_build} function to correctly track the number of pairs and quadruples of planks. Instead of counting squares and rectangles separately, maintain a count of pairs (planks with at least 2) and quadruples (planks with at least 4). This way, you can check if there is at least one quadruple for the square and at least two pairs for the rectangle. \\
    
    \textbf{Shared Information Updated}
    \begin{enumerate}[leftmargin=*]
        \item When modifying logic to track counts (like pairs and quadruples), ensure that the counting accurately reflects the current state of the data structure after each event.
        \item Consider the implications of each operation on the overall counts and ensure that the logic is robust enough to handle dynamic updates.
        \item Testing should include edge cases where the number of planks is minimal or where operations could lead to rapid changes in the ability to form the required shapes.
    \end{enumerate}
\end{attemptbox}

\begin{attemptbox}{Attempt 2: Unsuccessful Refinement}
    \textbf{Repair Instruction} \\
    Ensure that the plank count is updated correctly after each event. When processing the \texttt{-x} operation, ensure that the count does not drop below zero. This can be done by checking if the count of \texttt{x} is greater than zero before decrementing. This will prevent any erroneous states where the count becomes negative, which could lead to incorrect calculations in the \texttt{can\_build} function. \\
    
    \textbf{Shared Information Updated}
    \begin{enumerate}[leftmargin=*]
        \item When modifying logic to track counts (like pairs and quadruples), ensure that the counting accurately reflects the current state of the data structure after each event.
        \item Consider the implications of each operation on the overall counts and ensure that the logic is robust enough to handle dynamic updates.
        \item The logic for determining whether the required shapes can be built must be carefully designed to account for all possible configurations of the available planks.
        \item Testing should include edge cases where the number of planks is minimal or where operations could lead to rapid changes in the ability to form the required shapes.
        \item It is crucial to validate the logic of any counting mechanism against the specific requirements of the problem to ensure that it meets the conditions necessary for success.
    \end{enumerate}
\end{attemptbox}

\begin{attemptbox}{Attempt 3: Successful Refinement}
    \textbf{Repair Instruction} \\
    Ensure that the update logic for the counts in the event processing loop correctly reflects the changes in the number of planks. Specifically, when a plank is added or removed, check if the count of that plank length crosses the thresholds for pairs or quadruples. This will allow the \texttt{can\_build} function to have accurate data to work with after each event.
\end{attemptbox}

In Attempt~1, the repair rewrote \texttt{can\_build} to track the numbers of pairs and quadruples, but it failed. From this failure, general lessons that do not depend on the specific code are distilled into the shared information: counts should accurately reflect the state after each event, the logic should be robust to dynamic updates, and edge cases should be tested.

In Attempt~2, the repair, guided by this shared information, added a guard to prevent the count from going negative on a \texttt{-x} operation, but it also failed. Consequently, two further lessons are added to the shared information: the logic that decides whether the required shapes can be built must account for all possible plank configurations, and the counting mechanism must be validated against the problem's requirements. This accumulates the insight that fixing the decision logic in \texttt{can\_build} or adding a non-negativity guard is insufficient, and that the correctness of the count updates themselves should be examined.

In Attempt~3, the accumulated shared information steers the model away from approaches that have already proven ineffective (modifying the decision logic or adding the non-negativity guard). Instead, it focuses on the core issue: when a plank is added or removed, the update logic must check whether the count of that length crosses the thresholds of 2 and 4, and update the numbers of pairs and quadruples accordingly. This yields a concrete repair instruction, and the repair succeeds.

In this way, lessons learned from failed attempts accumulate as shared information and progressively refine the repair instructions from general guidance into specific instructions that directly target the root cause.

\end{document}